\documentclass[11pt]{article}
\pdfoutput=1
\usepackage{enumerate}
\usepackage{pdfsync}
\usepackage[OT1]{fontenc}

\usepackage[usenames]{color}
\usepackage{smile}
\usepackage[colorlinks,
            linkcolor=red,
            anchorcolor=blue,
            citecolor=blue
            ]{hyperref}
\usepackage{fullpage}
\usepackage{hyperref}
\usepackage[protrusion=true,expansion=true]{microtype}
\usepackage{pbox}
\usepackage{setspace}
\usepackage{tabularx}
\usepackage{float}
\usepackage{wrapfig,lipsum}
\usepackage{enumitem}
\usepackage{microtype}
\usepackage{graphicx}
\usepackage{subfigure}
\usepackage{enumerate}
\usepackage{enumitem}
\usepackage{pgfplots}
\usepackage{mathtools}
\usetikzlibrary{arrows,shapes,snakes,automata,backgrounds,petri}
\usepackage{booktabs} 

\allowdisplaybreaks
\newcommand{\la}{\langle}
\newcommand{\ra}{\rangle}
\newcommand{\qvalue}{Q}
\newcommand{\vvalue}{V}

\newcommand{\reward}{r}

\newcommand{\valueite}{\text{EVI}}

\newcommand{\event}{\mathcal{E}}
\newcommand{\ip}[1]{\langle #1 \rangle}

\newcommand{\onep}[1]{\one\{#1\}}

\newcommand{\var}{\mathbb{V}}

\def \algnamedis {\text{UCLK}^{+}}
\def \algnamefin {\text{UCRL-VTR}^{+}}
\def \error {E}
\newcommand{\state}{x}
\usepackage{enumitem}

\def \ig {\bar H}
\def \pnorm {B}

\usepackage{xspace}
\usepackage{setspace}

\ifdefined\final
\usepackage[disable]{todonotes}
\else
\usepackage[textsize=tiny]{todonotes}
\fi

\makeatletter
\newcommand*{\rom}[1]{\expandafter\@slowromancap\romannumeral #1@}
\makeatother
\title{\huge Nearly Minimax Optimal Reinforcement Learning for Linear Mixture Markov Decision Processes}

\author
{
	Dongruo Zhou\thanks{Department of Computer Science, University of California, Los Angeles, CA 90095, USA; e-mail: {\tt drzhou@cs.ucla.edu}} 
	~~~and~~~
	Quanquan Gu\thanks{Department of Computer Science, University of California, Los Angeles, CA 90095, USA; e-mail: {\tt qgu@cs.ucla.edu}}
	~~~and~~~
	Csaba Szepesv\'ari\thanks{Deepmind and University of Alberta; e-mail: {\tt szepi@google.com}}
}

\begin{document}
\date{}
\maketitle

\begin{abstract}

We study reinforcement learning (RL) with linear function approximation where the underlying transition probability kernel of the Markov decision process (MDP) is a linear mixture model \citep{jia2020model, ayoub2020model, zhou2020provably} and the learning agent has access to either an integration or a sampling oracle of the individual basis kernels. 
We propose a new Bernstein-type concentration inequality for self-normalized martingales for linear bandit problems with bounded noise. Based on the new inequality, we propose a new, computationally efficient algorithm with linear function approximation named $\text{UCRL-VTR}^{+}$ for the aforementioned linear mixture MDPs in the episodic undiscounted setting.
We show that ${\text{UCRL-VTR}^{+}}$ attains an $\tilde O(dH\sqrt{T})$ regret where $d$ is the dimension of feature mapping, $H$ is the length of the episode and $T$ is the number of interactions with the MDP. We also prove a matching lower bound $\Omega(dH\sqrt{T})$ for this setting, which shows that ${\text{UCRL-VTR}^{+}}$ is minimax optimal up to logarithmic factors. In addition, we propose the ${\text{UCLK}^{+}}$ algorithm for the same family of MDPs under discounting and show that it attains an $\tilde O(d\sqrt{T}/(1-\gamma)^{1.5})$ regret, where $\gamma\in [0,1)$ is the discount factor. Our upper bound matches the lower bound $\Omega(d\sqrt{T}/(1-\gamma)^{1.5})$ proved by \citet{zhou2020provably} up to logarithmic factors, suggesting that ${\text{UCLK}^{+}}$ is nearly minimax optimal. To the best of our knowledge, these are the first computationally efficient, nearly minimax optimal algorithms for RL with linear function approximation.

\end{abstract}

\section{Introduction}

Improving the sample efficiency of reinforcement learning (RL) algorithms has been a central research question in the RL community. When there are finitely many states and actions and the value function are represented using ``tables'', the case known as ``tabular RL'', a number of breakthroughs during the past decade led to a thorough understanding of the limits of sample efficiency of RL. In particular, algorithms with nearly minimax optimal sample complexity have been discovered for the planning setting where a generative model is available \citep{azar2013minimax, sidford2018near, agarwal2020model}. Significant further work then led to nearly minimax optimal algorithms%
\footnote{In this paper, we say an algorithm is nearly minimax optimal if this algorithm attains a regret/sample complexity that matches the minimax lower bound up to logarithmic factors.} 
for the more challenging online learning setting, where the results cover a wide variety of objectives, ranging from episodic Markov Decision Process (MDP) \citep{azar2017minimax, zanette2019tighter, zhang2020almost}, through discounted MDPs \citep{lattimore2012pac, zhang2020model, he2020minimax} to infinite horizon MDPs with the average reward criterion \citep{zhang2019regret, tossou2019near}. 

Results developed for the tabular case are significant because the core algorithmic ideas often generalize beyond the tabular case. They are also important because they show in a precise, quantitative way that without extra structure, efficient learning in large state-action space MDPs is inherently intractable. A classical approach to deal with such large MDPs is to assume access to a \emph{function approximation} technique that allows for a compact, or compressed representation of various objects of interest, such as policies or value functions. 
The question then is whether for MDPs where the function approximator is able to provide a good approximation to (say) the value functions that one may encounter, sample efficient algorithm exists. A specific case of much interest is when the function approximator is linear in that in it a fixed number of basis functions that map either the state space or the state-action space to the reals are linearly combined with some weights to be computed
\citep{schwesei85}.

When a generative model is available, 
\citet{yang2019sample}
proposed a computationally efficient,
nearly minimax optimal RL algorithm that works with such linear function approximation 
for a special case when the learner has access to a polynomially sized set of  ``anchor state-action pairs''. \citet{lattimore2020learning} proposed an optimal-design based RL algorithm without the anchor state-action pairs assumption. 
However, for online RL where no generative model is accessible, as of today a gap between the upper bounds \citep{yang2019reinforcement, jin2019provably, wang2019optimism, modi2019sample,zanette2020frequentist, zanette2020learning, jia2020model,ayoub2020model} and the lower bounds \citep{du2019good, zhou2020provably} still exist, with or without the anchor state-action assumption.
Therefore, a natural question arises: 
\begin{center}
    \emph{Does there exist a computationally efficient, nearly minimax optimal RL algorithm with linear function approximation?}
\end{center}
In this paper, we answer this question affirmatively for a special class of MDPs named linear mixture MDPs, where the transition probability kernel is a linear mixture of a number of basis kernels \citep{jia2020model, ayoub2020model, zhou2020provably}. 
Following ideas developed for the tabular case (e.g., \citealt{azar2013minimax}), for undiscounted problems we replace the conservative Hoeffding-type confidence bounds used in UCRL-VTR of \citet{ayoub2020model} with a Bernstein-type confidence bound that is based on a new, Bernstein-type variant of the standard self-normalized concentration inequality of  \citet{abbasi2011improved}. For discounted problems the same modification is done on UCLK of \citet{zhou2020provably}. Both algorithms are computationally efficient provided access to either an integration or sampling oracle of the basis kernels. 
In detail, our contributions are listed as follows.  
\begin{itemize}
    \item We propose a Bernstein-type self-normalized concentration inequality for vector-valued martingales, which improves the dominating term
    of 
    the analog inequality of \citet{abbasi2011improved}
     from $R\sqrt{d}$ to $\sigma\sqrt{d}+R$, where $R$ and $\sigma^2$ are the magnitude and the variance of the noise respectively, and $d$ is the dimension of the vectors involved. Our concentration inequality is a non-trivial extension of the Bernstein concentration inequality from the scalar case to the vector case. 
    \item With the Bernstein-type tail inequality, we consider a linear bandit problem as a ``warm-up'' example, whose noise at round $t$ is $R$-bounded and of $\sigma_t^2$-variance. Note that bandits can be seen as a special instance of episodic RL where the length of the episode equals one. We propose a new algorithm named Weighted OFUL, which adapts a new linear regression scheme called \emph{weighted ridge regression}. We prove that Weighted OFUL enjoys an $\tilde O(R\sqrt{d T} + d\sqrt{\sum_{t=1}^T \sigma_t^2})$ regret, which strictly improves the regret $\tilde O(  Rd\sqrt{T})$ obtained for the OFUL algorithm by \citet{abbasi2011improved}.
    \item We further apply the new tail inequality to the design and analysis of online RL algorithms for the aforementioned linear mixture MDPs \citep{jia2020model, ayoub2020model, zhou2020provably}.  
    In the episodic setting, we propose a new algorithm $\algnamefin$, which can be seen as an extension of UCRL-VTR in \citet{jia2020model, ayoub2020model}. The key idea of $\algnamefin$ is to utilize weighted ridge regression and a new estimator for the variance of the value function. We show that $\algnamefin$ attains an $\tilde O(dH\sqrt{T})$ regret, where $T$ is the number of rounds and $H$ is the length of the episodes. We also prove a nearly matching lower bound $\Omega(dH\sqrt{T})$ on the regret, which shows that our $\algnamefin$ algorithm is minimax optimal up to logarithmic factors.
    \item We also propose an algorithm named $\algnamedis$  for discounted MDPs under the linear mixture MDP assumption, which can be seen as an extension of UCLK algorithm proposed in \citet{zhou2020provably}. We show that $\algnamedis$ attains an $\tilde O(d\sqrt{T}/(1-\gamma)^{1.5})$ regret, where $\gamma$ is the discount factor. It matches the lower bound  $\Omega(d\sqrt{T}/(1-\gamma)^{1.5})$ on the regret proved in \citet{zhou2020provably} up to logarithmic factors, which suggests that $\algnamedis$ is also nearly minimax optimal.
\end{itemize}
 To the best of our knowledge, ignoring logarithmic factors, our proposed $\algnamefin$ and $\algnamedis$ are the first minimax optimal online RL algorithms with linear function approximation. Previous regret bounds of online RL with linear function approximation are not optimal, by differing from the corresponding lower bounds by factors of $d$ and/or $H$ (or $1/(1-\gamma)$) \citep{jin2019provably, jia2020model, ayoub2020model, zhou2020provably}. The only exception is \citet{zanette2020learning}, which studied RL with linear function approximation under the \emph{low inherent Bellman error} assumption. 
 They proposed an ELEANOR algorithm with a regret $\tilde O(\sum_{h=1}^H d_h \sqrt{K})$, where $d_h$ is the dimension of the feature mapping at the $h$-th stage within the episodes and $K$ is the number of episodes. They also proved a lower bound $\Omega(\sum_{h=1}^H d_h \sqrt{K})$ under the sub-Gaussian norm assumption
of the rewards and transitions but only for the special case when $d_1 = \sum_{h=2}^H d_h$. It can be seen that in this special case, their upper bound matches their lower bound up to logarithmic factors, and thus their algorithm is statistically near optimal. However, in the general case when $d_1 = \dots = d_H = d$, there still exists a gap of $H$ between their upper and lower bounds. Furthermore, as noted by the authors, the ELEANOR algorithm is not computationally efficient.
 
 
 The remainder of this paper is organized as follows. In Section \ref{sec:relatedwork}, we review some prior work that is most related to ours. In Section \ref{section 3}, we introduce necessary background and preliminaries of our work. In Section \ref{sec:linearbandit}, we propose a Bernstein-type self-normalized concentration inequality for vector-valued martingales and show an improved bound for linear bandits with bounded noises. In Section \ref{section:finite_main}, we propose $\algnamefin$ for linear mixture MDPs in the episodic setting along with its regret analysis, and prove a nearly matching lower bound. In Section \ref{section 5} we propose $\algnamedis$ for linear mixture MDPs in the discounted setting. We also prove the nearly matching upper and lower regret bounds correspondingly. We conclude the paper and discuss the future work in Section \ref{sec:conclusion}.
 

\paragraph{Notation} 
We use lower case letters to denote scalars, and use lower and upper case bold face letters to denote vectors and matrices respectively. We denote by $[n]$ the set $\{1,\dots, n\}$. For a vector $\xb\in \RR^d$ and matrix $\bSigma\in \RR^{d\times d}$, a positive semi-definite matrix, we denote by $\|\xb\|_2$ the vector's Euclidean norm and define $\|\xb\|_{\bSigma}=\sqrt{\xb^\top\bSigma\xb}$. For $\xb, \yb\in \RR^d$, let $\xb\odot\yb$ be the Hadamard product of $\xb$ and $\yb$. For two positive sequences $\{a_n\}$ and $\{b_n\}$ with $n=1,2,\dots$, 
we write $a_n=O(b_n)$ if there exists an absolute constant $C>0$ such that $a_n\leq Cb_n$ holds for all $n\ge 1$ and write $a_n=\Omega(b_n)$ if there exists an absolute constant $C>0$ such that $a_n\geq Cb_n$ holds for all $n\ge 1$. We use $\tilde O(\cdot)$ to further hide the polylogarithmic factors. We use $\ind\{\cdot\}$ to denote the indicator function. For $a,b \in \RR$ satisfying $a \leq b$, we use 
$[x]_{[a,b]}$ to denote the function $x\cdot \ind\{a \leq x \leq b\} + a\cdot \ind\{x<a\} + b\cdot \ind\{x>b\}$.

\section{Related Work}\label{sec:relatedwork}

The purpose of this section is to review prior works that are most relevant to our contributions.

\noindent\textbf{Linear Bandits}
Linear bandits can be seen as the simplest version of RL with linear function approximation, where the episode length (i.e., planning horizon) $H=1$.
There is a huge body of literature on linear bandit problems \citep{auer2002using, chu2011contextual, li2010contextual, li2019nearly, dani2008stochastic, abbasi2011improved}. Most of the linear bandit algorithms can be divided into two categories: algorithms for $k$-armed linear bandits, and algorithms for infinite-armed linear bandits. For the $k$-armed case, \citet{auer2002using} proposed a SupLinRel algorithm, which makes use of the eigenvalue decomposition and enjoys an $O(\log^{3/2}(kT)\sqrt{dT})$ regret\footnote{We omit the $\text{poly}(\log\log(kT))$ factors for the simplicity of comparison.} . \citet{li2010contextual, chu2011contextual} proposed a SupLinUCB algorithm using the regularized least-squares estimator, which enjoys the same regret guarantees. 
\citet{li2019nearly} proposed a VCL-SupLinUCB algorithm with a refined confidence set design which enjoys an improved $O(\sqrt{\log( T) \log( k )dT})$ regret, 
which matches the lower bound up to a logarithmic factor. For the infinite-armed case, \citet{dani2008stochastic} proposed an algorithm with a confidence ball, which enjoys $O(d\sqrt{T\log^3 T})$ regret. \citet{abbasi2011improved} improved the regret to $O(d\sqrt{T\log^2 T})$ with a new self-normalized concentration inequality for vector-valued martingales. \citet{li2019tight} further improved the regret to $O(d\sqrt{T\log T})$, which matches the lower bound up to a logarithmic factor. However, previous works only focus on the case where the reward noise is sub-Gaussian. In this paper, we show that if the reward noise is restricted to a smaller class of distributions with bounded magnitude and variance, a better regret bound can be obtained. The main motivation to consider this problem is that linear bandits with bounded reward and variance can be seen as a special RL with linear function approximation when the episode length $H=1$. Thus, this result immediately sheds light on the challenges involved in achieving minimax optimal regret for general RL with linear function approximation.




\noindent\textbf{Reinforcement Learning with Linear Function Approximation} 
Recent years have witnessed a flurry of activity on
RL with linear function approximation \citep[e.g.,][]{jiang2017contextual,yang2019sample, yang2019reinforcement, jin2019provably, wang2019optimism,  modi2019sample, dann2018oracle, du2019good,  sun2019model,  zanette2020frequentist,zanette2020learning,cai2019provably,jia2020model,ayoub2020model,weisz2020exponential,zhou2020provably,he2020logarithmic}. 
These results can be generally grouped into four categories based on their assumptions on the underlying MDP. The first category of work uses the low Bellman-rank assumption \citep{jiang2017contextual} which assumes that 
the Bellman error ``matrix''  where ``rows'' are index by a test function and columns are indexed by a distribution generating function from the set of test functions 
assumes a low-rank factorization. Representative work includes \citet{jiang2017contextual, dann2018oracle, sun2019model}. The second category of work considers the 
\emph{linear MDP} assumption \citep{yang2019sample, jin2019provably} which assumes both the transition probability function and reward function are parameterized as a linear function of a given feature mapping over state-action pairs. 
Representative work includes \citet{yang2019sample, jin2019provably, wang2019optimism, du2019good, zanette2020frequentist, wang2020reinforcement, he2020logarithmic}. The third category of work focuses on the low inherent Bellman error assumption \citep{zanette2020learning}, which assumes the Bellman backup can be parameterized as a linear function up to some misspecification error.  
The last category considers linear mixture MDPs (a.k.a., linear kernel MDPs) \citep{jia2020model, ayoub2020model, zhou2020provably}, which assumes the transition probability function is parameterized as a linear function of a given feature mapping over state-action-next-state triple. Representative work includes \citet{yang2019reinforcement, modi2019sample, jia2020model, ayoub2020model, cai2019provably, zhou2020provably, he2020logarithmic}.
Our work also considers linear mixture MDPs. 

\noindent\textbf{Bernstein Bonuses for Tabular MDPs}
There is a series of work proposing algorithms with nearly minimax optimal sample complexity or regret for the tabular MDP under different settings, including average-reward, discounted, and episodic MDPs \citep{azar2013minimax, azar2017minimax,zanette2019tighter,zhang2019regret, simchowitz2019non, zhang2020almost, he2020minimax, zhang2020reinforcement}. The key idea at the heart of these works is the usage of the law of total variance to obtain tighter bounds on the expected sum of the variances for the estimated value function. These works have designed tighter confidence sets or upper confidence bounds by replacing the Hoeffding-type exploration bonuses with Bernstein-type exploration bonuses, and obtained more accurate estimates of the optimal value function,
a technique pioneered by \citet{lattimore2012pac}.
Our work shows how this idea extends to algorithms with linear function approximation. To the best of our knowledge, our work is the first work using Bernstein bonus and law of total variance to achieve nearly minimax optimal regret for RL with linear function approximation.

\section{Preliminaries}\label{section 3}

In this paper, we study RL with linear function approximation for both episodic MDPs and infinite-horizon discounted MDPs. In the following, we will introduce the necessary background and definitions. For further background, the reader is advised to consult, e.g., \cite{puterman2014Markov}.
For a positive integer $n$, we use $[n] = \{1,\dots,n\}$ to denote the set of integers from one to $n$.

\paragraph{Inhomogeneous, episodic MDP} We denote an inhomogeneous, episodic MDP by a tuple $M=M(\cS, \cA, H, \{\reward_h\}_{h=1}^H, \{\PP_h\}_{h=1}^H)$, where $\cS$ is the state space and $\cA$ is the action space, $H$ is the length of the episode, $\reward_h: \cS \times \cA \rightarrow [0,1]$ is the deterministic reward function, and $\PP_h$ is the transition probability function at stage $h$
so that for $s,s'\in \cS$, $a\in \cA$, $\PP_h(s'|s,a)$ is the probability of arriving at stage $h+1$ at state $s'$ provided that the state at stage $h$ is $s$ and action $a$ is chosen at this stage.
For the sake of simplicity, we restrict ourselves to countable state space and finite action space. However, 
for most purposes, this is a non-essential assumption: At the expense of some technicalities, our main results continue to hold for general state and action spaces \citep{bertsekas2004stochastic} except our results concerning computational efficiency of the algorithms we will be concerned with, which rely on having finitely many actions.
A policy $\pi = \{\pi_h\}_{h=1}^H$ is a collection of $H$ functions, where each of them maps a state $s$ to an action $a$. 
For $(s,a)\in \cS \times \cA$, 
we define the action-values $\qvalue_h^{\pi}(s,a)$ and (state) values $\vvalue_h^\pi(s)$ as follows:
\begin{align}
\qvalue_h^{\pi}(s,a) = \EE_{\pi,h,s,a}\bigg[\sum_{h' = h}^H \reward_h(s_{h'}, a_{h'}) \bigg],\ 
\vvalue_h^{\pi}(s) = \qvalue_h^{\pi}(s,\pi_h(s)),\ 
\vvalue_{H+1}^{\pi}(s) = 0.\notag
\end{align}
In the definition of $\qvalue_h^\pi$, $\EE_{\pi,h,s,a}$ means an expectation over 
the probability measure over state-action pairs of length $H-h+1$ 
that is induced by the interconnection of policy $\pi$ and the MDP $M$ when initializing the process to start at stage $h$ with the pair $(s,a)$.
In particular, the probability of sequence $(s_h,a_h,s_{h+1},a_{h+1},\dots,s_{H},a_H)$ under this sequence is $\one(s_h=s) \one(a_h=a) \PP_h(s_{h+1}|s_h,a_h) \one_{\pi_{h+1}(s_{h+1})=a_{h+1}} \dots \PP_{H-1}(s_H|s_{H-1},a_{H-1}) \one_{\pi_{H}(s_H)=a_H}$.
These definitions trivially extend to stochastic policies, which give distributions over the actions for each stage and state. In what follows, we also allow stochastic policies.

The optimal value function $V_h^*(\cdot)$ and the optimal action-value function $\qvalue_h^*(\cdot, \cdot)$ are defined by $V^*_h(s) = \sup_{\pi}\vvalue_h^{\pi}(s)$ and $\qvalue_h^*(s,a) = \sup_{\pi}\qvalue_h^{\pi}(s,a)$, respectively. For any function $\vvalue: \cS \rightarrow \RR$, we introduce the shorthands
\begin{align*}
[\PP_h \vvalue](s,a) & =\EE_{s' \sim \PP_h(\cdot|s,a)}\vvalue(s')\,, \\
[\var_h\vvalue](s,a)  &= [\PP_h \vvalue^2](s,a) - ([\PP_h \vvalue](s,a))^2\,,
\end{align*}
 where $\vvalue^2$ stands for the function whose value at $s$ is $\vvalue^2(s)$. Using this notation, the Bellman equations for policy $\pi$ can be written as
\begin{align}
    \qvalue_h^{\pi}(s,a) = \reward_h(s,a) +[\PP_h\vvalue_{h+1}^\pi](s,a)\,,\notag
\end{align}
while the Bellman optimality equation takes the form
\begin{align}
    \qvalue_h^*(s,a) = \reward_h(s,a) +[\PP_h\vvalue_{h+1}^*](s,a).\notag
\end{align}
Note that both hold \emph{simultaneously} for all $(s,a)\in \cS\times \cA$ and $h\in [H]$.

In the \textbf{online learning setting}, a learning agent 
who does not know the kernels $\{\PP_h\}_h$ but, for the sake of simplicity, knows the rewards $\{r_h\}_h$,
aims to learn to take good actions by interacting with the environment. For each $k \geq 1$, at the beginning of the $k$-th episode, the environment picks the initial state $s_1^k$ and the agent chooses a policy $\pi^k$ to be followed in this episode. As the agent follows the policy through the episode, it observes the sequence of states $\{s_h^k\}_h$
with $s_{h+1}^k \sim \PP_h(\cdot|s_h^k,\pi^k(s_h^k))$.
The difference between $\vvalue_1^*(s_1^k)$ and $\vvalue_{1}^{\pi^k}(s_1^k)$ represents the 
total reward that the the agent loses in the $k$-th episode as compared with acting optimally.
The goal is to design a learning algorithm that constructs the sequence $\{\pi^k\}_k$ based on past information
so that the $K$-episode regret, 
\begin{align}
    \text{Regret}(M, K) = \sum_{k=1}^K \big[\vvalue_1^*(s_1^k) - \vvalue_{1}^{\pi^k}(s_1^k)\big]\notag
\end{align}
is kept small.
In this paper, we focus on proving high probability bounds on the regret $\text{Regret}(M, K)$, as well as lower bounds in expectation. 

\paragraph{Discounted MDP}
We denote a discounted MDP by a tuple $M=M(\cS, \cA, \gamma, \reward, \PP)$, where $\cS$ is the countable state space and $\cA$ is the finite action space, $\gamma: 0 \leq \gamma <1$ is the discount factor, $\reward: \cS \times \cA \rightarrow [0,1]$ is the deterministic reward function, and $\PP(s'|s,a) $ is the transition probability function. A (nonstationary) deterministic policy $\pi$ is a collection of functions $\{\pi_t\}_{t=1}^\infty$, where each $\pi_t: \{\cS\times\cA\}^{t-1}\times \cS \rightarrow \cA$ maps history $\{s_1,a_1,\dots,s_{t-1},a_{t-1}, s_t\}$ to an action $a$.  A stochastic policy is similar, except that the image space is the set of probability distributions over $\cA$. For simplicity, let us now stick to deterministic policies.
Let $\{s_t, a_t\}_{t=1}^\infty$ be the random sequence of states and actions 
induced from the interconnection of $\PP$ and $\pi$. That is,
 $s_t \sim \PP(\cdot| s_{t-1}, a_{t-1})$ holds for $t\ge 2$ and 
 $a_t = \pi_t(s_1,a_1,\dots,s_{t-1},a_{t-1}, s_t)$ 
 holds for $t\ge 1$. When there is no confusion, we write $\pi_t(s_t) = \pi_t(s_1,a_1,\dots,s_{t-1},a_{t-1}, s_t)$ as a shorthand notation.  
Intuitively, we want to define
the action-value function $\qvalue^{\pi}_t$ and value function $\vvalue^{\pi}_t$ as follows: 
\begin{equation}
\begin{split}
&\qvalue^{\pi}_t(s,a) = \EE\bigg[\sum_{i = 0}^\infty \gamma^{i}\reward(s_{t+i}, a_{t+i})
\,\bigg| \,s_1,a_1,\ldots, s_{t-1}, a_{t-1}, s_t = s, a_t = a\bigg], \\
&\vvalue^{\pi}_t(s) = \EE\bigg[\sum_{i = 0}^\infty \gamma^{i}\reward(s_{t+i}, a_{t+i})
\,\bigg| \, s_1,a_1,\ldots, s_{t-1}, a_{t-1}, s_t = s\bigg]\,,
\end{split}
\label{eq:valueconddef}
\end{equation}
where $(s,a)\in \cS \times \cA$. 
These quantities are almost surely well-defined, though the astute reader will note that this hold only when the probability of $\{s_t=s,a_t=a\}$ is positive. A slightly more technical definition, which is consistent with the above definition, avoids this issue. To keep the flow of the paper, this definition,
which has the same spirit as the definition given in the finite-horizon case,
 is deferred to Appendix~\ref{sec:app_valuedef}.
Note that $\qvalue^\pi_t$ is a random (history-dependent) function; 
it is by definition a measurable function of $(s_1,a_1,\dots,s_{t-1},a_{t-1})$.
As usual in probability theory, this dependence on the history is suppressed. The same holds for $V_t^\pi$.

We also define the optimal value function $V^*(\cdot)$ and the optimal action-value function $\qvalue^*(\cdot, \cdot)$ as $V^*(s) = \sup_{\pi}\vvalue^{\pi}_1(s)$ and $\qvalue^*(s,a) = \sup_{\pi}\qvalue^{\pi}_1(s,a)$, $(s,a)\in \cS \times \cA$. 
For any function $\vvalue: \cS \rightarrow \RR$, we introduce the shorthand notations 
\begin{align*}
[\PP \vvalue](s,a)&=\EE_{s' \sim \PP(\cdot|s,a)}\vvalue(s')\,,\\
[\var\vvalue](s,a) &= [\PP \vvalue^2](s,a) - ([\PP \vvalue](s,a))^2\,,
\end{align*}
where $\vvalue^2$ stands for the function whose value at $s$ is $\vvalue^2(s)$.
Using this notation, the ``Bellman equation'' for $\pi$ reads
\begin{align}\label{eq:nonstatbellman}
    \qvalue^{\pi}_t(s,a) = \reward(s,a) +\gamma [\PP\vvalue^{\pi}_{t+1}](s,a).
\end{align}
We also have the following relation, which captures ``half'' of the Bellman optimality equation: 
\begin{align}
    \qvalue^*(s,a) = \reward(s,a) + \gamma[\PP\vvalue^*](s,a).\notag
\end{align}

In the \textbf{online learning setting}, 
the goal is to design an algorithm (equivalently, choosing a nonstationary policy $\pi$) so that, regardless of $\PP$,
$V_t^\pi(s_t)$, the total discounted expected return of policy from time step $t$ on from the current state $t$, is as close to $V^*(s_t)$, the optimal value of the current state as possible.
Again, for simplicity, we assume that the immediate reward function is known.
The environment picks the state $s_1$ at the beginning. 
Given a designated bound $T$ on 
the length of the interaction between the learning agent and the environment,
we can formalize the goal as that of minimizing the regret
\begin{align}
    \text{Regret}(M, T) = \sum_{t=1}^T \Delta_t,\ \text{where} \ \Delta_t = 
    \vvalue^*(s_t) - \vvalue^{\pi}_t(s_t).\notag
\end{align}
incurred during the first $T$ rounds of interaction between $\pi$ and the MDP $M$.
Here, for $t\in [T]$, $\Delta_t$ is called the suboptimality gap of $\pi$ in $M$ at time $t$.
This notion of regret was introduced by \citet{liu2020regret} and later adapted by \citet{zhou2020provably, yang2020q, he2020minimax}, which is further inspired by the sample
complexity of exploration in \citet{kakade2003sample}. 
In detail, our regret is the summation of suboptimality gaps $\Delta_t$ and sample complexity of exploration counts the number of timesteps when the suboptimality gaps are greater than some threshold $\epsilon$. 
For a more in-depth discussion of this notion of regret the interested reader may consult 
 \citet{zhou2020provably}.

\noindent\textbf{Linear Mixture MDPs}
We consider a special class of MDPs called \emph{linear mixture MDPs} (a.k.a., linear kernel MDPs), where the transition probability kernel is a linear mixture of a number of basis kernels. This class has been considered by a number of previous authors
\citep{jia2020model, ayoub2020model,zhou2020provably}
and is defined as follows:
Firstly, let $\bphi(s'|s,a): \cS \times \cA \times \cS \rightarrow \RR^d$ be a feature mapping satisfying that for any bounded function $\vvalue: \cS \rightarrow [0,1]$ and any tuple $(s,a)\in \cS \times \cA$, we have 
\begin{align}
    \|\bphi_{{\vvalue}}(s,a)\|_2 \leq 1,\text{where}\ 
    \bphi_{{\vvalue}}(s,a) = 
    \sum_{s'\in \cS}\bphi(s'|s,a)\vvalue(s') \,.
    \label{def:bbbphi}
\end{align}
\textbf{Episodic linear mixture MDPs} and \textbf{discounted linear mixture MDPs} are defined as follows:
\begin{definition}[\citealt{jia2020model, ayoub2020model}]\label{assumption-linear}
$M=M(\cS, \cA, H, \{\reward_h\}_{h=1}^H, \{\PP_h\}_{h=1}^H)$ is called an inhomogeneous, episodic $\pnorm$-bounded  linear mixture MDP if there exist vectors $\btheta_h \in \RR^d$ with $\|\btheta_h\|_2 \leq \pnorm$ and $\bphi(\cdot|\cdot, \cdot)$ satisfying \eqref{def:bbbphi}, such that $\PP_h(s'|s,a) = \la \bphi(s'|s,a), \btheta_h\ra$ for any state-action-next-state triplet $(s,a,s') \in \cS \times \cA \times \cS$ and stage $h$. 
\end{definition}
\begin{definition}[\citealt{zhou2020provably}]\label{assumption-discount}
$M=M(\cS, \cA, \gamma, \reward, \PP)$ is called a $\pnorm$-bounded discounted linear mixture MDP if there exists a vector $\btheta \in \RR^d$ with $\|\btheta\|_2 \leq \pnorm$ and $\bphi(\cdot|\cdot, \cdot)$ satisfying \eqref{def:bbbphi}, such that $\PP(s'|s,a) = \la \bphi(s'|s,a), \btheta\ra$ for any state-action-state triplet $(s,a,s') \in \cS \times \cA \times \cS$.
\end{definition}
Note that in the learning problem, the vectors introduced in the above definition are initially unknown to the learning agent. In the rest of this paper, we assume the underlying episodic linear mixture MDP is parameterized by $\bTheta^* = \{\btheta^*_h\}_{h=1}^H$ and denote this MDP by $M_{\bTheta^*}$. We also assume the underlying discounted linear mixture MDP is parameterized by $\btheta^*$. Similarly, we denote this MDP by $M_{\btheta^*}$. In addition, for the ease of presentation, we introduce $\ig = 1/(1-\gamma)$, which denotes the effective horizon length of the discounted MDP. 

\section{Challenges and New Technical Tools}
\label{sec:linearbandit}

To motivate our approach, we start this section with a recap of previous work addressing online learning in episodic linear mixture MDPs. This allows us to argue for how this work falls short of achieving minimax optimal regret: In effect, the confidence bounds used are too conservative as they do not make use of the fact that the  variance of value functions that appear in the algorithm are often significantly smaller than their magnitude. 
A second source of the problem is that the variances of responses used in the least-squares fits are non-uniform.
As a result, naive least-squares estimators (i.e. those that do not take the variance information into account) 
introduces errors that are too large, calling for a new, weighted least-squares estimator.
This estimator, together with a new concentration inequality that allows us to construct Bernstein-type (variance aware) confidence bounds in linear prediction with dependent inputs is described in the second half of the section.

\subsection{Barriers to Minimax Optimality in RL with Linear Function Approximation}
To understand the key technical challenges that underly achieving minimax optimality in RL with linear function approximation, we first look into the UCRL-VTR  method of \citet{jia2020model} (for a longer exposition, with refined results see \citet{ayoub2020model}) for episodic linear mixture MDPs. The key idea of UCRL-VTR is to use a model-based supervised learning framework to learn the underlying unknown parameter vector $\btheta_h^*$ of linear mixture MDP, and use the learned parameter vector $\btheta_{k,h}$ to build an optimistic estimator $\qvalue_{k,h}(\cdot,\cdot)$ for the optimal action-value function $\qvalue^*(\cdot,\cdot)$. In detail, for any stage $h$ of the $k$-th episode, the following equation holds:
For a value function $V_k = (V_{k,h})_h$ constructed based on data received before episode $k$
and the state action pair $(s_h^k,a_h^k)$ visited in stage $h$ of episode $k$,
\begin{align}
    [\PP_h\vvalue_{k, h+1}](s_h^k, a_h^k) =\bigg\la\sum_{s'}\bphi(s'|s_h^k, a_h^k)\vvalue_{k, h+1}(s'), \btheta_h^*\bigg\ra = \big\la\bphi_{\vvalue_{k, h+1}}(s_h^k, a_h^k), \btheta_h^*\big\ra,\notag
\end{align}
where the first equation holds due to the definition of linear mixture MDPs (cf. Definition \ref{assumption-linear}), the second equation holds due to the definition of $\bphi_{\vvalue_{k, h+1}}(\cdot, \cdot)$ in \eqref{def:bbbphi}. 
As it turns out, taking actions that maximize the value shown above with an appropriately constructed value function $V_k$ is sufficient for minimizing regret.
Therefore,
learning the underlying $\btheta_h^*$ can be regarded as solving  a ``linear bandit'' problem \citep[Part V,][]{lattimore2018bandit}, where the context is $\bphi_{\vvalue_{k, h+1}}(s_h^k, a_h^k) \in \RR^d$, and the noise is $\vvalue_{k, h+1}(s_{h+1}^k) - [\PP_h\vvalue_{k, h+1}](s_h^k, a_h^k)$. Previous work \citep{jia2020model, ayoub2020model} proposed an estimator $\btheta_{k,h}$ as the minimizer to the following regularized linear regression problem:
\begin{align}
    \btheta_{k,h} = \argmin_{\btheta \in \RR^d}\lambda\|\btheta\|_2^2 + \sum_{j = 1}^{k-1} \big[\big\la\bphi_{\vvalue_{j, h+1}}(s_h^j, a_h^j), \btheta\big\ra - \vvalue_{j, h+1}(s_{h+1}^j)\big]^2.\label{eq:trivialest}
\end{align}
By using the standard self-normalized concentration inequality for vector-valued martingales of \citet{abbasi2011improved}, one can show then that, with high probability, $\btheta^*_h$ lies in the ellipsoid 
\begin{align}
    \cC_{k,h} = \bigg\{\btheta: \Big\|\bSigma_{k,h}^{1/2}(\btheta - \btheta_{k,h})\Big\|_2 \leq \beta_k\bigg\}\notag
\end{align}
which is centered at $\btheta_{k,h}$, with shape parameter 
$\bSigma_{k,h} = \lambda\Ib + \sum_{j = 1}^{k-1} \bphi_{\vvalue_{j, h+1}}(s_h^j, a_h^j)\bphi_{\vvalue_{j, h+1}}(s_h^j, a_h^j)^\top$ and where $\beta_k$ is the radius chosen to be proportional to the magnitude of the value function $\vvalue_{k, h+1}(\cdot)$, which eventually gives $ \beta_k = \tilde O(\sqrt{d}H)$. It follows that if we define
\begin{align}
    \qvalue_{k,h}(\cdot, \cdot)= \min\Big\{H, \reward_h(\cdot, \cdot) + \big\la \btheta_{k,h}, \bphi_{\vvalue_{k, h+1}}(\cdot, \cdot) \big\ra + \beta_k \Big\| \bSigma_{k, h}^{-1/2} \bphi_{\vvalue_{k, h+1}}(\cdot, \cdot)\Big\|_2\Big\},
\end{align}
then, with high probability,
$\qvalue_{k,1}(\cdot, \cdot)$ is an overestimate of $\qvalue^*_1(\cdot, \cdot)$, and the summation of suboptimality gaps can be bounded by $\sum_{k=1}^K \sum_{h=1}^H\beta_k \| \bSigma_{k, h}^{-1/2} \bphi_{\vvalue_{k, h+1}}(\cdot, \cdot)\|_2$. This leads to the  $\tilde O(dH^{3/2}\sqrt{T})$ regret by further applying the elliptical potential lemma from linear bandits \citep{abbasi2011improved}.  

However, we note that the above reasoning has the following shortcomings. First, it chooses the confidence radius $\beta_k$ proportional to the \emph{magnitude} of the value function $\vvalue_{k, h+1}(\cdot)$ rather than its \emph{variance} $[\var_h\vvalue_{k, h+1}](\cdot, \cdot)$. This is known to be too conservative:
Tabular RL is a special case of linear mixture MDPs and here it is known by the \emph{law of total variance} \citep{lattimore2012pac, azar2013minimax} 
that the variance of the value function is  smaller than its magnitude by a factor $\sqrt{H}$. 
This inspires us to derive a Bernstein-type self-normalized concentration bound for vector-valued martingales which is sensitive to the variance of the martingale terms. 
Second, even if we were able to build such a tighter concentration bound, we still need to carefully design an algorithm because the variances of the value functions $\{\var_h\vvalue_{k, h+1}(s_h^k, a_h^k)\}_{h=1}^H$ at different stages  of the episodes are non-uniform: We face a so-called \emph{heteroscedastic} linear bandit problem. 
Naively choosing a uniform upper bound for all the variances $[\var_h\vvalue_{k, h+1}](s_h^k, a_h^k)$ yields no improvement compared with previous results. 
To address this challenge, we will need to build variance estimates and use these in a weighted least-squares estimator to achieve a better aggregation of the heteroscedastic data.

\subsection{A Bernstein-type Self-normalized Concentration Inequality for Vector-valued Martingales}

One of the key results of this paper is the following
Bernstein-type self-normalized concentration inequality:
\begin{theorem}[Bernstein inequality for vector-valued martingales]\label{lemma:concentration_variance}
Let $\{\cG_{t}\}_{t=1}^\infty$ be a filtration, $\{\xb_t,\eta_t\}_{t\ge 1}$ a stochastic process so that
$\xb_t \in \RR^d$ is $\cG_t$-measurable and $\eta_t \in \RR$ is $\cG_{t+1}$-measurable. 
Fix $R,L,\sigma,\lambda>0$, $\bmu^*\in \RR^d$. 
For $t\ge 1$ 
let $y_t = \la \bmu^*, \xb_t\ra + \eta_t$ and
suppose that $\eta_t, \xb_t$ also satisfy 
\begin{align}
    |\eta_t| \leq R,\ \EE[\eta_t|\cG_t] = 0,\ \EE [\eta_t^2|\cG_t] \leq \sigma^2,\ \|\xb_t\|_2 \leq L.\notag
\end{align}
Then, for any $0 <\delta<1$, with probability at least $1-\delta$ we have 
\begin{align}
    \forall t>0,\ \bigg\|\sum_{i=1}^t \xb_i \eta_i\bigg\|_{\Zb_t^{-1}} \leq \beta_t,\ \|\bmu_t - \bmu^*\|_{\Zb_t} \leq \beta_t + \sqrt{\lambda}\|\bmu^*\|_2,\label{eq:concentration_variance:xx}
\end{align}
where for $t\ge 1$,
 $\bmu_t = \Zb_t^{-1}\bbb_t$,
 $\Zb_t = \lambda\Ib + \sum_{i=1}^t \xb_i\xb_i^\top$,
$\bbb_t = \sum_{i=1}^ty_i\xb_i$
 and
 \[
\beta_t = 8\sigma\sqrt{d\log(1+tL^2/(d\lambda)) \log(4t^2/\delta)} + 4R \log(4t^2/\delta)\,.
\]
\end{theorem}
\begin{proof}
The proof adapts the proof technique of  \citet{dani2008stochastic}; for details
see Appendix \ref{sec:proof:concentration_variance}.
\end{proof}

Theorem \ref{lemma:concentration_variance} can be viewed as a non-trivial extension of the Bernstein concentration inequality from scalar-valued martingales to self-normalized vector-valued martingales. It is a strengthened version of self-normalized tail inequality for vector-valued martingales when the magnitude and the variance of the noise are bounded. It is worth to compare it with a few Hoeffding-Azuma-type results proved in prior work \citep{dani2008stochastic, rusmevichientong2010linearly, abbasi2011improved}.  In particular, \citet{dani2008stochastic} considered the setting where $\eta_t$ is $R$-bounded and showed that for large enough $t$, the following holds with probability at least $1-\delta$: 
\begin{align}
    \|\bmu_t - \bmu^*\|_{\Zb_t} \leq R\max\{\sqrt{128d\log (tL^2) \log(t^2/\delta)}, 8/3\cdot \log(t^2/\delta)\}.\notag
\end{align}
\citet{rusmevichientong2010linearly} considered a more general setting than \citet{dani2008stochastic} where $\eta_t$ is $R$-sub-Gaussian and showed that \eqref{eq:concentration_variance:xx} holds when $\beta_t = 2\kappa^2 R\sqrt{\log t}\sqrt{d\log t + \log(t^2/\delta)}$, where $\kappa = \sqrt{3+2\log (L^2/\lambda + d)}$. 
\citet{abbasi2011improved} considered the same setting as  \citet{rusmevichientong2010linearly} where $\eta_t$ is $R$-sub-Gaussian and showed that \eqref{eq:concentration_variance:xx} holds when $\beta_t = R\sqrt{d \log((1+tL^2/\lambda)/\delta)}$, which improves the bound of \citet{rusmevichientong2010linearly} in terms of logarithmic factors. By selecting proper $\lambda$, all these results yield an $\|\bmu_t - \bmu^*\|_{\Zb_t} = \tilde O(R\sqrt{d})$ bound. As a comparison, 
with the choice $\lambda = \sigma^2d/\|\bmu^*\|_2^2$,
our result gives 
\begin{align}
    \|\bmu_t - \bmu^*\|_{\Zb_t} =\tilde O(\sigma\sqrt{d} + R).\label{eq:compare}
\end{align}
Note that for any random variable, its standard deviation is always upper bounded by its magnitude or sub-Gaussian norm, therefore our result strictly improves the mentioned previous results.
This improvement is due to the fact that here we consider a subclass of sub-Gaussian noise variables
which allows us to derive a tighter upper bound. 
Indeed, Exercise~20.1 in the book of \citet{lattimore2018bandit} shows that the previous inequalities are tight in the worst-case for $R$-sub-Gaussian noise.


Even more closely related are results by 
 \citet{LaCrSze15,kirschner2018information} and \citet{faury2020improved}.
In all these papers the strategy is to use a weighted ridge regression estimator, which we will also make use of in the next section.
In particular, \citet{LaCrSze15} study the special case of Bernoulli payoffs. For this special case,
with our notation, they show a result implying that with high probability $\| \bmu_t - \bmu^* \|_{\Zb_t} = \tilde O( \sigma \sqrt{d} )$. The lack of the scale term $R$ is due to that Bernoulli's are single-parameter: The variance and the mean control each other, which the proof exploits. As such, this result does not lead in a straightforward way to ours, where the scale and variance are independently controlled. A similar comment applies to the result of \citet{kirschner2018information} 
who considered the case when the noise in the responses are sub-Gaussian. 

For the case of $R=1$, $L=1$ and $\EE [\eta_t^2|\cG_t] \leq \sigma_t^2$, the recent work 
of \citet{faury2020improved} also proposed a Bernstein-type concentration inequality (cf. Theorem 1 in their paper)
and showed that this gives
rise to better results in the context of logistic bandits. 
Their result can be extended to arbitrary $R$ and $L$ (see Appendix \ref{app:faura}), which gives that with high probability,
\begin{align}
    \|\bmu_t - \bmu^*\|_{\Zb_t} = \tilde O\big(\sigma\sqrt{d} + \sqrt{d\|\bmu^*\|_2RL}\big),\label{eq:faura_1}
\end{align} 
which has a polynomial dependence on $\|\bmu^*\|_2, R,L$, whereas in \eqref{eq:compare} the second term is only a function of $R$. This is a significant difference. In particular,
in the linear mixture MDP setting, we have $\sigma = \tilde O( \sqrt{H})$, $\|\bmu^*\|_2 = O( \pnorm)$, $R =O( H)$ and $L =O( H)$. 
Plugging these into both bounds, we see that our new result 
gives $\tilde O(\sqrt{dH} + H)$, while \eqref{eq:faura_1} gives the worse bound $\tilde O(\sqrt{dH} + \sqrt{d\pnorm}H)$. 
As it will be clear from the further details of our derivations given in 
Sections \ref{section:finite_main} and \ref{section 5}, as a result of the above difference,
their bound would not result in a minimax optimal bound on the regret in our setting. 

\subsection{Weighted Ridge Regression and Heteroscedastic Linear Bandits} 
In this subsection we consider the problem of linear bandits where the learner is given at the end of each round an upper bound on the (conditional) variance of the noise in the responses as input,
which is similar to the setting studied by
\citet{kirschner2018information}, where it is not the variance, but the sub-Gaussianity parameter that the learner observes at the end of the rounds.
The learner's goal is then to make use of this information to achieve a smaller regret as a function of the sum of squared variances (a ``second-order bound''). This is also related to the Gaussian side-observation setting and partial monitoring with feedback graphs considered in \citet{WGySz:NeurIPS15}. 
This abstract problem is studied to work out the tools needed to handle the heteroscedasticity of the noise that arises in the linear mixture MDPs in a cleaner setting.


In more details, 
let $\{\cD_t\}_{t=1}^\infty$ be a fixed sequence of decision sets. 
The agent selects an action $\ab_t \in \cD_t$ and then observes the reward $r_t = \la \bmu^*, \ab_t\ra + \epsilon_t$, where $\bmu^* \in \RR^d$ is a vector unknown to the agent and $\epsilon_t$ is a random noise satisfying the following properties almost surely:
\begin{align}
    \forall t,\ |\epsilon_t| \leq R,\ \EE[\epsilon_t|\ab_{1:t}, \epsilon_{1:t-1}] = 0,\ \EE [\epsilon_t^2|\ab_{1:t}, \epsilon_{1:t-1}] \leq \sigma_t^2,\ \|\ab_t\|_2 \leq A.\label{def:eee}
\end{align}
As noted above, the learner gets to observe $\sigma_t$ together with $r_t$ after each choice it makes.
We assume that $\sigma_t$ is $(\ab_{1:t}, \epsilon_{1:t-1})$-measurable.
The goal of the agent is to minimize its \emph{pseudo-regret}, defined as follows:
\begin{align}
    \text{Regret}(T) = \sum_{t=1}^T \la \ab_t^*, \bmu^*\ra - \sum_{t=1}^T \la \ab_t, \bmu^*\ra,\ \text{where}\ \ab_t^* = \argmax_{\ab \in \cD_t} \la \ab, \bmu^*\ra.\notag
\end{align}
To make use of the variance information, we propose a \emph{Weighted OFUL}, 
which is an extension of the ``Optimism in the Face of Uncertainty for Linear bandits'' algorithm (OFUL) of \citet{abbasi2011improved}. The algorithm's  pseudocode is shown in Algorithm \ref{algorithm:reweightbandit}. 

\begin{algorithm}[ht]
	\caption{Weighted OFUL}\label{algorithm:reweightbandit}
	\begin{algorithmic}[1]
	\REQUIRE Regularization parameter $\lambda>0$, 
	and $\pnorm$, 
	an upper bound on the $\ell_2$-norm of $\bmu^*$
	\STATE $\Ab_0 \leftarrow \lambda \Ib$, $\cbb_0 \leftarrow \zero$, $\hat\bmu_0 \leftarrow \Ab_0^{-1}\cbb_0$, $\hat\beta_0 = 0$, $\cC_0 \leftarrow \{\bmu: \|\bmu -\hat\bmu_0 \|_{\Ab_0} \leq \hat\beta_0 + \sqrt{\lambda}\pnorm\}$
	\FOR{$t=1,\ldots, T$}
	\STATE Observe $\cD_t$
	\STATE Let $(\ab_t, \tilde\bmu_t) \leftarrow \argmax_{\ab \in \cD_t, \bmu \in \cC_{t-1}} \la \ab, \bmu\ra$\label{alg:ofu}
	\STATE Select $\ab_t$ and observe $(r_t,\sigma_t)$, set $\bar\sigma_t$ based on $\sigma_t$, set radius $\hat\beta_t$ as defined in \eqref{eq:defbanditbeta} 
	\STATE $\Ab_t \leftarrow \Ab_{t-1} + \ab_t\ab_t^\top/\bar\sigma_t^2$, $\cbb_t \leftarrow \cbb_{t-1} + r_t\ab_t/\bar\sigma_t^2$, $\hat\bmu_t\leftarrow \Ab_t^{-1}\cbb_t$, $\cC_t \leftarrow \{\bmu: \|\bmu -\hat\bmu_t \|_{\Ab_t} \leq \hat\beta_t + \sqrt{\lambda} \pnorm\}$\label{alg:reweightbandit}
	\ENDFOR
	\end{algorithmic}
\end{algorithm}
In round $t$, Weighted OFUL selects the estimate $\hat\bmu_t$ of the unknown $\bmu^*$ as the minimizer to the following \emph{weighted ridge regression} problem:
\begin{align}
    \hat\bmu_t \leftarrow \argmin_{\bmu \in \RR^d} \lambda \|\bmu\|_2^2 + \sum_{i = 1}^t [\la \bmu, \ab_{i}\ra  - r_{i}]^2/\bar\sigma_{i}^2,\label{eq:reweightbandit}
\end{align}
where $\bar\sigma_i$ is a selected upper bound of $\sigma_i$. 
The closed-form solution to \eqref{eq:reweightbandit} is in Line \ref{alg:reweightbandit} of Algorithm \ref{algorithm:reweightbandit}. The term ``weighted'' refers to the normalization constant $\bar\sigma_i$ used in \eqref{eq:reweightbandit}. The estimator in \eqref{eq:reweightbandit} is closely related to the best linear unbiased estimator (BLUE) \citep{henderson1975best}. 
In particular, in the language of linear regression, 
with $\lambda=0$ and when $\bar\sigma_t^2$ is the variance of $r_t$, with a fixed design
$\hat\bmu_t$ is known to be the lowest variance estimator of $\bmu^*$ 
in the class of linear unbiased estimators. 
Note that both \citet{LaCrSze15} and \citet{kirschner2018information} used a similar weighted ride-regression estimator for their respective problem settings, mentioned in the previous subsection.

By adapting the new Bernstein-type self-normalized concentration inequality in Theorem \ref{lemma:concentration_variance}, we obtain the following bound on the regret of Weighted OFUL:
\begin{theorem}\label{coro:linearregret}
Suppose that for all $t \geq 1$ and all $\ab \in \cD_t$, $\la \ab, \bmu^*\ra \in [-1, 1]$, $\|\bmu^*\|_2 \leq \pnorm$
and $\bar \sigma_t \ge \sigma_t$. Then with probability at least $1-\delta$, the regret of Weighted OFUL for the first $T$ rounds is bounded as follows:
\begin{align}
    \text{Regret}(T) 
    &\leq 2\sqrt{2d \log(1+TA^2/(d\lambda[\bar\sigma^T_{\text{min}}]^2))}\sqrt{\sum_{t=1}^T (\hat\beta_{t-1} + \sqrt{\lambda}\pnorm)^2\bar\sigma_t^2} \notag \\
    &\quad + 4d \log\big(1+TA^2/(d\lambda[\bar\sigma^T_{\text{min}}]^2)\big),\label{eq:cororegret}
\end{align}
where $\bar\sigma^t_{\text{min}} = \min_{1\le i \le t} \bar\sigma_i$ and $\hat\beta_t$ is defined as
\begin{align}
    \hat\beta_0 = 0,\ \hat\beta_t = 
8\sqrt{d\log(1+tA^2/([\bar\sigma^t_{\text{min}}]^2 d\lambda)) \log(4t^2/\delta)} + 4R/\bar\sigma^t_{\text{min}}\cdot\log(4t^2/\delta),\ t \geq 1\,.\label{eq:defbanditbeta}
\end{align}
\end{theorem}
\begin{proof}
See Appendix \ref{sec:proof:linearregret}.
\end{proof}
\begin{corollary}\label{coro:banditcoro}
Let the same conditions as in Theorem \ref{coro:linearregret} hold and assume that Weighted OFUL is used with $\bar\sigma_t = \max\{R/\sqrt{d}, \sigma_t\}$ and $\lambda = 1/{\pnorm}^2$. Then with probability at least $1-\delta$, the regret of Weighted OFUL for the first $T$ rounds is bounded as follows:
\begin{align}
    \text{Regret}(T) = \tilde O\bigg(R\sqrt{dT} + d\sqrt{\sum_{t=1}^T \sigma_t^2}\bigg).\label{eq:cororegret11}
\end{align}
\end{corollary}
\begin{proof}
See Appendix \ref{sec:banditcoro}.
\end{proof}
\begin{remark}
    Comparing \eqref{eq:cororegret11} in Corollary \ref{coro:banditcoro} with the regret bound $\text{Regret}(T) = \tilde O(Rd\sqrt{T})$ achieved by OFUL in \citet{abbasi2011improved}, it can be seen that the regret of Weighted OFUL is strictly better than that of OFUL since $\sigma_t \leq R$. 
\end{remark}

\section{Optimal Exploration for Episodic MDPs}\label{section:finite_main}
In this section, equipped with new technical tools discussed in Section \ref{sec:linearbandit}, we propose a new algorithm $\algnamefin$ for episodic linear mixture MDPs (see Definition \ref{assumption-linear}). We also prove its near minimax optimality by providing matching upper and lower bounds. 
\subsection{The Proposed Algorithm}\label{sec:finitealg}
\begin{algorithm}[ht]
	\caption{$\algnamefin$ for Episodic Linear Mixture MDPs}\label{algorithm:finite}
	\begin{algorithmic}[1]
	\REQUIRE Regularization parameter $\lambda$, an upper bound $\pnorm$ of the $\ell_2$-norm of $\btheta_h^*$
	\STATE For $h \in [H]$, set $\hat \bSigma_{1, h} , \tilde \bSigma_{1, h} \leftarrow \lambda\Ib$, $\hat\bbb_{1,h}, \tilde\bbb_{1,h} \leftarrow \zero$,  $\hat\btheta_{1,h},\tilde\btheta_{1,h} \leftarrow\zero$, $\vvalue_{1, H+1}(\cdot) \leftarrow 0$
	\FOR{$k=1,\ldots, K$}
		\FOR{$h = H,\dots, 1$}
	\STATE $\qvalue_{k,h}(\cdot, \cdot)\leftarrow \min\Big\{H, \reward_h(\cdot, \cdot) + \big\la \hat\btheta_{k,h}, \bphi_{\vvalue_{k, h+1}}(\cdot, \cdot) \big\ra + \hat\beta_k \Big\|\hat \bSigma_{k,h}^{-1/2} \bphi_{\vvalue_{k, h+1}}(\cdot, \cdot)\Big\|_2\Big\}$, where $\hat\beta_k$ is defined in \eqref{def:hatbeta} \label{alg:ucbstart}
\STATE $\pi_h^{k}(\cdot) \leftarrow \argmax_{a \in \cA}\qvalue_{k,h}(\cdot, a)$\label{alg:ucbgreedy}
\STATE $\vvalue_{k,h}(\cdot) \leftarrow \max_{a \in \cA}\qvalue_{k,h}(\cdot, a)$\label{alg:ucbend}
	\ENDFOR
		\STATE	Receive $s_1^k$ 
	\FOR{$h = 1,\dots, H$}
	\STATE Take action $a_h^k  \leftarrow \pi_h^k(s_h^k)$, receive $s_{h+1}^k \sim \PP_h(\cdot|s_h^k, a_h^k)$
	\label{alg:ucbaction}
				\STATE  $[\bar\var_{k,h}\vvalue_{k,h+1}](s_h^k, a_h^k)$ as in \eqref{def:estvariance}, $\error_{k,h}$ as in \eqref{eq:def_error_finite}
					\label{alg:ucbvar}
	\STATE  $\bar\sigma_{k,h}\leftarrow \sqrt{\max\big\{H^2/d, [\bar\var_{k,h}\vvalue_{k, h+1}](s_h^k, a_h^k) + \error_{k,h}\big\}}$\label{alg:sigmak} \COMMENT{Variance upper bound}
		\STATE  $\hat\bSigma_{k+1,h}\leftarrow \hat\bSigma_{k,h} + \bar\sigma_{k,h}^{-2}\bphi_{{\vvalue}_{k,h+1}}(s_h^k,a_h^k)\bphi_{{\vvalue}_{k,h+1}}(s_h^k,a_h^k)^\top$\label{alg:covdefin}
		\COMMENT{``Covariance'', 1st moment}
	\STATE  $\hat\bbb_{k+1,h} \leftarrow \hat\bbb_{k,h} + \bar\sigma_{k,h}^{-2}\bphi_{\vvalue_{k,h+1}}(s_h^k, a_h^k)\vvalue_{k,h+1}(s_{h+1}^k)$
	\COMMENT{Response, 1st moment}
	\STATE  $\tilde\bSigma_{k+1,h} \leftarrow \tilde\bSigma_{k,h} + \bphi_{{\vvalue}_{k,h+1}^2}(s_h^k, a_h^k)\bphi_{{\vvalue}_{k,h+1}^2}(s_h^k, a_h^k)$
		\COMMENT{``Covariance'', 2nd moment}
	\STATE  $\tilde\bbb_{k+1,h} \leftarrow \tilde\bbb_{k,h} + \bphi_{{\vvalue}_{k,h+1}^2}(s_h^k, a_h^k)\vvalue_{k,h+1}^2(s_{h+1}^k)$
	\COMMENT{Response, 2nd moment}
	\STATE  $\hat\btheta_{k+1,h} \leftarrow \hat\bSigma_{k+1,h}^{-1}\hat\bbb_{k+1,h}$,  $\tilde\btheta_{k+1,h} \leftarrow \tilde\bSigma_{k+1,h}^{-1}\tilde\bbb_{k+1,h}$\label{alg:closeformtheta}
	\COMMENT{1st and 2nd moment parameters}
					\label{alg:ucbparams}
	\ENDFOR
	\ENDFOR
	\end{algorithmic}
\end{algorithm}

At a high level, $\algnamefin$ is an improved version of the UCRL with ``value-targeted regression'' (UCRL-VTR) algorithm by \citet{jia2020model} and refined and generalized by \citet{ayoub2020model}.
UCRL-VTR is an optimistic model-based method 
and as such UCRL-VTR keeps a confidence set for the models which are highly probable given the past data.
The unique distinguishing feature of UCRL-VTR is that this set is defined as those models that can accurately predict,
 along the transitions encountered,
the values of the optimistic value functions produced by the algorithm.

Our algorithm, $\algnamefin$, shares the basic structure of UCRL-VTR.
For the specific case of linear mixture MDPs, the confidence set $\hat\cC_{k,h}$ constructed is an ellipsoid in the parameter space, centered at
the parameter vector $\hat \btheta_{k,h}$ and shape given by the ``covariance'' matrix $\hat\bSigma_{k,h}$ and having a radius of $\hat \beta_k$:
\begin{align}
    \hat\cC_{k,h} = \bigg\{\btheta: \Big\|\hat\bSigma_{k,h}^{1/2}(\btheta - \hat\btheta_{k,h})\Big\|_2 \leq \hat\beta_k\bigg\},\label{eq:hatcc}
\end{align}
Following the optimism in the face of uncertainty principle, $\algnamefin$ then constructs 
an optimistic estimate of the optimal action-value function:
\begin{align}
    \qvalue_{k,h}(\cdot, \cdot)= \min\Big\{H, \reward_h(\cdot, \cdot) + \max_{\btheta \in \hat\cC_{k,h}}\big\la \btheta, \bphi_{\vvalue_{k, h+1}}(\cdot, \cdot) \big\ra \Big\}\,.\label{eq:innerproduct}
\end{align}
(Recall, that for simplicity, we assumed that the reward functions $\{r_h\}$ are known.)
Note that the definition 
of $\qvalue_{k,h}(\cdot, \cdot)$
allows the ``optimistic'' parameter be a function of the individual arguments of $Q_{k,h}$, just like in the work of \citet{yang2019reinforcement}.
The alternative of this is to choose a single optimistic parameter by maximizing the value at the initial state for the current episode as done by \citet{ayoub2020model}.
The above choice, as explained below in more details, can be computationally advantageous \citep{yang2019reinforcement,ayoub2020model,Neu2020-xp}, while our proof 
also shows that,  as far as minimax optimality of the regret is concerned, there is no loss of performance. 
Given the choice of $\hat \cC_{k,h}$, it is not hard to see that the update in Line \ref{alg:ucbstart}
is equivalent to \eqref{eq:innerproduct}.
Given $\{\qvalue_{k,h}\}$, in each episode $k$,
$\algnamefin$ executes takes actions that are greedy with respect to $\qvalue_{k,h}(\cdot, \cdot)$ 
(Line \ref{alg:ucbgreedy}). 

\paragraph{Weighted Ridge Regression and Optimistic Estimates of Value Functions} 
The key novelty of $\algnamefin$ is the use the weighted ridge regression (cf. Section \ref{sec:linearbandit}) to learn the underlying $\btheta_h^*$. 
To understand the mechanism behind $\algnamefin$, we first recall that for all $k,h$, $[\PP_h\vvalue_{k, h+1}](s_h^k, a_h^k)= \big\la\bphi_{\vvalue_{k, h+1}}(s_h^k, a_h^k), \btheta_h^*\big\ra$, which indicates that the expectation of $\vvalue_{k, h+1}(s_{h+1}^k)$ is a linear function of $\bphi_{\vvalue_{k, h+1}}(s_h^k, a_h^k)$ due to the definition of linear mixture MDPs. Therefore, $\vvalue_{k, h+1}(s_{h+1}^k)$ and $\bphi_{\vvalue_{k, h+1}}(s_h^k, a_h^k)$ can be regarded as the stochastic reward and context of a linear bandits problem.
Let $\sigma_{k,h}^2 = [\var_h\vvalue_{k, h+1}](s_h^k, a_h^k)$ be the variance of the value function. 
Then, the analysis in Section \ref{sec:linearbandit} 
suggests that one 
should use a weighted ridge regression estimator, such as
\begin{align}
    \hat\btheta_{k,h} = \argmin_{\btheta \in \RR^d}\lambda\|\btheta\|_2^2 + \sum_{j = 1}^{k-1} \big[\big\la\bphi_{\vvalue_{j, h+1}}(s_h^j, a_h^j), \btheta\big\ra - \vvalue_{j, h+1}(s_{h+1}^j)\big]^2/\bar\sigma_{j,h}^2, \label{eq:reweightedest}
\end{align}
where $\bar\sigma_{j,h}$ is an appropriate upper bound on $\sigma_{j,h}$. 
In particular, we construct $\bar\sigma_{k,h}$ as follows 
\begin{align}
    \bar\sigma_{k,h}= \sqrt{\max\big\{H^2/d, [\bar\var_{k,h}\vvalue_{k, h+1}](s_h^k, a_h^k) + \error_{k,h}\big\}},\notag
\end{align}
where $[\bar\var_{k,h}\vvalue_{k, h+1}](s_h^k, a_h^k)$ is a scalar-valued empirical estimate for the variance of the value function $\vvalue_{k, h+1}$
under the transition probability $\PP_h(\cdot|s_k,a_k)$, and $\error_{k,h}$ is an offset term that is used to guarantee
 $[\bar\var_{k,h}\vvalue_{k, h+1}](s_h^k, a_h^k) + \error_{k,h}$ upper bounds $\sigma_{k,h}^2$ with high probability. 
The detailed specifications of these are deferred later.
Moreover, by construction, we have $\bar\sigma_{k,h} \geq H/\sqrt{d}$. Our construction of $\bar\sigma_{k,h}$ shares a similar spirit as the variance estimator used in \emph{empirical Bernstein inequalities} \citep{audibert2009,maurer2009empirical}, which proved to be pivotal to achieve nearly minimax optimal sample complexity/regret in tabular MDPs \citep{azar2013minimax,azar2017minimax, zanette2019tighter, he2020minimax}. 

Several nontrivial questions remain to be resolved. First, we need to specify how to calculate the empirical variance $[\bar\var_{k,h}\vvalue_{k, h+1}](s_h^k, a_h^k)$. Second, in order to ensure $\qvalue_{k,h}(\cdot, \cdot)$ is an overestimate of $\qvalue_h^*(\cdot, \cdot)$, we need to choose an appropriate $\hat\beta_k$ such that $\hat\cC_{k,h}$ contain $\btheta_h^*$ with high probability. Third, we need to select $\error_{k,h}$ to guarantee $[\bar\var_{k,h}\vvalue_{k, h+1}](s_h^k, a_h^k) + \error_{k,h}$ upper bounds $\sigma_{k,h}^2$ with high probability. 

\paragraph{Variance Estimator}
To address the first question, we recall that by definition, we have
\begin{align}
    [\var_h\vvalue_{k, h+1}](s_h^k, a_h^k) &= [\PP_h\vvalue_{k, h+1}^2](s_h^k, a_h^k) - \big([\PP_h\vvalue_{k, h+1}](s_h^k, a_h^k)\big)^2\notag \\
    & = \big\la\bphi_{\vvalue_{k,h+1}^2}(s_h^k, a_h^k), \btheta_h^*\big\ra  - \big[\big\la \bphi_{\vvalue_{k,h+1}}(s_h^k, a_h^k), \btheta_h^*\big\ra\big]^2,\label{eq:help1}
\end{align}
where the second equality holds due to the definition of linear mixture MDPs. By \eqref{eq:help1} we conclude that the expectation of $\vvalue_{k, h+1}^2(s_{h+1}^k)$ over the next state, $s_{h+1}^k$, is a linear function of $\bphi_{\vvalue_{k,h+1}^2}(s_h^k, a_h^k)$. Therefore, we use $\la \bphi_{\vvalue_{k,h+1}}(s_h^k, a_h^k), \tilde\btheta_{k,h}\ra$ to estimate this term, where $\tilde\btheta_{k,h}$ is the solution to the following ridge regression problem: 
\begin{align}
    \tilde\btheta_{k,h} = \argmin_{\btheta \in \RR^d}\lambda\|\btheta\|_2^2 + \sum_{j = 1}^{k-1} \big[\big\la\bphi_{\vvalue_{j, h+1}^2}(s_h^j, a_h^j), \btheta\big\ra - \vvalue_{j, h+1}^2(s_{h+1}^j)\big]^2.\label{eq:secondest}
\end{align}
The closed-form solution to \eqref{eq:secondest} is in Line \ref{alg:closeformtheta}. In addition, we use $\la \bphi_{\vvalue_{k,h+1}}(s_h^k, a_h^k), \hat\btheta_{k,h}\ra$ to estimate the second term in \eqref{eq:help1}. Meanwhile, since $[\PP_h\vvalue_{k, h+1}^2](s_h^k, a_h^k) \in [0,H^2]$ and $[\PP_h\vvalue_{k, h+1}](s_h^k, a_h^k) \in [0,H]$ hold, 
we add clipping to control the range of our
variance estimator $[\bar\var_{k,h}\vvalue_{k, h+1}](s_h^k, a_h^k)$: 
\begin{align}
    [\bar\var_{k,h}\vvalue_{k,h+1}](s_h^k, a_h^k) =  \Big[\Big\la\bphi_{\vvalue_{k,h+1}^2}(s_h^k, a_h^k), \tilde\btheta_{k,h}\Big\ra\Big]_{[0, H^2]} -  \Big[\Big\la \bphi_{\vvalue_{k,h+1}}(s_h^k, a_h^k), \hat\btheta_{k,h}\Big\ra\Big]_{[0,H]}^2.\label{def:estvariance}
\end{align}
\begin{remark}
Currently $\algnamefin$ uses two estimate sequences $\check\btheta_{k,h}$ and $\tilde\btheta_{k,h}$ to estimate the first-order moment $\big\la\bphi_{\vvalue_{k, h+1}}(s_h^k, a_h^k), \btheta_h^*\big\ra$ and second-order moment $\big\la\bphi_{\vvalue_{k, h+1}^2}(s_h^k, a_h^k), \btheta_h^*\big\ra$ separately. We would like to point out that it is possible to use only one sequence to estimate both. Such an estimator can be constructed as a weighted ridge regression estimator based on both $\bphi_{\vvalue_{k, h+1}}(s_h^k, a_h^k)$'s and $\bphi_{\vvalue_{k, h+1}^2}(s_h^k, a_h^k)$, and the corresponding responses $\vvalue_{k, h+1}(s_{h+1}^k)$ and $\vvalue_{k, h+1}^2(s_{h+1}^k)$. However, since second-order moments generally have larger variance than the first-order moments, we need to use different weights for the square loss evaluated at $\big\{\bphi_{\vvalue_{k, h+1}}(s_h^k, a_h^k),\vvalue_{k, h+1}(s_{h+1}^k)\big\}$ and $\big\{\bphi_{\vvalue_{k, h+1}^2}(s_h^k, a_h^k),\vvalue_{k, h+1}^2(s_{h+1}^k)\big\}$. 
Also, by merging the data, 
even with using perfect weighting,
we would expect to win at best a (small) constant factor on the regret 
since the effect of not merging the data can be seen as not worse than throwing away ``half of the data''.
As a result, for the sake of simplicity, we chose to use two estimate sequences instead of one in our algorithm.
\end{remark}

In what follows, let $\PP$ denote the probability distribution over the infinite sequences of state-action pairs obtained from using $\algnamefin$ in the MDP $M$. All probabilistic statements will refer to this distribution $\PP$.

\paragraph{Confidence Set} To address the choice of $\hat\beta_k$ and $\error_{k,h}$, we need the following key technical lemma:
\begin{lemma}\label{thm:concentrate:finite}
Let $\hat\cC_{k,h}$ be defined in \eqref{eq:hatcc} and set $\hat\beta_k$ as
\begin{align}
\hat\beta_k= 8\sqrt{d\log(1+k/\lambda) \log(4k^2H/\delta)}+ 4\sqrt{d} \log(4k^2H/\delta) + \sqrt{\lambda}\pnorm\,.\label{def:hatbeta}
\end{align}
Then, with probability at least $1-3\delta$, we have that simultaneously for all $k\in [K]$ and $h\in [H]$,
\begin{align}
    \btheta^*_h \in \hat \cC_{k,h},\ |[\bar\var_{k,h}\vvalue_{k,h+1}](s_h^k, a_h^k) - [\var_h\vvalue_{k,h+1}](s_h^k, a_h^k)| \leq \error_{k,h},\notag
\end{align}
where $\error_{k,h}$ is defined as follows:
\begin{align}
	    \error_{k,h} &=   \min\Big\{H^2,2H\check\beta_k\Big\|\hat\bSigma_{k,h}^{-1/2}\bphi_{\vvalue_{k,h+1}}(s_h^k, a_h^k)\Big\|_2\Big\} + \min\Big\{H^2, \tilde\beta_k\Big\|\tilde\bSigma_{k,h}^{-1/2}\bphi_{\vvalue_{k,h+1}^2}(s_h^k, a_h^k)\Big\|_2\Big\},\label{eq:def_error_finite}
	\end{align}
where 
\begin{align*}
	        \check\beta_k &= 8d\sqrt{\log(1+k/\lambda) \log(4k^2H/\delta)}+ 4\sqrt{d} \log(4k^2H/\delta) + \sqrt{\lambda}\pnorm,\notag \\
    \tilde\beta_k &= 8\sqrt{dH^4\log(1+kH^4/(d\lambda)) \log(4k^2H/\delta)}+ 4H^2 \log(4k^2H/\delta) + \sqrt{\lambda}\pnorm.\notag  
\end{align*}	
\end{lemma}
\begin{proof}
See Appendix \ref{sec:proof:concentrate:finite}.
\end{proof}

Lemma \ref{thm:concentrate:finite} shows that with high probability, for all stages $h$ and episodes $k$,
$\btheta_h^*$ lies in the confidence set centered at its estimate $\hat\btheta_{k,h}$, and the error between the estimated variance and the true variance is bounded by the offset term $\error_{k,h}$. Equipped with Lemma \ref{thm:concentrate:finite}, we can verify the following facts. First, since $\btheta_h^* \in \hat\cC_{k,h}$, it can be easily verified that $\big\la \hat\btheta_{k,h}, \bphi_{\vvalue_{k, h+1}}(\cdot, \cdot) \big\ra + \hat\beta_k \big\|\hat \bSigma_{k, h}^{-1/2} \bphi_{\vvalue_{k, h+1}}(\cdot, \cdot)\big\|_2 \geq \big\la \btheta_h^*, \bphi_{\vvalue_{k, h+1}}(\cdot, \cdot) \big\ra = [\PP_h \vvalue_{k, h+1}](\cdot, \cdot)$, which shows that our constructed $\qvalue_{k,h}(\cdot, \cdot)$ in Line \ref{alg:ucbstart} is indeed an overestimate of $\qvalue_h^*(\cdot, \cdot)$. Second, recalling the definition of $\bar\sigma_{k,h}$ defined in Line \ref{alg:sigmak}, since $\big|[\bar\var_{k,h}\vvalue_{k,h+1}](s_h^k, a_h^k) - [\var_h\vvalue_{k,h+1}](s_h^k, a_h^k)\big| \leq \error_{k,h}$, we have $\bar\sigma_{k,h}^2 \geq [\bar\var_{k,h}\vvalue_{k, h+1}](s_h^k, a_h^k) + \error_{k,h} \geq [\var_h\vvalue_{k, h+1}](s_h^k, a_h^k)$, which shows that $\bar\sigma_{k,h}$ is indeed an overestimate of the true variance $[\var_h\vvalue_{k, h+1}](s_h^k, a_h^k)$. 

\newcommand{\tim}{\cY}
\paragraph{Computational Efficiency} Similar to UCRL-VTR \citep{ayoub2020model}, the computational complexity of $\algnamefin$ depends on the specific family of feature mapping $\bphi(\cdot|\cdot,\cdot)$. As an example, let us consider a special class of linear mixture MDPs studied by \citet{yang2019reinforcement, zhou2020provably}.
In this setting, $\bphi(s'|s,a) = \bpsi(s')\odot\bmu(s,a)$, $\bpsi(\cdot): \cS \rightarrow \RR^d$ and $\bmu(\cdot, \cdot):\cS \times \cA \rightarrow \RR^d$ are two features maps and $\odot$ denotes componentwise product. Recall that, by assumption, the action space $\cA$ is finite. 

We now argue that $\algnamefin$ is computationally efficient for this class of MDPs as long as we have access to either an integration oracle $\cO$ underlying the basis kernels. 
In particular, the assumption is that $\sum_{s'} \psi(s') V(s')$ can be evaluated at the cost of evaluating $V$ at $p(d)$ states with some polynomial $p$.
Now, for $1\le h \le H$, $\btheta \in \RR^d$ and $\bSigma\in \RR^{d\times d}$ let
\begin{align}
    Q_{h,\btheta,\bSigma}(\cdot,\cdot) 
    = \min\big\{H, \reward_h(\cdot ,\cdot) + \la \btheta, \bmu(\cdot,\cdot)\ra + \|\bSigma \bmu(\cdot ,\cdot)\|_2\big\}.\notag
\end{align}
It is easy to verify that for 
any $k,h$, $Q_{k,h}= Q_{h,\btheta_{k,h},\bSigma_{k,h}}$ 
where $\btheta_{k,h} = \hat\btheta_{k,h}\odot [\sum_{s'}\bpsi(s')\vvalue_{k, h+1}(s')]$ 
and the $(i,j)$-th entry of ${\bSigma}_{k,h}$ 
is $ \hat\beta_k (\hat\bSigma_{k, h}^{-1/2})_{i,j} [\sum_{s'}\bpsi_j(s')\vvalue_{k, h+1}(s')]$.
Now notice that  $\btheta_{k,H}=\mathbf{0}$, $\bSigma_{k,H} = \mathbf{0}$.
Thus, for $1\le h \le H-1$, assuming that 
$\btheta_{k,h+1}$ and $\bSigma_{k,h+1}$ have been calculated,
evaluating $V_{k,h+1}$ at any state $s\in \cS$ costs $O(d^2|\cA|)$ arithmetic operations.
Now, calculating $\btheta_{k,h}$ and $\bSigma_{k,h}$ costs $O(d^2)$ arithmetic operations 
given access $\hat\btheta_{k,h}$ and $\hat\bSigma_{k,h}^{-1/2}$,
in addition to $p(d)$ evaluations of $V_{k,h+1}$.
Since each evaluation of $V_{k,h+1}$ takes $O(d^2|\cA|)$ operations, as established, 
calculating 
$\btheta_{k,h}$ and $\bSigma_{k,h}$ cost a total of $O( p(d)d^2|\cA|)$ operations.
From this, it is clear that calculating the $H$ actions to be taken in episode $k$ takes a total of 
$O( p(d)d^2|\cA|H)$ operations (Line~\ref{alg:ucbaction}).
It also follows that calculating 
either $\bphi_{V_{k,h+1}}$ 
or $\bphi_{V^2_{k,h+1}}$ 
at any state-action pair costs $O(p(d)d^2|\cA|)$ operations.

To calculate the quantities appearing in Lines~\ref{alg:ucbvar}--\ref{alg:ucbparams}, 
first $\bphi_{V_{k,h+1}}(s_h^k,a_h^k)$
and $\bphi_{V^2_{k,h+1}}(s_h^k,a_h^k)$ ($h\in[H]$)
are evaluated at the cost of $O(p(d)d^2|\cA|H)$.
It is then clear that the rest of the calculation costs at most $O(d^3H)$:
the most expensive step is to obtain $\hat \bSigma_{k,h}^{-1/2}$ (the cost could be reduced to $O(d^2H)$ by using the matrix inversion lemma and organizing the calculation of $Q_{k,h}$ slightly differently).
It follows that the total computational complexity of $\algnamefin$ is $O(\text{poly}(d)|\cA|HK) = O(\text{poly}(d)|\cA|T)$. For many other MDP models, $\algnamefin$ can still be computationally efficient. Please refer to \citet{ayoub2020model} for a detailed discussion.

\subsection{Regret Upper Bound}
Now we present the regret upper bound of $\algnamefin$. 
\begin{theorem}\label{thm:regret:finite}
Set $\lambda = 1/B^2$. Then with probability at least $1-5\delta$, the regret of $\algnamefin$ on MDP $M_{\bTheta^*}$ is upper bounded as follows:
\begin{align}
    &\text{Regret}\big(M_{\bTheta^*}, K\big) = \tilde O\Big(\sqrt{d^2H^2+dH^3}\sqrt{T}+ d^2H^3 + d^3H^2\Big),\ T = KH.\label{eq:ppp}
\end{align}
\end{theorem}
\noindent We provide a proof sketch here and defer the full proof to Appendix \ref{sec:proof:regret:finite}. 
\begin{proof}[Proof sketch]
By Lemma \ref{thm:concentrate:finite}, it suffices to prove the result on the event $\event$ when the conclusions of this lemma hold. Hence, in what follows assume that this event holds.
By using the standard regret decomposition and using the definition of the confidence sets $\{\hat\cC_{k,h}\}$, 
we can show that the total regret is bounded by the summation of the bonus terms, $\sum_{k=1}^K\sum_{h=1}^H\hat\beta_k \big\|\hat \bSigma_{k, h}^{-1/2} \bphi_{\vvalue_{k, h+1}}(s_h^k, a_h^k)\big\|_2$, which, by the Cauchy-Schwarz inequality, can be further bounded by, 
\begin{align}
    \hat\beta_K\sqrt{dH\sum_{k=1}^K\sum_{h=1}^H \bar\sigma_{k,h}^2}.\label{sketch:1}
\end{align}
Finally, we have $\bar\sigma_{k,h}^2 \leq H^2/d+ \error_{k,h}+[\bar\var_{k,h}\vvalue_{k, h+1}](s_h^k, a_h^k)  \leq H^2/d+ 2\error_{k,h}+ [\var_h\vvalue_{k, h+1}](s_h^k, a_h^k)$. 
Therefore the summation of $\bar\sigma_{k,h}^2$ can be bounded as 
\begin{align}
    \sum_{k=1}^K\sum_{h=1}^H\bar\sigma_{k,h}^2 &\leq H^3K/d + 2\sum_{k=1}^K\sum_{h=1}^H \error_{k,h}+ \sum_{k=1}^K\sum_{h=1}^H[\var_h\vvalue_{k, h+1}](s_h^k, a_h^k)\notag \\
    &= \tilde O(HT +H^2T/d+ dH^3\sqrt{T}),\label{sketch:2}
\end{align}
where the equality holds since $\sum_{k=1}^K\sum_{h=1}^H[\var_h\vvalue_{k, h+1}](s_h^k, a_h^k) = \tilde O(HT)$ by the \emph{law of
total variance} \citep{lattimore2012pac, azar2013minimax}, and  $\sum_{k=1}^K\sum_{h=1}^H \error_{k,h} = \tilde O(dH^3\sqrt{T} + d^{1.5}H^{2.5}\sqrt{T})$ by the elliptical potential lemma. Substituting \eqref{sketch:2} into \eqref{sketch:1} completes our proof. 
\end{proof}
\begin{remark}
When $d \geq H$ and $T \geq d^4H^2 + d^3H^3$, the regret in \eqref{eq:ppp} can be simplified as $\tilde O(dH\sqrt{T})$. Compared with the regret $\tilde O(dH^{3/2}\sqrt{T})$ of UCRL-VTR in \citet{jia2020model, ayoub2020model}\footnote{\citet{jia2020model, ayoub2020model} report a regret of order $\tilde O(dH\sqrt{T})$. However, these works considered the time-homogeneous case where $\PP_1 = \cdots = \PP_H$.
In particular, in the time-homogeneous setting parameters are shared between the stages of an episode, and this reduces the regret.
When UCRL-VTR is modified for the inhomogenous case, the regret picks up an additional $\sqrt{H}$ factor. Similar observation 
has also been made by \citet{jin2018q}.}, the regret of $\algnamefin$ is improved by a factor of $\sqrt{H}$.
\end{remark}
\begin{remark}\label{rmk:lower}
Our result actually only needs a weaker assumption on reward functions $\reward_h$ such that for any policy $\pi$, we have $0 \leq \sum_{h=1}^H \reward_h(s_h, a_h) \leq H$, where $a_h = \pi_h(s_h),\ s_{h+1} \sim \PP(\cdot|s_h, a_h)$. Therefore, under the assumption $0 \leq \sum_{h=1}^H \reward_h(s_h, a_h) \leq 1$ studied in \citet{dann2015sample,jiang2018open, wang2020long, zhang2020reinforcement}, by simply rescaling all parameters in Algorithm \ref{algorithm:finite} by a factor of $1/H$, $\algnamefin$ achieves the regret $\tilde O(\sqrt{d^2+dH}\sqrt{T}+ d^2H^2 + d^3H)$. 
\citet{zhang2020reinforcement} has shown that in the tabular, \emph{homogeneous} case with this normalization the regret is $\tilde O(\sqrt{T})$, regardless of the value of $H$. It remains an interesting open question whether this can be also achieved in homogeneous linear mixture MDPs. 
\end{remark}

\subsection{Lower Bound}\label{sec:lowerbound}
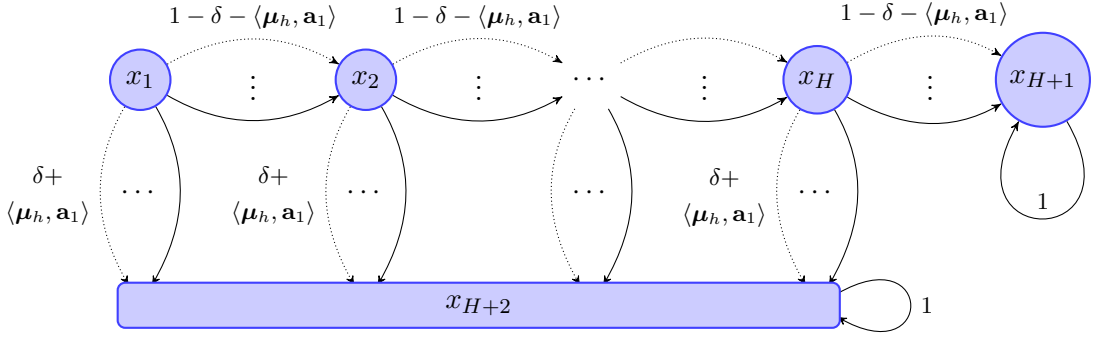
\begin{figure}[!h]
    \centering
    \begin{tikzpicture}[node distance=1.5cm,>=stealth',bend angle=45,auto]

  \tikzstyle{place}=[circle,thick,draw=blue!75,fill=blue!20,minimum size=6mm]
  \tikzstyle{hold}=[circle,draw=white,fill=white,minimum size=6mm]
  \tikzstyle{absorb}=[rectangle,rounded corners=.1cm, thick,draw=blue!75,fill=blue!20,minimum height=6mm,minimum width=9.6cm]
\tikzset{every loop/.style={min distance=15mm, font=\footnotesize}}
  \begin{scope}[xshift=-3.5cm]
    
    \node [place, label] (c0){$\state_{1}$};
    \coordinate [right of = c0, label = center:{$\vdots$}] (c1) {};
    
    \coordinate [below of = c0, label = center:{$\cdots$}] (d00) {};
    \node [hold, label, below of = d00] (d01){};
    
    \node [place] (c21) [right of=c1]{$\state_{2}$};
    \path[->] (c0)
    edge [in=150,out=30, densely dotted] node[above]{\footnotesize{$1-\delta - \la \bmu_h, \ab_1\ra$}} (c21)
    edge [in=-150,out=-30] node[below]{
    } (c21)
    edge [in=120,out=-120, densely dotted] node[left, align = center]{\footnotesize{$\delta + $} \\ \footnotesize{$\la \bmu_h, \ab_1\ra$}} (d01)
    edge [in=60,out=-60] node[left]{
    } (d01)
    ;
    
    \coordinate [below of = c21, label = center:{$\cdots$}] (d10) {};
    \node [hold, label, below of = d10] (d11){};

    \coordinate [right of = c21, label = center:{$\vdots$}] (c2) {};
    
    \node [hold] (c3) [right of=c2]{$\cdots$};
    \coordinate [right of = c3, label = center:{$\vdots$}] (cx) {};
    \coordinate [below of = c3, label = center:{$\cdots$}] (d30) {};
    \node [hold, label, below of = d30] (d31){};
    
    \path[->] (c3)
    edge [in=120,out=-120, densely dotted] node[left, align = center]{} (d31)
    edge [in=60,out=-60] node[right]{} (d31)
    ;
    
    \path[->] (c21)
    edge [in=150,out=30, densely dotted] node[above]{\footnotesize{$1-\delta - \la \bmu_h, \ab_1\ra$}} (c3)
    edge [in=-150,out=-30] node[below]{
    } (c3)
    edge [in=120,out=-120, densely dotted] node[left, align = center]{\footnotesize{$\delta + $} \\ \footnotesize{$\la \bmu_h, \ab_1\ra$}} (d11)
    edge [in=60,out=-60] node[right]{} (d11)
    ;

    \node [place, right of=cx, label] (c31){$\state_{H}$};
    \path[->] (c3)
    edge [in=150,out=30, densely dotted] node[above]{} (c31)
    edge [in=-150,out=-30] node[below]{
    } (c31)
    ;

        \coordinate [below of = c31, label = center:{$\cdots$}] (d20) {};
    \node [hold, label, below of = d20] (d21){};
    \path[->] (d21)
    edge [in=-30,out=30, min distance = 6cm, loop] node[right] {$1$} ()
    ;

    \coordinate [right of = c31, label = center:{$\vdots$}] (c32) {};
    \node [place] [right of=c32] (c4){$\state_{H+1}$};
    \path[->] (c4)
    edge [in=240,out=300, min distance = 6cm, loop] node[above] {$1$} ()
    ;
    \path[->] (c31)
    edge [in=150,out=30, densely dotted] node[above]{\footnotesize{$1-\delta - \la \bmu_h, \ab_1\ra$}} (c4)
    edge [in=-150,out=-30] node[below]{
    } (c4)
        edge [in=120,out=-120, densely dotted] node[left, align = center]{\footnotesize{$\delta + $}\\ \footnotesize{$\la \bmu_h, \ab_1\ra$}} (d21)
    edge [in=60,out=-60] node[right, align = center]{
    } (d21)
    ;
    
    \node [absorb, label, right of = d11] (aa){$\state_{H+2}$};
  \end{scope}

\end{tikzpicture}
\vspace*{-1mm}
    \caption{
    The transition kernel $\PP_h$ of the 
    class of hard-to-learn linear mixture MDPs. 
    The kernel $\PP_h$ is parameterized by $\bmu_h\in \{-\Delta,\Delta\}^{d-1}$ for some small $\Delta$,
    $\delta=1/H$ and the actions are from $\ab\in \{+1,-1\}^{d-1}$.
	The learner knows this structure, but does not know $\bmu = (\bmu_1,\dots,\bmu_H)$.
	}
    \label{fig:hardmdp}
\end{figure}

In this subsection, we present a lower bound for episodic linear mixture MDPs, which shows the optimality of $\algnamefin$.  

\begin{theorem}\label{thm:lowerbound:finite}
Let $B>1$ and
suppose $K \geq \max\{(d-1)^2H/2, (d-1)/(32H(\pnorm - 1))\}$, $d \geq 4$, $H\geq 3$. Then for any algorithm there exists an episodic, $B$-bounded linear mixture MDP parameterized by $\bTheta=(\btheta_1,\dots, \btheta_H)$ such that the expected regret is lower bounded as follows: 
\begin{align}
     \EE_{\bTheta}\text{Regret}\big( M_{\bTheta}, K\big) \geq \Omega\big(dH\sqrt{T}\big),\notag
\end{align}
where $T= KH$ and $\EE_{\bTheta}$ denotes the expectation over 
the probability distribution generated by the interconnection of the algorithm and the MDP.
\end{theorem}
The detailed proof is given in Appendix \ref{sec:proof:lowerbound:finite}; here we provide a proof sketch.
\begin{proof}[Proof sketch]
To prove the lower bound, we construct a hard instance $M(\cS, \cA, H, \{\reward_h\}, \{\PP_h\})$ based on the hard-to-learn MDPs introduced in \citet{zhou2020provably}. 
The transitions for stage $h$ of the MDP are shown in Figure~\ref{fig:hardmdp}.
The state space $\cS$ consists of states $\state_1,\dots \state_{H+2}$, where $\state_{H+1}$
and $\state_{H+2}$  are absorbing states. 
There are $2^{d-1}$ actions and $\cA = \{-1,1\}^{d-1}$. 
Regardless of the stage $h\in [H]$,
no transition incurs a reward except transitions originating at $\state_{H+2}$, which, as a result, can be regarded as the goal state.
Under $\PP_h$, the transition structure is as follows:
As noted before, $x_{H+1}$ and $x_{H+2}$ are absorbing regardless of the action taken.
If the state is $x_i$ with $i\le H$, 
under action $\ba \in \{-1,1\}^{d-1}$,
the next state is either $x_{H+2}$ or $x_{i+1}$, with respective probabilities $\delta+\ip{\bmu_h,\ba}$ and
$1-(\delta+\ip{\bmu_h,\ba})$, where $\delta=1/H$ and $\bmu_h\in \{-\Delta,\Delta\}^{d-1}$ with  $\Delta = \sqrt{\delta/K}/(4\sqrt{2})$ so that the probabilities are well-defined.

This is an inhomogeneous, linear mixture MDP.
In particular,
$\PP_h(s'|s,\ab) = \la \bphi(s'|s,\ab), \btheta_h\ra$, with
\begin{align*}
 \bphi(s'|s,\ab) &= 
\begin{cases}
     (\alpha(1-\delta), -\beta\ab^\top)^\top, &s = \state_h, s' = \state_{h+1}, h\in[H]\,;\\
   (\alpha\delta, \beta\ab^\top)^\top, &s = \state_h, s' = \state_{H+2}, h\in[H]\,;\\
    (\alpha,\zero^\top)^\top, &s \in \{ \state_{H+1}, \state_{H+2}\}, s' = s\,; \\
    \zero, &\text{otherwise}\,.
    \end{cases},\ \btheta_h = (1/\alpha, \bmu_h^\top/\beta)^\top,\ h\in[H],
\end{align*}
where $\alpha = \sqrt{1/(1+\Delta(d-1))},\ \beta =\sqrt{\Delta/(1+\Delta(d-1))}$.
It can be verified that $\bphi(\cdot|\cdot, \cdot)$ and $\{\btheta_h\}$ satisfy the requirements of a $\pnorm$-bounded linear mixture MDPs. 
In particular, \eqref{def:bbbphi} holds. Indeed, if we let $\vvalue:\cS \rightarrow [0,1]$ be any bounded function then for $s=x_{H+1}$ or $s=x_{H+2}$, $\bphi_{V}(s,\ba)=\sum_{s'} \bphi(s'|s,\ba) V(s')= (\alpha V(s),\zero^\top)^\top$ and hence $\norm{\bphi_{V}(s,\ba)}_2\le 1$, while for $s=x_h$ with $h\in [H]$, we have
\begin{align}
    \norm{\bphi_V(s,\ba)}_2^2 &=\alpha^2(V(x_{H+2})\delta + V(x_{h+1})(1-\delta))^2+ \beta^2(V(x_{H+2})-V(x_{h+1}))^2 \norm{\ba}_2^2 \notag \\
    & \leq \alpha^2 + (d-1)\beta^2\notag \\
    & = 1.
\end{align}
 Meanwhile, since $K \geq (d-1)/(32H(\pnorm - 1))$, we have 
 \begin{align}
     \|\btheta_h\|_2^2 = \frac{1}{\alpha^2} + \frac{\|\bmu_h\|_2^2}{\beta^2} =(1+\Delta(d-1))^2  = (1+\sqrt{\delta/K}/4\sqrt{2}\cdot(d-1))^2\leq \pnorm^2.\notag
 \end{align}
The initial state in each episode $k$ is $s_{k,1} = \state_1$. 
Note that if the agent transitions to $x_{H+2}$ it remains there until the end of the episode.
Due to the special structure of the MDP,
at any stage $h\in [H]$, either the state is $x_{H+2}$ or it is $x_h$.
Further, state $x_h$ can only be reached one way, through states $x_1$, $x_2$, $\dots$, $x_{h-1}$.
As such, knowing the current state is equivalent to knowing the history from the beginning of the episode and hence 
policies that simply decide at the beginning of the episode what actions to take upon reaching a state are as powerful as those that can use the ``within episode'' history. 

Now, clearly, since the only rewarding transitions are those from $x_{H+2}$, the optimal strategy 
in stage $h$ when in state $x_h$
is to take action $\argmax_{\ba \in \cA} \ip{ \bmu_h,\ba}$. 
Intuitively, the learning problem is not \emph{harder} than minimizing the regret on $H$ linear bandit problems with a shared action set $\cA = \{-1,+1\}^{d-1}$ and where the payoff on bandit $h\le H/2$ of taking action $\ba \in \cA$ is $\Omega(H)B$, where $B$ is drawn from a Bernoulli with parameter $\delta+\ip{\bmu_h,\ba}$. 
Some calculation shows that the reverse is also true: 
Thanks to the choice of $\delta$, $(1-\delta)^{H/2} \approx \text{const}$, hence
there is sufficiently high probability of reaching all stages including stage $H/2$, even under the optimal policy.
Hence, the MDP learning problem is not easier than solving the first $\Omega(H/2)$ bandit problems.
Choosing $\Delta = \Theta(\sqrt{\delta/K})$, for $K$ large enough, $(d-1)\Delta \le \delta$ so the probabilities are well defined.
Furthermore, on each of the bandit, the regret is at least $\Omega( d H\sqrt{K \delta} )$.
Since there are $\Omega(H/2)$ bandit problems, plugging in the choice of $\delta$,
we find that the total regret is $\Omega( d H \sqrt{KH} )$ and the result follows by noting that $T=KH$.
\end{proof}

\begin{remark}
Theorem \ref{thm:lowerbound:finite} shows that for any algorithm running on episodic linear mixture MDPs, its regret is lower bounded by $\Omega(dH\sqrt{T})$. The lower bound together with the upper bound of $\algnamefin$ in Theorem \ref{thm:regret:finite} shows that $\algnamefin$ is minimax optimal up to logarithmic factors. 
\end{remark}

\begin{remark}
Our lower bound analysis can be adapted to prove a lower bound for linear MDPs proposed in \citep{yang2019sample, jin2019provably}. In specific, based on our constructed linear mixture MDP $M$ in the proof sketch of Theorem \ref{thm:lowerbound:finite}, we can construct a linear MDP $\bar M(\cS, \cA, H, \{\bar\reward_h\}, \{\bar\PP_h\})$ as follows. For each stage $h \in [H]$, the transition probability kernel $\bar\PP_h$ and the reward function $\bar\reward_h$ are defined as $\bar\PP_h(s'|s,\ab) = \la \bphi(s,\ab), \bmu_h(s')\ra$ and $\bar\reward_h(s,\ab) = \la \bphi(s,\ab), \bxi_h\ra$, where $\bphi(s,a), \bmu(s') \in \RR^{d+1}$ are two feature mappings, and $\bxi_h \in \RR^{d+1}$ is a parameter vector. Here, we choose $\bphi(s,\ab), \bmu_h(s'), \bxi_h \in \RR^{d+1}$ as follows:
\begin{align*}
 &\bphi(s,\ab) = 
\begin{cases}
     (\alpha, \beta\ab^\top,0)^\top, &s = \state_h,\ h\in[H+1];\\
    (0,\zero^\top, 1)^\top, &s = \state_{H+2}.
\end{cases},\,
    \bmu_h(s') = 
    \begin{cases}
    ((1-\delta)/\alpha, -\bmu_h^\top/\beta,0)^\top,&s' = \state_{h+1};\notag \\
    (\delta/\alpha, \bmu_h^\top/\beta,1)^\top,&s' = \state_{H+2};\notag\\
    \zero,&\text{otherwise},\notag
    \end{cases}
\end{align*}
and $\bxi_h = (\zero^\top, 1)^\top$. It can be verified that  $\max\{\|\bxi_h\|_2, \|\bmu_h(\cS)\|_2\} \leq \sqrt{d+1}$, and $\|\bphi(s,\ab)\|_2 \leq 1$ for any $(s,a) \in \cS \times \cA$. 
In addition, for any $h \in [H]$, we have $\PP_h(s'|s,a) =\bar\PP_h(s'|s,a)$ and $\reward_h(s,a) = \bar\reward_h(s,a)$ when $s = \state_h$ or $\state_{H+2}$.  Since at stage $h$, $s$ can be either $\state_h$ or $\state_{H+2}$, we can show that the constructed linear MDP $\bar M$ has the same transition probability as the the linear mixture MDP $M$, which suggests the same lower bound $\Omega(dH\sqrt{T})$ in Theorem \ref{thm:lowerbound:finite} also holds for linear MDP. 
\end{remark}

\section{Optimal Exploration for Discounted MDPs} \label{section 5}

In this section, we consider infinite-horizon discounted linear mixture MDPs. We propose the $\algnamedis$ algorithm, which is a counterpart of $\algnamefin$ for the discounted MDP setting. We also provide a regret analysis, which shows that the algorithm, apart from logarithmic factors, 
is nearly minimax optimal.

\subsection{Proposed Algorithm}
The details of $\algnamedis$ are described in Algorithm \ref{algorithm}. Recall that $\ig = 1/(1-\gamma)$ is the effective horizon length of the discounted MDP. $\algnamedis$ is built upon the previous algorithm UCLK proposed by \citet{zhou2020provably}, with a modified bonus term. Just like UCLK, $\algnamedis$ is a multi-epoch algorithm inspired by \citet{jaksch2010near, abbasi2011improved, lattimore2012pac}. The $k$-th epoch of $\algnamedis$ starts at round $t_k$ and ends at round $t_{k+1}-1$. The length of each epoch is not prefixed but depends on previous observations. In each epoch, $\algnamedis$ maintains a set of plausible MDPs through confidence set $\hat\cC_{t}$ which includes the underlying $\btheta^*$ with high probability, then computes the action-value function corresponding to the near-optimal MDP among all the plausible MDPs by calling the Extended Value Iteration ($\valueite$) sub-procedure in Algorithm \ref{algorithm:2}. At a high level, EVI provides the optimistic estimate of $\qvalue^*(\cdot, \cdot)$ within a given confidence set $\cC$, similar to the optimistic value function estimator in Line \ref{alg:ucbstart} of Algorithm \ref{algorithm:finite}. More details about EVI can be found in \citet{lattimore2012pac, dann2015sample, zhou2020provably}. 
The multi-epoch structure and the use of EVI are the main differences between $\algnamedis$ and its episodic counterpart $\algnamefin$ proposed in Section \ref{section:finite_main}.

\begin{algorithm}[t]
	\caption{$\algnamedis$ for Discounted Linear Mixture MDPs}\label{algorithm}
	\begin{algorithmic}[1]
	\REQUIRE Regularization parameter $\lambda$, an upper bound $\pnorm$ of the $\ell_2$-norm of $\btheta^*$, number of value iteration rounds $U$ 
	\STATE	Receive $s_1$, set $k\leftarrow 0, t_0 \leftarrow 1$
	\STATE Set  
	$ \hat\bSigma_1, \tilde\bSigma_1 \leftarrow \lambda\Ib$, $\hat\bbb_1, \tilde\bbb_1, \hat\btheta_1, \tilde\btheta_1 \leftarrow \zero$, $\qvalue_0(\cdot, \cdot), \vvalue_0(\cdot) \leftarrow \ig$, $\pi_0(\cdot)\leftarrow \argmax_{a \in \cA}\qvalue_0(\cdot, a)$
	\FOR{$t=1,2,\dots$}
	\IF {$\text{det}(\hat\bSigma_{t}) \leq 2\text{det}(\hat\bSigma_{t_k})$} 
	\STATE Set $\qvalue_{t}(\cdot,\cdot) \leftarrow \qvalue_{t-1}(\cdot,\cdot)$, $\vvalue_{t}(\cdot) \leftarrow \vvalue_{t-1}(\cdot)$, $\pi_{t}(\cdot) \leftarrow \pi_{t-1}(\cdot)$ \COMMENT{Keep policy}
	\ELSE 
	\STATE $k \leftarrow k+1$, $t_k\leftarrow t$  \COMMENT{New phase starts at $t=t_k$}
	\STATE Set $\hat\cC_{t}$ as in \eqref{def:discountset} \COMMENT{$\hat\cC_t$ is based on $s_1,\dots,s_t$, to be used by the policy in this round}
\STATE Set $\qvalue_{t}(\cdot,\cdot)\leftarrow\valueite(\hat\cC_{t}, U)$, $\vvalue_{t}(\cdot) \leftarrow \max_{a \in \cA}\qvalue_{t}(\cdot,a)$, $\pi_{t}(\cdot)\leftarrow \argmax_{a \in \cA} \qvalue_{t}(\cdot, a)$
	\ENDIF
		\STATE Take action $a_t  \leftarrow \pi_t(s_t)$, receive $s_{t+1} \sim \PP(\cdot|s_t, a_t)$
		\STATE Set $[\bar\var_t\vvalue_t](s_t, a_t)$ as in \eqref{alg:dis:estt}, $\error_t$ as in \eqref{def:discounterror}
	\STATE $\bar\sigma_t\leftarrow \sqrt{\max\big\{\ig^2/d, [\bar \var_t \vvalue_t](s_t, a_t) + \error_t\big\}}$
	\STATE $\hat\bSigma_{t+1}\leftarrow \hat\bSigma_{t} + \bar\sigma_t^{-2}\bphi_{{\vvalue}_t}(s_t,a_t)\bphi_{{\vvalue}_t}(s_t,a_t)^\top$, $\hat\bbb_{t+1} \leftarrow \hat\bbb_{t} + \bar\sigma_t^{-2}\bphi_{\vvalue_t}(s_t, a_t)\vvalue_t(s_{t+1})$\label{alg:dis:sigma}
	\STATE $\tilde\bSigma_{t+1} \leftarrow \tilde\bSigma_{t} + \bphi_{{\vvalue}_t^2}(s_t,a_t)\bphi_{{\vvalue}_t^2}(s_t,a_t)$, $\tilde\bbb_{t+1} \leftarrow \tilde\bbb_{t} + \bphi_{{\vvalue}_t^2}(s_t,a_t)\vvalue_t^2(s_{t+1})$
	\STATE $\hat\btheta_{t+1} \leftarrow \hat\bSigma_{t+1}^{-1}\hat\bbb_{t+1}$, $\tilde\btheta_{t+1} \leftarrow \tilde\bSigma_{t+1}^{-1}\tilde\bbb_{t+1}$
	
	\ENDFOR
	\end{algorithmic}
\end{algorithm}

\begin{algorithm}[t]
	\caption{Extended Value Iteration: $\valueite(\cC, U)$}\label{algorithm:2}
	\begin{algorithmic}[1]
	\REQUIRE Set $\cC\subset \RR^d$, number of value iteration rounds $U$
\STATE Set $\qvalue^{(0)}(\cdot,\cdot ) \leftarrow \ig$.
\STATE $Q(\cdot,\cdot) \leftarrow Q^{(0)}(\cdot,\cdot)$
\IF {$\cC \cap \cB\neq \emptyset$}
\FOR{$u=1,\ldots, U$}
	\STATE Let $\vvalue^{(u-1)}(\cdot) \leftarrow \max_{a \in \cA}\qvalue^{(u-1)}(\cdot, a)$ and
	\begin{align}
	    \qvalue^{(u)}(\cdot, \cdot)\leftarrow  \reward(\cdot, \cdot) + \gamma \max_{\btheta \in \cC\cap \cB} \big\la \btheta, \bphi_{\vvalue^{(u-1)}}(\cdot, \cdot)\big\ra.\notag
	\end{align}
	\ENDFOR
	\STATE Let $\qvalue(\cdot, \cdot) \leftarrow \qvalue^{(U)}(\cdot, \cdot)$
\ENDIF
	\RETURN $\qvalue(\cdot,\cdot)$
	\end{algorithmic}
\end{algorithm}

The key difference between $\algnamedis$ and UCLK lies in the regression-based estimator. In particular, $\algnamedis$ adapts the weighted ridge regression proposed in Section \ref{section:finite_main} to calculate its estimate $\hat\btheta_t$, instead of the regular ridge regression used in \citet{zhou2020provably}. The weight $\bar\sigma_t$ is calculated in Line \ref{alg:dis:sigma}, which is an upper bound of the variance $[\var_t \vvalue_t](s_t, a_t)$,
as the sum of an empirical variance estimate $[\bar\var_t\vvalue_t](s_t, a_t)$ and the offset term $\error_t$. As discussed in Section \ref{section:finite_main}, we expect this weighting to give rise a more accurate estimate of $\btheta^*$. 
The computational efficiency of $\algnamedis$ also depends on the specific family of feature mapping $\bphi(\cdot|\cdot, \cdot)$. Suppose $\bphi(\cdot|\cdot, \cdot)$ has the same structure $\bphi(s'|s, a) =\bpsi(s')\odot\bmu(s,a)$ as we have discussed in Section \ref{sec:finitealg}, and we have access to a $O(\text{poly}(d))$-time integration oracle $\cO$. Moreover, suppose we have access to a maximization oracle $\cO'$ which can solve the constrained linear maximization problem $\max_{\btheta \in \cC \cap \cB}\la \btheta, \ab\ra$ within polynomial time in $d$.
Then it can be verified that the total computation cost of $\algnamedis$ is also $O(\text{poly}(d)|\cA|\ig T)$. 
The argument is easiest to explain in the case when $\cO$ needs the evaluation of $V$ to compute $\sum_{s'} \bpsi(s') V(s')$ at a fixed set $\cS'\subset \cS$ of $\text{poly(d)}$ size; -- the generic case is similar to this.
EVI will then compute  a sequence of vectors $\mb_1,\dots,\mb_U\in \RR^d$. In particular, for $u\in [U]$,
$\mb_u = \sum_{s'\in \cS'} \bpsi(s') V^{(u-1)}(s')$.
This implies that at any $(s,a)\in \cS\times \cA$,
$Q^{(u)}(s,a)$ can  be computed 
by evaluating $r(s,a)+\gamma \max_{\btheta\in \cC \cap \cB} \ip{\btheta, \mb_{u-1} \odot \bphi(s,a)}$, which can be done in polynomial time by our assumption on $\cO'$.
Then, to compute $\mb_u$, one needs to compute $\max_{a\in \cA} Q^{(u-1)}(s',a)$ at $s'\in \cS'$, 
which is again, polynomial, resulting in the claimed compute cost. We note in passing that this is an example of computation with a factored linear model (cf. Section 3, \citealt{PiSze16:FLM}).

\subsection{Upper Bound}
Now we provide the regret analysis for $\algnamedis$. First, we define confidence set $\hat\cC_t$:
\begin{align}
    &\hat\cC_t = \bigg\{\btheta: \Big\|\hat\bSigma_{t}^{1/2}(\btheta - \hat\btheta_t)\Big\|_2 \leq \hat\beta_t\bigg\},\label{def:discountset}\\
    &\hat\beta_t = 8\sqrt{d\log(1+t/\lambda) \log(4t^2/\delta)}+ 4\sqrt{d} \log(4t^2/\delta) + \sqrt{\lambda}\pnorm.\notag
\end{align}
Meanwhile, we define a set $\cB$ similar to that of \citet{zhou2020provably}. The set $\cB$ includes all possible $\btheta$ such that $\la \btheta, \bphi(\cdot|s,a)\ra$ is a probability distribution: 
\begin{align}
    \cB =\{\btheta\in \RR^d: \la \btheta, \bphi(\cdot|s,a)\ra \text{ is a probability distribution for all }
    (s,a)\in \cS \times \cA\,
    \}.\notag
\end{align}
This set is non-empty as $\theta^*\in \cB$.
We also specify the empirical variance $[\bar\var_t\vvalue_t](s_t, a_t)$ and offset term $\error_t$ as follows:
\begin{align}
	    &[\bar\var_t\vvalue_t](s_t, a_t) = \big[\big\la \bphi_{\vvalue_t^2}(s_t, a_t), \tilde\btheta_t\big\ra\big]_{[0,\ig^2]} - \big[\la \bphi_{\vvalue_t}(s_t, a_t), \hat\btheta_t\ra\big]_{[0, \ig]}^2,\label{alg:dis:estt}\\
	    &\error_t =  \min\Big\{\ig^2, 2\ig\check\beta_t\big\|\hat\bSigma_{t}^{-1/2} \bphi_{\vvalue_t}(s_t, a_t)\big\|_2\Big\} +  \min\Big\{\ig^2, \tilde\beta_t\big\|\tilde\bSigma_{t}^{-1/2} \bphi_{\vvalue_t^2}(s_t, a_t)\big\|_2\Big\}\,.\label{def:discounterror}
\end{align}
Here, $\check\beta_t$ and $\tilde\beta_t$ are defined by
\begin{align}
            &\check\beta_t = 8d\sqrt{\log(1+t/\lambda) \log(4t^2/\delta)}+ 4\sqrt{d} \log(4t^2/\delta) + \sqrt{\lambda}\pnorm,\notag \\
    & \tilde\beta_t = 8\sqrt{d\ig^4\log(1+t\ig^4/(d\lambda)) \log(4t^2/\delta)}+ 4\ig^2 \log(4t^2/\delta) + \sqrt{\lambda}\pnorm.\notag
\end{align}
The following key technical lemma guarantees that with high probability, $\btheta^*$ lies in a sequence of confidence sets. Meanwhile, the difference between the empirical variance $[\bar\var_t\vvalue_t](s_t, a_t)$ and the true variance $[\var\vvalue_t](s_t, a_t)$ can be bounded by the offset term $\error_t$. 
Let $\PP$ be the probability distribution obtained over state-action sequences by using 
$\algnamedis$ in the MDP $M$. In the next two results, the probabilistic statements refer to this distribution.
\begin{lemma}\label{lemma:theta-ball}
With probability at least $1-3\delta$, simultaneously for all $1 \leq t \leq T$, we have 
\begin{align}
    \btheta^* \in \hat\cC_t\cap \cB,\ |[\bar\var_t\vvalue_t](s_t, a_t) - [\var\vvalue_t](s_t, a_t)| \leq \error_t.\notag
\end{align}
\end{lemma}
\begin{proof}
See Appendix~\ref{sec:proof:theta-ball}. 
\end{proof}
\begin{remark}
Currently our $\hat\beta_t, \check\beta_t, \tilde\beta_t$ depend on a logarithmic term $\log(4t^2/\delta)$. This is due to a union bound over $T$ events, each of which holds with probability $\delta/(4t^2)$ for $t \in [T]$. This term can actually be tightened to $\log(4d^2\log^2 t/\delta)$ by the fact that $\hat\cC_t$ only changes at the beginning of each epoch and there are at most $d\log T$ epochs. Thus, by using a more dedicated empirical variance estimate $[\bar\var \vvalue_t](\cdot, \cdot)$ based on a rarely-updated version of $\hat\btheta_t$ and $\tilde\btheta_t$, and applying a union bound on the $d\log T$ events, we can sharpen the logarithmic term.
\end{remark}
Equipped with Lemma \ref{lemma:theta-ball},  we are able to prove the following theorem about the regret of $\algnamedis$. 
\begin{theorem}\label{thm:regret}
Let $M_{\btheta^*}$ be a discounted, $\pnorm$-bounded linear mixture MDP. Set $\lambda = 1/{\pnorm}^2$. 
Then, with probability at least $1-5\delta$, the total regret of $\algnamedis$ on $M_{\btheta^*}$ is bounded by
\begin{align}
    \text{Regret}\big(M_{\btheta^*}, T\big) &=
  \tilde O\bigg(\sqrt{d^2\ig^3 + d\ig^4}\sqrt{T} +  d^{2.5}\ig^2 + d^2\ig^3 + \ig T\gamma^U\bigg).\notag
\end{align}
\end{theorem}
\begin{proof}
See Appendix~\ref{2:main}. 
\end{proof}

\begin{remark}
Theorem \ref{thm:regret} implies that when $U = \lceil (1-\gamma)^{-1}\log (T/(1-\gamma))\rceil$, $d \geq (1-\gamma)^{-1}$ and $T \geq d^4(1-\gamma)^{-1} + d^3(1-\gamma)^{-2}$, the regret is $\tilde O(d\sqrt{T}/(1-\gamma)^{1.5})$. Now note that 
\citet{zhou2020provably} proved that for any algorithm, its regret on a discounted linear mixture MDP is lower bounded by $\Omega(d\sqrt{T}/(1-\gamma)^{1.5})$. Hence,
in the large $d$ and large sample limit, the upper bound of the regret achieved by $\algnamedis$ matches the lower bound up to logarithmic factors. Therefore, $\algnamedis$ is nearly minimax optimal for discounted linear mixture MDPs in the said regime. 
\end{remark}


\section{Conclusion and Future Work}\label{sec:conclusion}
In this paper, we considered RL with linear function approximation for linear mixture MDPs. We proposes a new Bernstein-type concentration inequality for self-normalized vector-valued martingales, which was shown to tighten existing confidence sets for linear bandits 
when the reward noise has low variance $\sigma_t^2$ and is almost surely uniformly bounded by a constant $R>0$. 
This also allowed us to derive a bandit algorithm 
for the stochastic linear bandit problem with changing actions sets.
The proposed algorithm uses weighted least-squares estimates and 
achieves a second-order regret bound of order $O(R\sqrt{dT} + d\sqrt{\sum_{t=1}^T \sigma_t^2})$,
which is a significant improvement on the dimension dependence in the low-noise regime.
Based on the new tail inequality, we propose a new, computationally efficient algorithm, $\algnamefin$ for episodic MDPs with an $\tilde O(dH\sqrt{T})$ regret, and the $\algnamedis$ algorithm for discounted MDPs with an $\tilde O(d\sqrt{T}/(1-\gamma)^{1.5})$ regret. Both regret bounds match the corresponding lower bounds up to logarithmic factors, which shows that both algorithms are nearly minimax optimal. 

We would like to point out that our current regret bounds are nearly minimax optimal 
only for the ``large dimension'' and ``large sample'' cases.
In particular, $\algnamefin$ is nearly minimax optimal 
only when $d \geq H$ and $T \geq d^4H^2 + d^3H^3$ 
and $\algnamedis$ is nearly minimax optimal only when 
$d \geq (1-\gamma)^{-1}$ and $T \geq d^4(1-\gamma)^{-1} + d^3(1-\gamma)^{-2}$.
It remains to be seen whether the range-restrictions on the dimension and the sample size can be loosened or altogether eliminated.

\appendix

\section{The Definition of State- and Action-value Functions for Nonstationary Policies in Discounted MDPs}\label{sec:app_valuedef}
The purpose of this section is to introduce a definition of the state and the action-value functions, which is consistent with the definition given in the introduction, but  avoids the issue that the definition given there may be ill-defined.
Fix a nonstationary policy $\pi = (\pi_t)_{t\ge 1}$,
 a state-action pair $(s,a)\in \cS\times \cA$, 
 a time index $t\ge 1$
 and a history $h_t:=(s_1,a_1,s_2,a_2,\dots,a_{t-1})\in (\cS\times\cA)^{t-1}$ of length $t-1$.
Let $\PP_{\pi,h_t,s,a}$ be a probability distribution over $(\cS\times\cA)^{\NN}$, where $\NN=\{1,2,\dots\}$ is the set of positive integers defined so that for a sequence 
$(s_t,a_t,s_{t+1},a_{t+1},\dots)\in (\cS\times\cA)^{\NN}$,
\begin{align*}
\MoveEqLeft 
\PP_{\pi,h_t,s,a}(s_t,a_t,s_{t+1},a_{t+1},\dots) = \\
	&\onep{s_t=s}\onep{a_t=a} \times \\
	&\PP(s_{t+1}|s_t,a_t) \onep{a_{t+1}=\pi_{t+1}(s_{t+1};s_1,a_1,\dots,s_t,a_t)} \times \\
	&\PP(s_{t+2}|s_{t+1},a_{t+1}) \onep{a_{t+2}=\pi_{t+2}(s_{t+2};s_1,a_1,\dots,s_{t+1},a_{t+1})} \times \\
	&\vdots
\end{align*}
It is not hard to see that this is indeed a probability distribution.
Denoting by $\EE_{\pi,h_t,s,a}$ the expectation corresponding to $\PP_{\pi,h_t,s,a}$, 
the action value $Q_t^\pi(s,a)$ is defined as
\begin{align*}
Q_t^\pi(s,a) = \EE_{\pi,h_t,s,a}\left[ \sum_{i=0}^\infty \gamma^i r(s_{t+i},a_{t+i}) \right]\,.
\end{align*}
Note that $Q_t^\pi$, while the notation suppresses this dependence, is a function of $h_t$.
The value function $V_t^\pi$ is defined similarly. With the above notation, in fact, we can define it as 
\begin{align*}
V_t^\pi(s) = Q_t^\pi(s, \pi_t(s_t;s_1,a_1,\dots,s_{t-1},a_{t-1}))\,.
\end{align*}
That the ``Bellman equations'' \eqref{eq:nonstatbellman} hold is a direct consequence of these definitions.
It also follows that when the values in \eqref{eq:valueconddef} are well-defined, they agree with the definitions given here (regardless of the choice of the initial distribution used in the definitions in
\eqref{eq:valueconddef}).
\section{Proof of Theorems in Section \ref{sec:linearbandit}}\label{app:concen_main}
In this section we prove Theorem \ref{lemma:concentration_variance} and Corollary \ref{coro:linearregret}.
\subsection{Proof of Theorem \ref{lemma:concentration_variance}}\label{sec:proof:concentration_variance}
We follow the proof in \citet{dani2008stochastic} with a refined analysis. 
Let us start with recalling two well known results that we will need:
\begin{lemma}[\citealt{Fre75}]
\label{lemma:freedman}
Let $M,v>0$ be fixed constants.
Let $\{x_i\}_{i=1}^n$ be a stochastic process,$\{\cG_i\}_i$ be a filtration so that 
so that for all $i\in [n]$
$x_i$ is $\cG_i$-measurable, while 
 almost surely 
$\EE[x_i|\cG_{i-1}]=0$,
 $|x_i| \leq M$ and 
\begin{align}
    \sum_{i=1}^n \EE(x_i^2|\cG_{i}) \leq v\,. \notag
\end{align}
Then, for any $\delta >0$, with probability at least $1-\delta$, 
\begin{align}
    \sum_{i=1}^n x_i \leq \sqrt{2v \log(1/\delta)} + 2/3\cdot M \log(1/\delta).\notag
\end{align}
\end{lemma}

\begin{lemma}[Lemma 11, \citealt{abbasi2011improved}]\label{lemma:sumcontext}
For any $\lambda>0$ and sequence $\{\xb_t\}_{t=1}^T \subset \RR^d$
for $t\in \{0,1,\dots,T\}$, define $\Zb_t = \lambda \Ib+ \sum_{i=1}^{t}\xb_i\xb_i^\top$.
Then, provided that $\|\xb_t\|_2 \leq L$ holds for all $t\in [T]$,
we have
\begin{align}
    \sum_{t=1}^T \min\{1, \|\xb_t\|_{\Zb_{t-1}^{-1}}^2\} \leq 2d\log\frac{d\lambda+TL^2}{d \lambda}\,.\notag
\end{align}
\end{lemma}

Recall that for
$t\ge 0$, $\Zb_t = \lambda \Ib+ \sum_{i=1}^{t}\xb_i\xb_i^\top$.
Since $\Zb_t = \Zb_{t-1}+\xb_t\xb_t^\top$, by the matrix inversion lemma
\begin{align}
    \Zb_t^{-1} = \Zb_{t-1}^{-1} - \frac{\Zb_{t-1}^{-1}\xb_t\xb_t^\top \Zb_{t-1}^{-1}}{1+w_t^2}.\label{eq:matrixinverse}
\end{align}
We need the following definitions: 
\begin{align}
    \db_0 = 0,\ 
    Z_0 = 0, \
    \db_t = \sum_{i=1}^t \xb_i \eta_i,\ Z_t = \|\db_t\|_{\Zb_t^{-1}},\ 
    w_t =  \|\xb_t\|_{\Zb^{-1}_{t-1}},\ 
    \event_t = \ind\{0 \le s \leq t, Z_s \leq \beta_s\}\,,\label{eq:concen_main_aux}
\end{align}
where $t\ge 1$ and we define $\beta_0=0$.
Recalling that $x_t$ is $\cG_t$-measurable and $\eta_t$ is $\cG_{t+1}$-measurable,
we find that $d_t$, $Z_t$ and $\event_t$ are $\cG_{t+1}$-measurable
while $w_t$ is $\cG_t$ measurable.
We now prove the following result:
\begin{lemma}\label{lemma:martingale_first}
Let $\db_i, w_i, \event_i$ be as defined in \eqref{eq:concen_main_aux}.
Then, with probability at least $1-\delta/2$, 
simultaneously for all $t\ge 1$ it holds that
\begin{align}
\sum_{i=1}^t \frac{2\eta_i\xb_i^\top \Zb_{i-1}^{-1}\db_{i-1}}{1+w_i^2} \event_{i-1} \leq 3\beta_t^2/4.\notag
\end{align}
\end{lemma}

\begin{proof}
We have
\begin{align}
    \bigg|\frac{2\xb_i^\top \Zb_{i-1}^{-1}\db_{i-1}}{1+w_i^2} \event_{i-1}\bigg| \leq \frac{2\|\xb_i\|_{\Zb_{i-1}^{-1}} [\|\db_{i-1}\|_{\Zb_{i-1}^{-1}}\event_{i-1}]}{1+w_i^2} \leq \frac{2w_i\beta_{i-1}}{1+w_i^2} \leq \min\{1, 2w_i\}\beta_{i-1}, \label{eq:martingale_first_0}
\end{align}
where the first inequality holds due to Cauchy-Schwarz inequality, the second inequality holds due to the definition of $\event_{i-1}$, the last inequality holds by algebra.
For simplicity, let $\ell_i$ denote
\begin{align}
    \ell_i =  \frac{2\eta_i\xb_i^\top \Zb_{i-1}^{-1}\db_{i-1}}{1+w_i^2} \event_{i-1}.
\end{align}
We are preparing to apply Freedman's inequality from Lemma \ref{lemma:freedman} to $(\ell_i)_i$ and $(\cG_i)_i$.
First note that 
$\EE[\ell_i|\cG_i]=0$. Meanwhile, by \eqref{eq:martingale_first_0}, the inequalities
\begin{align}
    |\ell_i| \leq  R\beta_{i-1} \min\{1, 2w_i\} \leq R\beta_{i-1} \leq R\beta_t \label{eq:martingale_first_1}
\end{align}
almost surely hold (the last inequality follows since $(\beta_i)_i$ is increasing).
We also have
\begin{align}
    \sum_{i=1}^t \EE[\ell_i^2|\cG_i] &\leq \sigma^2 \sum_{i=1}^t\bigg(\frac{2\xb_i^\top \Zb_{i-1}^{-1}\db_{i-1}}{1+w_i^2} \event_{i-1}\bigg)^2\notag \\
    & \leq \sigma^2\sum_{i=1}^t [\min\{1, 2w_i\}\beta_{i-1}]^2\notag \\
    & \leq 4\sigma^2\beta_t^2\sum_{i=1}^t\min\{1, w_i^2\}\notag \\
    & \leq 8\sigma^2\beta_t^2d\log(1+tL^2/(d\lambda)),\label{eq:martingale_first_2}
\end{align}
where the first inequality holds since $\EE[\eta_i^2|\cG_i] \leq \sigma^2$, the second inequality holds due to \eqref{eq:martingale_first_0}, the third inequality holds again since $(\beta_i)_i$ is increasing, 
the last inequality holds due to Lemma \ref{lemma:sumcontext}. Therefore, by 
\eqref{eq:martingale_first_1} and \eqref{eq:martingale_first_2}, using Lemma \ref{lemma:freedman}, we know that for any $t$, with probability at least $1-\delta/(4t^2)$, we have
\begin{align}
    \sum_{i=1}^t \ell_i &\leq \sqrt{16\sigma^2\beta_t^2d\log(1+tL^2/(d\lambda)) \log(4t^2/\delta)} + 2/3\cdot R\beta_t\log(4t^2/\delta)\notag \\
    & \leq \frac{\beta_t^2}{4}+16\sigma^2d\log(1+tL^2/(d\lambda)) \log(4t^2/\delta) + \frac{\beta_t^2}{4} + 4R^2 \log^2(4t^2/\delta)\notag \\
    & \leq \beta_t^2/2 + \frac{1}{4}\big(8\sigma\sqrt{d\log(1+tL^2/(d\lambda)) \log(4t^2/\delta)} + 4R \log(4t^2/\delta)\big)^2\notag \\
    & = 3\beta_t^2/4, \label{eq:martingale_first_3}
\end{align}
where the first inequality holds due to Lemma \ref{lemma:freedman}, the second inequality holds due to $2\sqrt{|ab|} \leq |a|+|b|$, the last equality holds due to the definition of $\beta_t$. Taking union bound for \eqref{eq:martingale_first_3} from $t = 1$ to $\infty$ and using the fact that $\sum_{t=1}^\infty t^{-2} <2$ finishes the proof.
\end{proof}
We also need the following lemma.
\begin{lemma}\label{lemma:martingale_second}
Let $w_i$ be as defined in \eqref{eq:concen_main_aux}.
Then, with probability at least $1-\delta/2$, simultaneously for all $t\ge 1$ it holds that
\begin{align}
\sum_{i=1}^t \frac{\eta_i^2w_i^2}{1+w_i^2} \leq \beta_t^2/4.\notag
\end{align}
\end{lemma}

\begin{proof}
We are preparing to apply Freedman's inequality (Lemma \ref{lemma:freedman}) to $(\ell_i)_i$ and $(\cG_i)_i$
where now
\begin{align}
    \ell_i = \frac{\eta_i^2w_i^2}{1+w_i^2} - \EE\bigg[\frac{\eta_i^2w_i^2}{1+w_i^2}\bigg|\cG_i\bigg].
\end{align}
Clearly, for any $i$, we have $\EE[\ell_i|\cG_i]=0$ almost surely (a.s.).
We further have that a.s.
\begin{align}
     \sum_{i=1}^t\EE[\ell_i^2|\cG_i] &\leq \sum_{i=1}^t \EE\bigg[\frac{\eta_i^4w_i^4}{(1+w_i^2)^2}\bigg|\cG_i\bigg]\notag \\
     &\leq R^2\sum_{i=1}^t \EE\bigg[\frac{\eta_i^2w_i^2}{1+w_i^2}\bigg|\cG_i\bigg] \notag \\
     & \leq R^2\sigma^2\sum_{i=1}^t \frac{w_i^2}{1+w_i^2} \notag \\
     &\leq 2R^2\sigma^2d\log(1+tL^2/(d\lambda)),\label{eq:8888_0}
\end{align}
where the first inequality holds due to the fact $\EE(X - \EE X)^2 \leq \EE X^2$, the second inequality holds since $|\eta_t| \leq R$ a.s., the third inequality holds since $\EE[\eta_i^2|\cG_i] \leq \sigma^2$ a.s. and $w_i$ is $\cG_i$-measurable, the fourth inequality holds due to the fact $w_i^2/(1+w_i^2) \leq \min\{1, w_i^2\}$ and Lemma \ref{lemma:sumcontext}. Furthermore, by the fact that $|\eta_i| \leq R$ a.s., we have
\begin{align}
    |\ell_i| \leq \bigg|\frac{\eta_i^2w_i^2}{1+w_i^2}\bigg| + \bigg|\EE\bigg[\frac{\eta_i^2w_i^2}{1+w_i^2}\bigg|\cG_i\bigg]\bigg| \leq 2R^2 \,\, \text{a.s.}\label{eq:8888_1}
\end{align}
Therefore, by \eqref{eq:8888_0} and \eqref{eq:8888_1}, using Lemma \ref{lemma:freedman}, we know that for any $t$, with probability at least $1-\delta/(4t^2)$, we have that a.s.,
\begin{align}
    \sum_{i=1}^t \frac{\eta_i^2w_i^2}{1+w_i^2} &\leq \sum_{i=1}^t \EE\bigg[\frac{\eta_i^2w_i^2}{1+w_i^2}\bigg|\cG_i\bigg] + \sqrt{4R^2\sigma^2d\log(1+tL^2/(d\lambda)) \log(4t^2/\delta)} + 4/3\cdot R^2\log(4t^2/\delta)\notag\\
    & \leq \sigma^2\sum_{i=1}^t \frac{w_i^2}{1+w_i^2} + 2R \sigma \sqrt{d\log(1+tL^2/(d\lambda)) \log(4t^2/\delta)} + 2R^2\log(4t^2/\delta)\notag\\
    & \leq 2\sigma^2d\log(1+tL^2/(d\lambda))+2R \sigma \sqrt{d\log(1+tL^2/(d\lambda)) \log(4t^2/\delta)} + 2R^2\log(4t^2/\delta)\notag \\
    & \leq 1/4\cdot \big(8\sigma\sqrt{d}\sqrt{\log(1+tL^2/(d\lambda)) \log(4t^2/\delta)} + 4R \log(4t^2/\delta)\big)^2\notag \\
    & = \beta_t^2/4,\label{eq:8888_2}
\end{align}
where the first inequality holds due to Lemma \ref{lemma:freedman}, the second inequality holds due to $\EE[\eta_i^2|\cG_i] \leq \sigma^2$, the third inequality holds due to the fact $w_i^2/(1+w_i^2) \leq \min\{1, w_i^2\}$ and Lemma \ref{lemma:sumcontext}, the last inequality holds due to the definition of $\beta_t$. Taking union bound for \eqref{eq:8888_2} from $t = 1$ to $\infty$ and using the fact that $\sum_{t=1}^\infty t^{-2} <2$ 
finishes the proof.
\end{proof}

With this, we are ready to prove
Theorem \ref{lemma:concentration_variance}.
\begin{proof}[Proof of Theorem \ref{lemma:concentration_variance}]
We first give a crude upper bound on $Z_t$. We have
\begin{align}
    Z_t^2  &= (\db_{t-1}+ \xb_t\eta_t)^\top \Zb_t^{-1}(\db_{t-1}+ \xb_t\eta_t) \notag \\
    & = \db_{t-1}^\top \Zb_t^{-1}\db_{t-1} + 2\eta_t\xb_t^\top \Zb_t^{-1}\db_{t-1} + \eta_t^2 \xb_t^\top \Zb_t^{-1}\xb_t\notag \\
    & \leq Z_{t-1}^2 + \underbrace{2\eta_t\xb_t^\top \Zb_t^{-1}\db_{t-1}}_{I_1} + \underbrace{\eta_t^2 \xb_t^\top \Zb_t^{-1} \xb_t}_{I_2},\notag
\end{align}
where the inequality holds since $\Zb_t \succeq \Zb_{t-1}$. For term $I_1$, 
from the matrix inversion lemma (cf. \eqref{eq:matrixinverse}),
we have
\begin{align}
    I_1 &= 2\eta_t \bigg(\xb_t^\top \Zb_{t-1}^{-1}\db_{t-1} - \frac{\xb_t^\top\Zb_{t-1}^{-1}\xb_t\xb_t^\top \Zb_{t-1}^{-1} \db_{t-1} }{1+w_t^2}\bigg)\notag \\
    & = 2\eta_t\bigg(\xb_t^\top \Zb_{t-1}^{-1}\db_{t-1} - \frac{w_t^2\xb_t^\top \Zb_{t-1}^{-1} \db_{t-1} }{1+w_t^2}\bigg)\notag \\
    & = \frac{2\eta_t\xb_t^\top \Zb_{t-1}^{-1}\db_{t-1}}{1+w_t^2}\,.\notag
\end{align}
For term $I_2$, again
from the matrix inversion lemma (cf. \eqref{eq:matrixinverse}),
we have
\begin{align}
    I_2 = \eta_t^2\bigg(\xb_t^\top\Zb_{t-1}^{-1}\xb_t^\top - \frac{\xb_t^\top\Zb_{t-1}^{-1}\xb_t\xb_t^\top \Zb_{t-1}^{-1}\xb_t}{1+w_t^2}\bigg) = \eta_t^2\bigg(w_t^2 - \frac{w_t^4}{1+w_t^2}\bigg) = \frac{\eta_t^2w_t^2}{1+w_t^2}\,.\notag
\end{align}
Therefore, we have
\begin{align}
    Z_t^2 \leq \sum_{i=1}^t \frac{2\eta_i\xb_i^\top \Zb_{i-1}^{-1}\db_{i-1}}{1+w_i^2} + \sum_{i=1}^t\frac{\eta_i^2w_i^2}{1+w_i^2}.\label{eq:concentration_variance_1}
\end{align}
Consider now the event $\event$ where the conclusions of Lemma \ref{lemma:martingale_first} and Lemma \ref{lemma:martingale_second} hold.
We claim that on this event for any $i\ge 0$, 
$Z_i \leq \beta_i$. We prove this by induction on $i$. 
Let the said event hold. The base case of $i=0$ holds since $\beta_0= 0=Z_0$, by definition.
Now fix some $t\ge 1$ and assume that for all $0 \le i < t$, we have $Z_i \leq \beta_i$.
This implies that $\event_1 = \event_2 = \dots = \event_{t-1}=1$. 
Then by \eqref{eq:concentration_variance_1}, we have
\begin{align}
    Z_t^2 \leq \sum_{i=1}^t \frac{2\eta_i\xb_i^\top \Zb_{i-1}^{-1}\db_{i-1}}{1+w_i^2} + \sum_{i=1}^t\frac{\eta_i^2w_i^2}{1+w_i^2} = \sum_{i=1}^t \frac{2\eta_i\xb_i^\top \Zb_{i-1}^{-1}\db_{i-1}}{1+w_i^2}\event_{i-1} + \sum_{i=1}^t\frac{\eta_i^2w_i^2}{1+w_i^2}.\label{eq:concentration_variance_2}
\end{align}
Since on the event $\cE$
the conclusions of Lemma \ref{lemma:martingale_first} and Lemma \ref{lemma:martingale_second} hold, we have 
\begin{align}
    \sum_{i=1}^t \frac{2\eta_i\xb_i^\top \Zb_{i-1}^{-1}\db_{i-1}}{1+w_i^2} \event_{i-1} \leq 3\beta_t^2/4,\ \sum_{i=1}^t \frac{\eta_i^2w_i^2}{1+w_i^2} \leq \beta_t^2/4.\label{eq:concentration_variance_3}
\end{align}
Therefore, substituting \eqref{eq:concentration_variance_3} into \eqref{eq:concentration_variance_2}, we have $Z_t \leq \beta_t$, which ends the induction. Taking the union bound, the events in Lemma \ref{lemma:martingale_first} and Lemma \ref{lemma:martingale_second} hold with probability at least $1-\delta$, which implies that with probability at least $1-\delta$, for any $t$, $Z_t \leq \beta_t$. 

Finally, we bound $\|\bmu_t - \bmu^*\|_{\Zb_t}$ as follows. First,
\begin{align}
    \bmu_t = \Zb_t^{-1}\bbb_t = \Zb_t^{-1}\sum_{i=1}^t\xb_i(\xb_i^\top \bmu^* + \eta_i) = \bmu^* - \lambda \Zb_t^{-1}\bmu^* + \Zb_t^{-1}\db_t\,. \notag
\end{align}
Then, on $\event$ we have
\begin{align}
    \|\bmu_t - \bmu^*\|_{\Zb_t} = \big\|\db_t - \lambda \bmu^*\big\|_{\Zb_t^{-1}} \leq Z_t + \sqrt{\lambda}\|\bmu^*\|_2 \leq \beta_t + \sqrt{\lambda}\|\bmu^*\|_2,
\end{align}
where the first inequality holds due to triangle inequality and $\Zb_t \succeq \lambda \Ib$, 
while the last one holds since we have shown that on $\event$, $Z_t \leq \beta_t$ for all $t\ge 0$, thus finishing the proof.

\end{proof}

\subsection{Proof of Theorem \ref{coro:linearregret}}\label{sec:proof:linearregret}
\begin{proof}[Proof of Theorem \ref{coro:linearregret}]
By the assumption on $\epsilon_t$, we know that
\begin{align}
    |\epsilon_t/\bar\sigma_t| \leq R/\bar\sigma^t_{\text{min}},\ \EE[\epsilon_t|\ab_{1:t}, \epsilon_{1:t-1}] = 0,\ \EE [(\epsilon_t/\bar\sigma_t)^2|\ab_{1:t}, \epsilon_{1:t-1}] \leq 1,\ \|\ab_t/\bar\sigma_t\|_2 \leq A/\bar\sigma^t_{\text{min}},\notag
\end{align}
Then, taking $\cG_t = \sigma(\ab_{1:t}, \epsilon_{1:t-1})$, 
using that $\sigma_t$ is $\cG_t$-measurable, we can apply
Theorem \ref{lemma:concentration_variance} to $(\bx_t,\eta_t)=(\ba_t/\sigma_t,\epsilon_t/\sigma_t)$ to get that
with probability at least $1-\delta$, 
\begin{align}
    \forall t \geq 1,\ \big\|\hat\bmu_t - \bmu^*\big\|_{\Ab_t} \leq \hat\beta_t + \sqrt{\lambda}\|\bmu^*\|_2 \leq \hat\beta_t + \sqrt{\lambda}\pnorm, 
    \label{eq:helper}
\end{align}
where $\hat\beta_t = 8\sqrt{d\log(1+tA^2/([\bar\sigma^t_{\text{min}}]^2 d\lambda)) \log(4t^2/\delta)} + 4R/\bar\sigma^t_{\text{min}}\cdot\log(4t^2/\delta)$.
Thus, in the remainder of the proof, we will assume that the event $\event$ when \eqref{eq:helper} is true holds and proceed to bound the regret on this event.

Note that on $\event$, $\bmu^* \in \cC_t$.
Recall that $\tilde \bmu_t$ is the optimistic parameter choice of the algorithm (cf. Line~\ref{alg:ofu} in Algorithm \ref{algorithm:reweightbandit}).
Then, using the standard argument for linear bandits, the pseudo-regret for round $t$ is bounded by
\begin{align}
     \la \ab_t^*, \bmu^*\ra - \la \ab_t, \bmu^*\ra \leq \la \ab_t, \tilde\bmu_t\ra - \la \ab_t, \bmu^*\ra = \la\ab_t, \tilde\bmu_t - \hat\bmu_{t-1} \ra + \la\ab_t, \hat\bmu_{t-1} - \bmu^* \ra,\label{eq:linearregret_0}
\end{align}
where the inequality holds due to the choice $\tilde \bmu_t$. To further bound \eqref{eq:linearregret_0}, we have
\begin{align}
    &\la\ab_t, \tilde\bmu_t - \hat\bmu_{t-1} \ra + \la\ab_t, \hat\bmu_{t-1} - \bmu^* \ra\notag \\
    &\quad \leq \|\ab_t\|_{\Ab_{t-1}^{-1}}(\|\tilde\bmu_t - \hat\bmu_{t-1}\|_{\Ab_{t-1}} + \|\bmu^* - \hat\bmu_{t-1}\|_{\Ab_{t-1}})\notag \\
    &\quad \leq 2(\hat\beta_{t-1} + \sqrt{\lambda}\pnorm)\|\ab_t\|_{\Ab_{t-1}^{-1}},\label{eq:linearregret_0.1}
\end{align}
where the first inequality holds due to Cauchy-Schwarz inequality, the second one holds since $\tilde\bmu_t, \bmu^* \in \cC_{t-1}$. Meanwhile, we have $0 \leq \la \ab_t^*, \bmu^*\ra - \la \ab_t, \bmu^*\ra \leq 2$. Thus, substituting \eqref{eq:linearregret_0.1} into \eqref{eq:linearregret_0} and summing up \eqref{eq:linearregret_0} for $t=1, \dots, T$, we have
\begin{align}
    \text{Regret}(T) = \sum_{t=1}^T\big[\la \ab_t^*, \bmu^*\ra - \la \ab_t, \bmu^*\ra\big]
    &\leq 2\sum_{t=1}^T\min\Big\{1, \bar\sigma_t(\hat\beta_{t-1} + \sqrt{\lambda}\pnorm)\|\ab_t/\bar\sigma_t\|_{\Ab_{t-1}^{-1}}\Big\}.\label{eq:linearregret_11}
\end{align}
To further bound the right-hand side above, 
we decompose the set $[T]$ into a union of two disjoint subsets $[T] = \cI_1 \cup \cI_2$, where
\begin{align}
    \cI_1 = \Big\{t \in [T]:\|\ab_t/\bar\sigma_t\|_{\Ab_{t-1}^{-1}} \geq 1 \Big\},\ \cI_2 = [T]\setminus \cI_1.
\end{align}
Then the following upper bound of $|\cI_1|$ holds:
\begin{align}
    |\cI_1| \leq \sum_{t\in\cI_1}\min\Big\{1, \|\ab_t/\bar\sigma_t\|_{\Ab_{t-1}^{-1}}^2\Big\} \leq \sum_{t=1}^T\min\Big\{1, \|\ab_t/\bar\sigma_t\|_{\Ab_{t-1}^{-1}}^2\Big\} \leq 2d \log(1+TA^2/(d\lambda[\bar\sigma^T_{\text{min}}]^2)),\label{eq:linearregret_12}
\end{align}
where the first inequality holds since $\|\ab_t/\bar\sigma_t\|_{\Ab_{t-1}^{-1}} \geq 1$ for $t \in \cI_1$, the third inequality holds due to Lemma \ref{lemma:sumcontext} together with the fact $\|\ab_t/\bar\sigma_t\|_2 \leq A/\bar\sigma^T_{\text{min}}$. Therefore, by \eqref{eq:linearregret_11}, 
\begin{align}
    \MoveEqLeft
    \text{Regret}(T)/2=\notag \\
    & \sum_{t\in \cI_1}\min\Big\{1, \bar\sigma_t(\hat\beta_{t-1} + \sqrt{\lambda}\pnorm)\|\ab_t/\bar\sigma_t\|_{\Ab_{t-1}^{-1}}\Big\} + \sum_{t\in \cI_2}\min\Big\{1, \bar\sigma_t(\hat\beta_{t-1} + \sqrt{\lambda}\pnorm)\|\ab_t/\bar\sigma_t\|_{\Ab_{t-1}^{-1}}\Big\}\notag \\
    & \leq \bigg[\sum_{t\in \cI_1} 1\bigg] + \sum_{t\in \cI_2}(\hat\beta_{t-1} + \sqrt{\lambda}\pnorm)\bar\sigma_t\|\ab_t/\bar\sigma_t\|_{\Ab_{t-1}^{-1}}\notag \\
    & = |\cI_1| + \sum_{t\in \cI_2}(\hat\beta_{t-1} + \sqrt{\lambda}\pnorm)\bar\sigma_t\min\Big\{1, \|\ab_t/\bar\sigma_t\|_{\Ab_{t-1}^{-1}}\Big\}\notag \\
    & \leq 2d \log(1+TA^2/(d\lambda[\bar\sigma^T_{\text{min}}]^2)) + \sum_{t=1}^T(\hat\beta_{t-1} + \sqrt{\lambda}\pnorm)\bar\sigma_t\min\Big\{1, \|\ab_t/\bar\sigma_t\|_{\Ab_{t-1}^{-1}}\Big\},\label{eq:linearregret_13}
\end{align}
where the first inequality holds since for any $x$ real, $\min\{1,x\}\le 1$ and also $\min\{1,x\}\le x$, 
the second inequality holds since $\|\ab_t/\bar\sigma_t\|_{\Ab_{t-1}^{-1}} \leq 1$ for $t \in \cI_2$ and the last one holds due to \eqref{eq:linearregret_12}. Finally, to further bound \eqref{eq:linearregret_13}, 
notice that
\begin{align}
    &\sum_{t=1}^T(\hat\beta_{t-1} + \sqrt{\lambda}\pnorm)\bar\sigma_t\min\Big\{1, \|\ab_t/\bar\sigma_t\|_{\Ab_{t-1}^{-1}}\Big\}\notag \\
    & \quad \leq \sqrt{\sum_{t=1}^T (\hat\beta_{t-1} + \sqrt{\lambda}\pnorm)^2\bar\sigma_t^2} \sqrt{\sum_{t=1}^T\min\Big\{1, \|\ab_t/\bar\sigma_t\|_{\Ab_{t-1}^{-1}}^2\Big\}} \notag \\
    &\quad \leq \sqrt{\sum_{t=1}^T (\hat\beta_{t-1} + \sqrt{\lambda}\pnorm)^2\bar\sigma_t^2}\sqrt{2d \log(1+TA^2/(d\lambda[\bar\sigma^T_{\text{min}}]^2))},\label{eq:linearregret_14}
\end{align}
where the first inequality holds due to Cauchy-Schwarz inequality, the second one holds due to Lemma \ref{lemma:sumcontext} and the the fact that $\|\ab_t/\sigma_t\|_2 \leq A/\bar\sigma^T_{\text{min}}$. Substituting \eqref{eq:linearregret_14} into \eqref{eq:linearregret_13} yields our result. 
\end{proof}

\subsection{Proof of Corollary \ref{coro:banditcoro}}\label{sec:banditcoro}
\begin{proof}[Proof of Corollary \ref{coro:banditcoro}]
Since $\bar\sigma_t = \max\{R/\sqrt{d}, \sigma_t\}$,  then we have $\bar\sigma^t_{\text{min}} \geq R/\sqrt{d}$. Therefore, with $\lambda = 1/\pnorm^2$, we have
\begin{align}
    &\log(1+TA^2/(d\lambda[\bar\sigma^T_{\text{min}}]^2)) \leq \log(1+T\pnorm^2A^2/R^2) = \tilde O(1),\label{eq:ppp1}
\end{align}
    and
\begin{align}
    \hat\beta_t+\sqrt{\lambda}\pnorm &= 
8\sqrt{d\log(1+tA^2/([\bar\sigma^t_{\text{min}}]^2 d\lambda)) \log(4t^2/\delta)} + 4R/\bar\sigma^t_{\text{min}}\cdot\log(4t^2/\delta)+\sqrt{\lambda}\pnorm\notag \\
&\leq 8\sqrt{d\log(1+T\pnorm^2A^2/R^2)\log(4T^2/\delta)} + 4\sqrt{d}\log(4T^2/\delta) +1\notag \\
& = \tilde O(\sqrt{d}).\label{eq:ppp2}
\end{align}
Substituting \eqref{eq:ppp1} and \eqref{eq:ppp2} into \eqref{eq:cororegret}, we have
\begin{align}
    \text{Regret}(T) = \tilde O\bigg(d\sqrt{\sum_{t=1}^T\bar\sigma_t^2}\bigg) = \tilde O\bigg(d\sqrt{\sum_{t=1}^T(R^2/d+\sigma_t^2)}\bigg) = \tilde O\bigg(R\sqrt{dT} + d\sqrt{\sum_{t=1}^T\sigma_t^2}\bigg),\notag
\end{align}
where the second equality holds since $\bar\sigma_t^2 = \max\{R^2/d, \sigma_t^2\} \leq R^2/d+\sigma_t^2$, the third equality holds since $\sqrt{|x|+|y|} \leq \sqrt{|x|} + \sqrt{|y|}$. 
\end{proof}

\subsection{Derivation of the Bound in \eqref{eq:faura_1}}\label{app:faura}

In this subsection, we derive the bound in \eqref{eq:faura_1} by the concentration inequality proved in \citet{faury2020improved}. The following proposition is a restatement of Theorem 1 in \citet{faury2020improved}. 
\begin{proposition}[Theorem 1, \citealt{faury2020improved}]\label{prop:faura}
Let $\{\cG_{t}\}_{t=1}^\infty$ be a filtration, where $\xb_t \in \RR^d$ is $\cG_t$-measurable and $\eta_t \in \RR$ is $\cG_{t+1}$-measurable. Suppose $\eta_t, \xb_t$ satisfy that
\begin{align}
    |\eta_t| \leq 1,\ \EE[\eta_t|\cG_t] = 0,\ \EE [\eta_t^2|\cG_t] \leq \sigma_t^2,\ \|\xb_t\|_2 \leq 1,\notag
\end{align}
Let $\Hb_t = \lambda\Ib + \sum_{i=1}^t \sigma_t^2\xb_i\xb_i^\top$. 
Then for any $0 <\delta<1, \lambda>0$, with probability at least $1-\delta$ we have
\begin{align}
    \forall t>0,\ \bigg\|\sum_{i=1}^t \xb_i \eta_i\bigg\|_{\Hb_t^{-1}} \leq \frac{\sqrt{\lambda}}{2} + \frac{2}{\sqrt{\lambda}}\log\bigg(\frac{\det(\Hb_t)^{1/2}\lambda^{-d/2}}{\delta}\bigg) + \frac{2d\log 2}{\sqrt{\lambda}}.
\end{align}
\end{proposition}
In the following, we first extend the above bound to the general case, where $|\eta_t| \leq R, \EE [\eta_t^2|\cG_t] \leq \sigma^2, \|\xb_t\|_2 \leq L$. In specific, we have 
\begin{align}
    |\eta_t/R| \leq 1,\ \EE[\eta_t/R|\cG_t] = 0,\ \EE [\eta_t^2/R^2|\cG_t] \leq \sigma^2/R^2,\ \|\xb_t/L\|_2 \leq 1,\notag
\end{align}
Therefore, by Proposition \ref{prop:faura}, let
\begin{align}
    \bar\Hb_t = \lambda\Ib+ \sum_{i=1}^t \sigma^2 \xb_i\xb_i^\top/(R^2L^2),\notag
\end{align}
the following holds with probability at least $1-\delta$,
\begin{align}
    \forall t>0,\ \bigg\|\bar\Hb_t^{-1/2}\sum_{i=1}^t \xb_i \eta_i/(RL)\bigg\|_2 &\leq \frac{\sqrt{\lambda}}{2} + \frac{2}{\sqrt{\lambda}}\log\bigg(\det(\bar\Hb_t)^{1/2}\lambda^{-d/2}\bigg) + \frac{2d\log 2+2\log(1/\delta)}{\sqrt{\lambda}}\notag \\
    & \leq \frac{\sqrt{\lambda}}{2} + \frac{d}{\sqrt{\lambda}}\log(1+t\sigma^2/\lambda) + \frac{2d\log 2+2\log(1/\delta)}{\sqrt{\lambda}},\label{eq:ggg1}
\end{align}
where the second inequality holds since $\det(\bar\Hb_t) \leq \|\bar\Hb_t\|_2^d \leq (\lambda+t\sigma^2)^d$. Set $\lambda \leftarrow \lambda\sigma^2/(R^2L^2)$, then \eqref{eq:ggg1} becomes
\begin{align}
    \forall t>0,\ \bigg\|\sum_{i=1}^t \xb_i\eta_i\bigg\|_{\Zb_t^{-1}} \leq \sigma\bigg(\frac{\sigma\sqrt{\lambda}}{2RL} + \frac{dRL}{\sigma\sqrt{\lambda}}\log(1+tR^2L^2/\lambda) + \frac{2d\log 2+2\log(1/\delta)}{\sigma\sqrt{\lambda}}RL\bigg), \label{eq:ggg3}
\end{align}
Now we are going to bound $\|\bmu_t - \bmu^*\|_{\Zb_t}$ by \eqref{eq:ggg3}. By the definition of $\bmu_t$, we have
\begin{align}
    \bmu_t = \Zb_t^{-1}\bbb_t = \Zb_t^{-1}\sum_{i=1}^t\xb_i(\xb_i^\top \bmu^* + \eta_i) = \bmu^* - \lambda \Zb_t^{-1}\bmu^* + \Zb_t^{-1}\sum_{i=1}^t \xb_i\eta_i, \notag
\end{align}
then $\|\bmu_t - \bmu^*\|_{\Zb_t}$ can be bounded as
\begin{align}
    \|\bmu_t - \bmu^*\|_{\Zb_t} = \bigg\|\Zb_t^{-1/2}\sum_{i=1}^t \xb_i\eta_i + \lambda \Zb_t^{-1/2}\bmu^*\bigg\|_2 \leq \bigg\|\sum_{i=1}^t \xb_i\eta_i\bigg\|_{\Zb_t^{-1}} + \sqrt{\lambda}\|\bmu^*\|_2 ,\label{eq:ggg0}
\end{align}
where the first equality holds due to triangle inequality and $\Zb_t \succeq \lambda \Ib$. Next, substituting \eqref{eq:ggg3} into \eqref{eq:ggg0} yields
\begin{align}
    \|\bmu_t - \bmu^*\|_{\Zb_t} \leq \frac{\sigma^2\sqrt{\lambda}}{2RL} + \frac{dRL}{\sqrt{\lambda}}\log(1+tR^2L^2/\lambda) + \frac{2d\log 2+2\log(1/\delta)}{\sqrt{\lambda}}RL + \sqrt{\lambda}\|\bmu^*\|_2.\notag
\end{align}
Finally, set $\lambda = \tilde\Theta(dR^2L^2/(\sigma^2+RL\|\bmu^*\|_2))$ to minimize the above upper bound, we have $\|\bmu_t - \bmu^*\|_{\Zb_t} \leq \tilde O(\sigma\sqrt{d} + \sqrt{dRL\|\bmu^*\|_2})$.




\section{Proof of Main Results in Section \ref{section:finite_main}}\label{sec:app_finite_1}

Further, we let $\PP$ be the distribution over $(\cS \times \cA)^{\NN}$ induced by the interconnection of $\algnamefin$ (treated as a nonstationary, history dependent policy) and the episodic MDP $M$. 
Further, let $\EE$ be the corresponding expectation operator. Note that the only source of randomness are the stochastic transitions in the MDP, hence, all random variables can be defined over the sample space $\Omega=(\cS\times \cA)^{\NN}$. Thus, we work with the probability space given by the triplet $(\Omega,\cF,\PP)$, where $\cF$ is the product $\sigma$-algebra generated by the discrete $\sigma$-algebras underlying $\cS$ and $\cA$, respectively.

For $1\le k \le K$, $1\le h \le H$, let $\cF_{k,h}$ be the $\sigma$-algebra generated by the random variables representing the state-action pairs up to and including those that appear stage $h$ of episode $k$. That is, $\cF_{k,h}$ is generated by
\begin{align*}
s_1^1,a_1^1, \dots, s_h^1,a_h^1, &\dots, s_H^1,a_H^1\,, \\
s_1^2,a_1^2, \dots, s_h^2,a_h^2, &\dots, s_H^2,a_H^2\,, \\
\vdots \\
s_1^k,a_1^k,\dots, s_h^k,a_h^k & \,.
\end{align*}
Note that, by construction,
\begin{align*}
 \bar\var_{k,h}\vvalue_{k,h+1}(s_h^k, a_h^k),
 E_{k,h},
 \bar \sigma_{k,h},
\hat\bSigma_{k+1,h}, 
\tilde\bSigma_{k+1,h},
\end{align*}
are $\cF_{k,h}$-measurable, $\hat\bb_{k+1,h},
\tilde\bb_{k+1,h},
\hat\btheta_{k+1,h}, 
\tilde\btheta_{k+1,h}$ are $\cF_{k,h+1}$-measurable, and $Q_{k,h},V_{k,h}, \pi_h^k,  \phi_{V_{k,h+1}}$ are $\cF_{k-1,H}$ measurable.
Note also that $Q_{k,h},V_{k,h}, \pi_h^k,  \phi_{V_{k,h+1}}$ are \emph{not} $\cF_{k-1,h}$ measurable:
The get their values only after episode $k-1$ is \emph{over}, due to their ``backwards'' construction.

\subsection{Proof of Lemma \ref{thm:concentrate:finite}}\label{sec:proof:concentrate:finite}
The main idea of the proof is to use a (crude) two-step, ``peeling'' device.
Let $\check\cC_{k,h}, \tilde\cC_{k,h}$ denote the following confidence sets:
\begin{align}
        &\check\cC_{k,h} = \bigg\{\btheta: \Big\|\hat\bSigma_{k,h}^{1/2}(\btheta - \hat\btheta_{k,h})\Big\|_2 \leq \check\beta_k\bigg\},\notag \\
    & \tilde\cC_{k,h} = \bigg\{\btheta: \Big\|\tilde\bSigma_{k,h}^{1/2}(\btheta - \tilde\btheta_{k,h})\Big\|_2 \leq \tilde\beta_k\bigg\}.\notag
\end{align}
Note that $\hat \cC_{k,h} \subset \check \cC_{k,h}$: The ``leading term'' in the definition of $\check\beta_k$ is larger than that in $\hat\beta_k$ by a factor of $\sqrt{d}$.
The idea of our proof is to show that $\btheta_h^*$ is included in $\check\cC_{k,h}\cap \tilde\cC_{k,h}$ with high probability (for this, a standard self-normalized tail inequality suffices) and then use that
when this holds, the weights used in constructing $\hat\btheta_{k,h}$ are sufficiently precise to ``balance'' the noise term, which allows to reduce $\check\beta_k$ by the extra $\sqrt{d}$ factor
without significantly increasing the probability of the bad event 
when $\btheta_h^* \not\in \hat\cC_{k,h}$.

We start with the following lemma. 
\begin{lemma}\label{lemma:variancebound:finite}
Let $\vvalue_{k, h+1}, \hat\btheta_{k,h}, \hat\bSigma_{k,h}, \tilde\btheta_{k,h}, \tilde\bSigma_{k,h}$ be defined in Algorithm \ref{algorithm:finite}, then we have
\begin{align}
\MoveEqLeft
    \big|\var_h\vvalue_{k,h+1}(s_h^k, a_h^k) - \bar\var_{k,h}\vvalue_{k,h+1}(s_h^k, a_h^k)\big|\notag \\
    &\leq   \min\Big\{H^2, \Big\|\tilde\bSigma_{k,h}^{-1/2}\bphi_{\vvalue_{k,h+1}^2}(s_h^k, a_h^k)\Big\|_2 \Big\|\tilde\bSigma_{k,h}^{1/2}\big(\tilde\btheta_{k,h} - \btheta^*_h\big)\Big\|_2\Big\}\notag \\
    &\qquad + \min\Big\{H^2,2H\Big\|\hat\bSigma_{k,h}^{-1/2}\bphi_{\vvalue_{k,h+1}}(s_h^k, a_h^k)\Big\|_2 \Big\|\hat\bSigma_{k,h}^{1/2}\big(\hat\btheta_{k,h} - \btheta^*_h\big)\Big\|_2\Big\}.\notag
\end{align}

\end{lemma}
\begin{proof}
We have
\begin{align}
    &\big|[\bar\var_{k,h}\vvalue_{k,h+1}](s_h^k, a_h^k) - [\var_h\vvalue_{k,h+1}](s_h^k, a_h^k)\big|\notag \\
    &  = \Big| \big[\big\la\bphi_{\vvalue_{k,h+1}^2}(s_h^k, a_h^k), \tilde\btheta_{k,h}\big\ra\big]_{[0, H^2]} - \big\la \bphi_{\vvalue_{k,h+1}^2}(s_h^k, a_h^k), \btheta^*_h\big\ra\notag \\
    &\qquad + \big(\big\la \bphi_{\vvalue_{k,h+1}}(s_h^k, a_h^k), \btheta^*_h\big\ra\big)^2 -\big[\big\la \bphi_{\vvalue_{k,h+1}}(s_h^k, a_h^k), \hat\btheta_{k,h}\big\ra\big]_{[0, H]}^2 \Big| \notag \\
    & \leq \underbrace{\big|\big[\big\la\bphi_{\vvalue_{k,h+1}^2}(s_h^k, a_h^k), \tilde\btheta_{k,h}\big\ra\big]_{[0, H^2]} - \big\la \bphi_{\vvalue_{k,h+1}^2}(s_h^k, a_h^k), \btheta^*_h\big\ra\big|}_{I_1}\notag\\
    &\qquad + \underbrace{\Big|\big(\big\la \bphi_{\vvalue_{k,h+1}}(s_h^k, a_h^k), \btheta^*_h\big\ra\big)^2 -\big[\big\la \bphi_{\vvalue_{k,h+1}}(s_h^k, a_h^k), \hat\btheta_{k,h}\big\ra\big]_{[0, H]}^2\Big|}_{I_2}, \notag 
\end{align}
where the inequality holds due to the triangle inequality. We bound $I_1$ first. We have $I_1 \leq H^2$ since both terms in $I_1$ belong to the interval $[0, H^2]$. Furthermore, 
\begin{align}
    I_1 &\leq \Big| \big\la\bphi_{\vvalue_{k,h+1}^2}(s_h^k, a_h^k), \tilde\btheta_{k,h}\big\ra - \big\la \bphi_{\vvalue_{k,h+1}^2}(s_h^k, a_h^k), \btheta^*_h\big\ra\Big|\notag \\
    &  = \Big| \big\la\bphi_{\vvalue_{k,h+1}^2}(s_h^k, a_h^k), \tilde\btheta_{k,h} - \btheta^*_h\big\ra\Big|\notag \\
    & \leq \Big\|\tilde\bSigma_{k,h}^{-1/2}\bphi_{\vvalue_{k,h+1}^2}(s_h^k, a_h^k)\Big\|_2 \Big\|\tilde\bSigma_{k,h}^{1/2}\big(\tilde\btheta_{k,h} - \btheta^*_h\big)\Big\|_2,\notag
\end{align}
where the first inequality holds since $\la \bphi_{\vvalue_{k,h+1}^2}(s_h^k, a_h^k), \btheta^*_h\ra\in [0,H^2]$ and the second inequality holds due to the Cauchy-Schwarz inequality. Thus, we have
\begin{align}
    I_1 \leq \min\Big\{H^2, \Big\|\tilde\bSigma_{k,h}^{-1/2}\bphi_{\vvalue_{k,h+1}^2}(s_h^k, a_h^k)\Big\|_2 \Big\|\tilde\bSigma_{k,h}^{1/2}\big(\tilde\btheta_{k,h} - \btheta^*_h\big)\Big\|_2\Big\}.\label{eq:concentrate:finite_0}
\end{align}
For the term $I_2$, since both terms in $I_2$ belong to the interval $[0, H^2]$, we have $I_2 \leq H^2$. Meanwhile, 
\begin{align}
    I_2 & = \Big|\big\la \bphi_{\vvalue_{k,h+1}}(s_h^k, a_h^k), \btheta^*_h\big\ra + \big[\big\la \bphi_{\vvalue_{k,h+1}}(s_h^k, a_h^k), \hat\btheta_{k,h}\big\ra\big]_{[0, H]} \Big|\notag \\
    &\qquad \cdot \Big|\big\la \bphi_{\vvalue_{k,h+1}}(s_h^k, a_h^k), \btheta^*_h\big\ra - \big[\big\la \bphi_{\vvalue_{k,h+1}}(s_h^k, a_h^k), \hat\btheta_{k,h}\big\ra\big]_{[0, H]} \Big|\notag \\
    &\leq 2H\Big|\big\la \bphi_{\vvalue_{k,h+1}}(s_h^k, a_h^k), \btheta^*_h\big\ra -  \big\la \bphi_{\vvalue_{k,h+1}}(s_h^k, a_h^k), \hat\btheta_{k,h}\big\ra \Big|\notag \\
    & = 2H\Big|\big\la \bphi_{\vvalue_{k,h+1}}(s_h^k, a_h^k), \btheta^*_h - \hat\btheta_{k,h}\big\ra\Big|\notag \\
    & \leq  2H\Big\|\hat\bSigma_{k,h}^{-1/2}\bphi_{\vvalue_{k,h+1}}(s_h^k, a_h^k)\Big\|_2 \Big\|\hat\bSigma_{k,h}^{1/2}\big(\hat\btheta_{k,h} - \btheta^*_h\big)\Big\|_2,
\end{align}
where the first inequality holds since both terms in this line are less than $H$ and the fact $\big\la \bphi_{\vvalue_{k,h+1}}(s_h^k, a_h^k), \btheta^*_h\big\ra \in [0,H]$, the second inequality holds due to the Cauchy-Schwarz inequality. Thus, we have
\begin{align}
    I_2 \leq \min\Big\{H^2,2H\Big\|\hat\bSigma_{k,h}^{-1/2}\bphi_{\vvalue_{k,h+1}}(s_h^k, a_h^k)\Big\|_2 \Big\|\hat\bSigma_{k,h}^{1/2}\big(\hat\btheta_{k,h} - \btheta^*_h\big)\Big\|_2\Big\}.\label{eq:concentrate:finite_1}
\end{align}
Combining \eqref{eq:concentrate:finite_0} and \eqref{eq:concentrate:finite_1} gives the desired result. 
\end{proof}

\begin{proof}[Proof of Lemma \ref{thm:concentrate:finite}]
Fix $h \in [H]$.
We first show that with probability at least $1-\delta/H$, for all $k$, $\btheta^*_h \in \check\cC_{k,h}$. To show this, we apply Theorem \ref{lemma:concentration_variance}. Let $\xb_i = \bar\sigma_{i,h}^{-1}\bphi_{\vvalue_{i, h+1}}(s_h^i, a_h^i)$ and $\eta_i =\bar\sigma_{i,h}^{-1}\vvalue_{i, h+1}(s_{h+1}^i) - \bar\sigma_{i,h}^{-1}\la\bphi_{\vvalue_{i, h+1}}(s_{i, h}, a_{i, h}), \btheta^*_h \ra$, $\cG_i = \cF_{i,h}$, $\bmu^* = \btheta^*_h$, $y_i = \la \bmu^*, \xb_i\ra + \eta_i$, $\Zb_i = \lambda \Ib + \sum_{i' = 1}^i \xb_{i'}\xb_{i'}^\top$, $\bbb_i  = \sum_{i' = 1}^i \xb_{i'}y_{i'}$ and $\bmu_i = \Zb_i^{-1}\bbb_i$. Then it can be verified that $y_i = \bar\sigma_{i,h}^{-1}\vvalue_{i,h+1}(s_{h+1}^i)$ and $\bmu_{i} = \hat\btheta_{i+1, h}$. Moreover,  almost surely,
\begin{align}
    \|\xb_i\|_2 \leq \bar\sigma_{i,h}^{-1}H \leq \sqrt{d}
    ,\ \ 
    |\eta_i| \leq \bar\sigma_{i,h}^{-1}H \leq \sqrt{d}
    ,\ \ 
    \EE[\eta_i|\cG_i] = 0
    ,\ \
    \EE[\eta_i^2|\cG_i] \leq d
    \,,\notag
\end{align}
where we used that $V_{i,h+1}$ takes values in $[0,H]$ and that $\| \bphi_{V_{i,h+1}}(s,a)\|_2 \le H$ by \eqref{def:bbbphi}.
Since we also have that $\bx_i$ is $\cG_i$ measurable and $\eta_i$ is $\cG_{i+1}$ measurable,
by Theorem \ref{lemma:concentration_variance}, we obtain that with probability at least $1-\delta/H$, for all $k \leq K$, 
\begin{align}
    \big\|\btheta^*_h - \hat\btheta_{k,h}\big\|_{\hat\bSigma_{k, h}} &\leq 8d\sqrt{\log(1+k/\lambda) \log(4k^2H/\delta)}+ 4\sqrt{d} \log(4k^2H/\delta) + \sqrt{\lambda}\pnorm = \check\beta_k,\label{eq:concentrate:finite:1}
\end{align}
implying that with probability $1-\delta/H$, for any $k \leq K$, $\btheta^*_h \in \check\cC_{k,h}$. 

An argument, which is analogous to the one just used (except that now the range of the ``noise'' matches the range of ``squared values'' and is thus bounded by $H^2$, rather than being bounded by $\sqrt{d}$) gives that with probability at least $1-\delta/H$, for any $k \leq K$ we have
\begin{align}
    \big\|\btheta^*_h - \tilde\btheta_{k,h}\big\|_{\tilde\bSigma_{k, h}} \leq 8\sqrt{dH^4\log(1+kH^4/(d\lambda)) \log(4k^2H/\delta)}+ 4H^2 \log(4k^2H/\delta) + \sqrt{\lambda}\pnorm = \tilde\beta_k,\label{eq:concentrate:finite:1.11}
\end{align}
which implies that with the said probability, $\btheta^*_h \in \tilde\cC_{k,h}$. 

We now show that $\btheta^*_h \in \hat\cC_{k,h}$ with high probability. We again apply Theorem \ref{lemma:concentration_variance}. Let $\xb_i = \bar\sigma_{i,h}^{-1}\bphi_{\vvalue_{i, h+1}}(s_h^i, a_h^i)$ and 
\begin{align}
    \eta_i = \bar\sigma_{i,h}^{-1}\ind\{\btheta^*_h \in \check\cC_{i,h} \cap \tilde\cC_{i,h} \}\big[\vvalue_{i, h+1}(s_{h+1}^i) - \la\bphi_{\vvalue_{i, h+1}}(s_h^i, a_h^i), \btheta^*_h \ra\big],\notag
\end{align}
$\cG_i = \cF_{i,h}$, $\bmu^* = \btheta^*_h$. Clearly $\EE[\eta_i|\cG_i] = 0$, $|\eta_i| \leq \bar\sigma_{i,h}^{-1}H \leq \sqrt{d}$ since $|\vvalue_{i, h+1}(\cdot)| \leq H$ and $\bar\sigma_{i,h} \geq H/\sqrt{d}$, $\|\xb_i\|_2 \leq \bar\sigma_{i,h}^{-1}H \leq \sqrt{d}$. Furthermore, 
owning to that $\ind\{\btheta^*_h \in \check\cC_{i,h} \cap \tilde\cC_{i,h} \}$ is $\cG_i$-measurable,
it holds that
\begin{align}
\MoveEqLeft
    \EE[\eta_i^2|\cG_i] = \bar\sigma_{i,h}^{-2}\ind\{\btheta^*_h \in \check\cC_{i,h} \cap \tilde\cC_{i,h} \}[\var_h\vvalue_{i, h+1}](s_h^i, a_h^i)\notag \\
    & \leq \bar\sigma_{i,h}^{-2}\ind\{\btheta^*_h \in \check\cC_{i,h} \cap \tilde\cC_{i,h} \}\bigg[[\bar\var_{i,h}\vvalue_{i,h+1}](s_h^i, a_h^i) \notag \\
    &\qquad + \min\Big\{H^2, \Big\|\tilde\bSigma_{i,h}^{-1/2}\bphi_{\vvalue_{i,h+1}^2}(s_h^i, a_h^i)\Big\|_2 \Big\|\tilde\bSigma_{i,h}^{1/2}\big(\tilde\btheta_{i,h} - \btheta^*_h\big)\Big\|_2\Big\}\notag \\
    &\qquad + \min\Big\{H^2,2H\Big\|\hat\bSigma_{i,h}^{-1/2}\bphi_{\vvalue_{i,h+1}}(s_h^i, a_h^i)\Big\|_2 \Big\|\hat\bSigma_{i,h}^{1/2}\big(\hat\btheta_{i,h} - \btheta^*_h\big)\Big\|_2\Big\}\bigg]\notag \\
    & \leq \bar\sigma_{i,h}^{-2}\bigg[[\bar\var_{i,h}\vvalue_{i,h+1}](s_h^i, a_h^i)  + \min\Big\{H^2, \tilde\beta_i\Big\|\tilde\bSigma_{i,h}^{-1/2}\bphi_{\vvalue_{i,h+1}^2}(s_h^i, a_h^i)\Big\|_2\Big\}\notag \\
    &\qquad + \min\Big\{H^2,2H\check\beta_i\Big\|\hat\bSigma_{i,h}^{-1/2}\bphi_{\vvalue_{i,h+1}}(s_h^i, a_h^i)\Big\|_2\Big\}\bigg]\notag \\
    & = 1, \notag
\end{align}
where the first inequality holds due to Lemma \ref{lemma:variancebound:finite}, the second inequality holds due to the indicator function, the last equality holds due to the definition of $\bar\sigma_{i,h}$. Now, let $y_i = \la \bmu^*, \xb_i\ra + \eta_i$, $\Zb_i = \lambda \Ib + \sum_{i' = 1}^i \xb_{i'}\xb_{i'}^\top$, $\bbb_i  = \sum_{i' = 1}^i \xb_{i'}y_{i'}$ and $\bmu_i = \Zb_i^{-1}\bbb_i$.
Then, by Theorem \ref{lemma:concentration_variance}, with probability at least $1-\delta/H$, $\forall k \leq K$, 
\begin{align}
    \|\bmu_k - \bmu^*\|_{\Zb_i} &\leq 8\sqrt{d\log(1+k/\lambda) \log(4k^2H/\delta)}+ 4\sqrt{d} \log(4k^2H/\delta) + \sqrt{\lambda}\pnorm = \hat\beta_k,\label{eq:concentrate:finite:2}
\end{align}
where the  equality uses the definition of $\hat\beta_k$.
Let $\event'$ be the event when
 $\btheta^*_h \in \cap_{k\le K}\check\cC_{k,h} \cap\tilde\cC_{k,h}$ and
 \eqref{eq:concentrate:finite:2} hold.
By the union bound, $\PP(\event')\ge 1-3\delta/H$.

We now show that $\btheta^*_h \in \hat\cC_{k,h}$ holds on $\event'$.
For this note that on $\event'$, for all $k\le K$, 
$\bmu_k = \hat\btheta_{k+1,h}$ for any $k \leq K$. 
Indeed, on this event, 
for any $i \leq K$, 
\begin{align}
    y_i &= \bar\sigma_{i,h}^{-1}\big(\la \btheta^*_h, \bphi_{\vvalue_{i,h+1}}(s_h^i, a_h^i)\ra + \ind\{\btheta^*_h \in \check\cC_{i,h} \cap \tilde\cC_{i,h} \}\big[\vvalue_{i, h+1}(s_{h+1}^i) - \la\bphi_{\vvalue_{i, h+1}}(s_h^i, a_h^i), \btheta^* \ra\big]\big) \notag \\
    & = \vvalue_{i, h+1}(s_{h+1}^i),\notag
\end{align}
which does implies the claim.
Therefore, by the definition  of $\hat\cC_{k,h}$ and since on $\event'$
\eqref{eq:concentrate:finite:2} holds, we get that on $\event'$, the relation
$\btheta^*_h \in \hat\cC_{k,h}$ also holds. 
Finally, taking union bound over $h$ and substituting \eqref{eq:concentrate:finite:1} and \eqref{eq:concentrate:finite:1.11} into Lemma \ref{lemma:variancebound:finite} 
shows that with probability at least $1-3\delta$, 
\begin{align}
\btheta^*_h \in \cap_{k,h} \hat\cC_{k,h} \cap \tilde \cC_{k,h}
\label{eq:strongerconf}
\end{align}

To finish our proof, it is thus sufficient to show that on the event when \eqref{eq:strongerconf} holds, it also holds
that
\begin{align}
    \big|[\bar\var_{k,h}\vvalue_{k,h+1}](s_h^k, a_h^k) - [\var_h\vvalue_{k,h+1}](s_h^k, a_h^k)\big| \leq \error_{k,h}.\notag
\end{align}
However, this is immediate from
Lemma~\ref{lemma:variancebound:finite} and the definition of $\error_{k,h}$.
\end{proof}

\subsection{Proof of Theorem \ref{thm:regret:finite}}\label{sec:proof:regret:finite}
In this subsection we prove Theorem \ref{thm:regret:finite}. The proof is broken down into a number of lemmas. However, first we need the Azuma-Hoeffding inequality: \begin{lemma}[Azuma-Hoeffding inequality, \citealt{azuma1967weighted}]\label{lemma:azuma} 
Let $M>0$ be a constant.
Let $\{x_i\}_{i=1}^n$ be a martingale difference sequence with respect to a filtration $\{\cG_{i}\}_i$ 
($\EE[x_i|\cG_i]=0$ a.s. and $x_i$ is $\cG_{i+1}$-measurable) such that
for all $i\in [n]$, $|x_i| \leq M$ 
holds almost surely. 
Then, for any $0<\delta<1$, with probability at least $1-\delta$, we have 
\begin{align}
    \sum_{i=1}^n x_i\leq M\sqrt{2n \log (1/\delta)}.\notag
\end{align} 
\end{lemma}
For the remainder of this subsection, let $\event$ denote the event when the conclusion of Lemma \ref{thm:concentrate:finite} holds. Then Lemma \ref{thm:concentrate:finite} suggests $\PP(\event)\geq 1-3\delta$. We introduce another two events $\event_1$ and $\event_2$:
\begin{align}
   &\event_1 = \bigg\{ \forall h' \in [H], \sum_{k=1}^K \sum_{h=h'}^H\Big[[\PP_h(\vvalue_{k,h+1} - \vvalue_{h+1}^{\pi^k})](s_h^k, a_h^k)  - [\vvalue_{k,h+1} - \vvalue_{h+1}^{\pi^k}](s_{h+1}^k)\Big] \leq 4H\sqrt{2T \log(H/\delta)}\bigg\},\notag\\
   &\event_2 = \bigg\{\sum_{k=1}^K \sum_{h=1}^H [\var_h \vvalue_{h+1}^{\pi^k}](s_h^k, a_h^k) \leq 3(HT + H^3 \log(1/\delta))\bigg\}.\notag
\end{align}
Then we have $\PP(\event_1) \geq 1-\delta$ and $\PP(\event_2) \geq 1-\delta$. The first one holds since $[\PP_h(\vvalue_{k,h+1} - \vvalue_{h+1}^{\pi^k})](s_h^k, a_h^k)  - [\vvalue_{k,h+1} - \vvalue_{h+1}^{\pi^k}](s_{h+1}^k)$ forms a martingale difference sequence 
and $|[\PP_h(\vvalue_{k,h+1} - \vvalue_{h+1}^{\pi^k})](s_h^k, a_h^k)  - [\vvalue_{k,h+1} - \vvalue_{h+1}^{\pi^k}](s_{h+1}^k)| \leq 4H$. Applying the Azuma-Hoeffding inequality 
(Lemma \ref{lemma:azuma}), we find that with probability at least $1-\delta$, simultaneously for all $h' \in [H]$, we have
\begin{align}
    \sum_{k=1}^K \sum_{h=h'}^H\Big[[\PP_h(\vvalue_{k,h+1} - \vvalue_{h+1}^{\pi^k})](s_h^k, a_h^k)  - [\vvalue_{k,h+1} - \vvalue_{h+1}^{\pi^k}](s_{h+1}^k)\Big] \leq 4H\sqrt{2T \log(H/\delta)},\label{eq:decompose_finite_azuma}
\end{align}
which implies $\PP(\event_1)\geq 1-\delta$.
That  $\PP(\event_2) \geq 1-\delta$ holds is due to the following lemma:
\begin{lemma}[Total variance lemma, Lemma C.5, \citealt{jin2018q}]\label{lemma:jin:var}
With probability at least $1-\delta$, we have
\begin{align}
    \sum_{k=1}^K \sum_{h=1}^H [\var_h \vvalue_{h+1}^{\pi^k}](s_h^k, a_h^k) \leq 3(HT + H^3 \log(1/\delta)).\notag
\end{align}
\end{lemma}

We now prove the following three lemmas based on $\event, \event_1, \event_2$.


\begin{lemma}\label{lemma:upper:finite}
Let $\qvalue_{k,h}, \vvalue_{k,h}$ be defined in Algorithm \ref{algorithm:finite}. 
Then, on the event $\event$, 
for any $s,a, k, h$ we have that
$\qvalue_h^*(s,a) \leq \qvalue_{k,h}(s,a)$, $\vvalue_h^*(s) \leq \vvalue_{k,h}(s)$.
\end{lemma}
\begin{proof}
Since $\event$ holds, we have for any $k\in[K]$ and $h \in [H]$, $\btheta^*_h \in \hat \cC_{k,h}$. 
We prove the statement by induction. The statement holds for $h= H+1$ since $\qvalue_{k, H+1}(\cdot, \cdot) = 0 = \qvalue_{H+1}^*(\cdot, \cdot)$. Assume the statement holds for $h+1$. That is, $\qvalue_{k, h+1}(\cdot, \cdot) \geq \qvalue_{h+1}^*(\cdot, \cdot)$, $\vvalue_{k, h+1}(\cdot) \geq \vvalue_{h+1}^*(\cdot)$. Given $s,a$, if $\qvalue_{k, h}(s,a) \geq H$, then $\qvalue_{k, h}(s,a) \geq H \geq \qvalue_h^*(s,a)$. Otherwise, we have
\begin{align}
    &\qvalue_{k, h}(s,a) - \qvalue_h^*(s,a) \notag \\
    &= \la \bphi_{\vvalue_{k, h+1}}(s,a),  \hat\btheta_{k,h} \ra + \hat\beta_k \Big\|\hat \bSigma_{k, h}^{-1/2} \bphi_{\vvalue_{k, h+1}}(s,a)\Big\|_2 - \la  \bphi_{\vvalue_{k, h+1}}(s,a) , \btheta^*_{h} \ra\notag \\
    &\qquad + \PP_h\vvalue_{k, h+1}(s,a) - \PP_h\vvalue^*_{h+1}(s,a)\notag \\
    &\geq \hat\beta_k \Big\|\hat \bSigma_{k, h}^{-1/2} \bphi_{\vvalue_{k, h+1}}(s,a)\Big\|_2 - \Big\|\hat \bSigma_{k, h}^{1/2} (\hat\btheta_{k,h} -\btheta^*_{h} )\Big\|_2 \Big\|\hat \bSigma_{k, h}^{-1/2} \bphi_{\vvalue_{k, h+1}}(s,a)\Big\|_2 \notag \\
    &\qquad + \PP_h\vvalue_{k, h+1}(s,a) - \PP_h\vvalue^*_{h+1}(s,a)\notag \\
    & \geq \PP_h\vvalue_{k, h+1}(s,a) - \PP_h\vvalue^*_{h+1}(s,a)\notag \\
    & \geq 0,\notag
\end{align}
where the first inequality holds due to Cauchy-Schwarz, the second inequality holds by the assumption that $\btheta^*_h \in \hat \cC_{k,h}$, the third inequality holds by the induction assumption and because $\PP_h$ is a monotone operator with respect to the partial ordering of functions. 
Therefore, for all $s,a$, we have $\qvalue_{k, h}(s,a) \geq \qvalue_{h}^*(s,a)$, which implies $\vvalue_{k, h}(s) \geq \vvalue_{h}^*(s)$, finishing the inductive step and thus the proof.
\end{proof}

\begin{lemma}\label{lemma:decompose_finite}
Let $\vvalue_{k,h}, \bar\sigma_{k,h}$ be defined in Algorithm \ref{algorithm:finite}. 
Then, on the event $\event\cap\event_1$,
we have
\begin{align}
    &\sum_{k=1}^K\Big[\vvalue_{k,1}(s_1^k) - \vvalue_{1}^{\pi^k}(s_1^k)\Big] \leq  2\hat\beta_K\sqrt{\sum_{k=1}^K \sum_{h=1}^H\bar\sigma_{k,h}^2}\sqrt{2Hd\log(1+K/\lambda)} + 4H\sqrt{2T \log(H/\delta)},\notag \\
    &\sum_{k=1}^K\sum_{h=1}^H\PP_h[\vvalue_{k,h+1} - \vvalue_{h+1}^{\pi^k}](s_h^k, a_h^k)  \leq 2 \hat\beta_K\sqrt{\sum_{k=1}^K \sum_{h=1}^H\bar\sigma_{k,h}^2}\sqrt{2dH^3\log(1+K/\lambda)} +  4H^2\sqrt{2T \log(H/\delta)}.\notag 
\end{align}
\end{lemma}
\begin{proof}
Assume that
$\event\cap\event_1$ holds.
We have
\begin{align}
    \vvalue_{k,h}(s_h^k) 
    - \vvalue_{h}^{\pi^k}(s_h^k) 
    &\leq   \la \hat \btheta_{k,h}, \bphi_{\vvalue_{k, h+1}}(s_h^k, a_h^k)\ra - [\PP_h\vvalue_{h+1}^{\pi^k}](s_h^k, a_h^k)  + \hat\beta_k \Big\| \hat \bSigma_{k, h}^{-1/2} \bphi_{\vvalue_{k, h+1}}(s_h^k, a_h^k)\Big\|_2\notag \\
    &\leq   \Big\|\hat\bSigma_{k, h}^{1/2}(\hat\btheta_{k,h} - \btheta^*_h)\Big\|_2 
    \Big\|\hat\bSigma_{k, h}^{-1/2}\bphi_{\vvalue_{k, h+1}}(s_h^k, a_h^k)\Big\|_2 \notag \\
    & \qquad +    [\PP_h\vvalue_{k,h+1}](s_h^k, a_h^k)
    - [\PP_h\vvalue_{h+1}^{\pi^k}](s_h^k, a_h^k) 
    + \hat\beta_k \Big\| \hat \bSigma_{k, h}^{-1/2} \bphi_{\vvalue_{k, h+1}}(s_h^k, a_h^k)\Big\|_2\notag \\
    &  \leq    [\PP_h\vvalue_{k,h+1}](s_h^k, a_h^k) - [\PP_h\vvalue_{h+1}^{\pi^k}](s_h^k, a_h^k) + 2\hat\beta_k \Big\| \hat \bSigma_{k, h}^{1/2} \bphi_{\vvalue_{k, h+1}}(s_h^k, a_h^k)\Big\|_2,
    \label{eq:decompose_finite_0}
\end{align}
where the first inequality holds due to the definition of $\vvalue_{k,h}$ 
and the Bellman equation for $V_h^{\pi^k}$, 
the second inequality holds due to Cauchy-Schwarz inequality and because we are in a linear MDP, 
the third inequality holds by the fact that on $\event$, $\btheta^*_h \in \hat\cC_{k,h}$. 
Meanwhile, since $\vvalue_{k,h}(s_h^k) - \vvalue_{h}^{\pi^k}(s_h^k) \leq H$, we also have
\begin{align}
    &\vvalue_{k,h}(s_h^k) - \vvalue_{h}^{\pi^k}(s_h^k)\notag \\
    & \leq \min\Big\{H, 2\hat\beta_k \Big\| \hat \bSigma_{k, h}^{1/2} \bphi_{\vvalue_{k, h+1}}(s_h^k, a_h^k)\Big\|_2 + [\PP_h\vvalue_{k,h+1}](s_h^k, a_h^k) - [\PP_h\vvalue_{h+1}^{\pi^k}](s_h^k, a_h^k)\Big\}\notag \\
    & \leq \min\Big\{H, 2\hat\beta_k \Big\| \hat \bSigma_{k, h}^{1/2} \bphi_{\vvalue_{k, h+1}}(s_h^k, a_h^k)\Big\|_2\Big\} + [\PP_h\vvalue_{k,h+1}](s_h^k, a_h^k) - [\PP_h\vvalue_{h+1}^{\pi^k}](s_h^k, a_h^k)\notag \\
    & \leq 2\hat\beta_k\bar\sigma_{k,h}\min\Big\{1,
    \Big\|\hat\bSigma_{k, h}^{-1/2}\bphi_{\vvalue_{k, h+1}}(s_h^k, a_h^k)/\bar\sigma_{k,h}\Big\|_2\Big\} + [\PP_h\vvalue_{k,h+1}](s_h^k, a_h^k) - [\PP_h\vvalue_{h+1}^{\pi^k}](s_h^k, a_h^k), \label{eq:decompose_finite_0.1}
\end{align}
where the second inequality holds 
since the optimal value function dominates the value function of any policy,
and thus on $\event$,
by Lemma~\ref{lemma:upper:finite},
$\vvalue_{k,h+1} (\cdot)\geq \vvalue_{h+1}^{\pi^k}(\cdot)$, the third inequality holds since $2\hat\beta_k\bar\sigma_{k,h} \geq \sqrt{d}\cdot H/\sqrt{d} \geq H$. 
 By \eqref{eq:decompose_finite_0.1} we have
\begin{align}
\MoveEqLeft
    \vvalue_{k,h}(s_h^k) - \vvalue_{h}^{\pi^k}(s_h^k) - [\vvalue_{k,h+1}(s_{h+1}^k) - \vvalue_{h+1}^{\pi^k}(s_{h+1}^k)] \\
    & \leq 2\hat\beta_k\bar\sigma_{k,h}\min\Big\{1,
    \Big\|\hat\bSigma_{k, h}^{-1/2}\bphi_{\vvalue_{k, h+1}}(s_h^k, a_h^k)/\bar\sigma_{k,h}\Big\|_2\Big\}\notag \\
    &\qquad + \PP_h[\vvalue_{k,h+1} - \vvalue_{h+1}^{\pi^k}](s_h^k, a_h^k) - [\vvalue_{k,h+1} - \vvalue_{h+1}^{\pi^k}](s_{h+1}^k). \label{eq:decompose_finite_0.2}
\end{align}
Summing up these inequalities for $k\in [K]$ and $h=h',\dots,H$,
\begin{align}
    &\sum_{k=1}^K \Big[\vvalue_{k,h'}(s_{k,h'}) - \vvalue_{h'}^{\pi^k}(s_{k,h'})\Big] \notag \\
    &\leq 2\sum_{k=1}^K \sum_{h=h'}^H\hat\beta_k\bar\sigma_{k,h}\min\Big\{1,
    \Big\|\hat\bSigma_{k, h}^{-1/2}\bphi_{\vvalue_{k, h+1}}(s_h^k, a_h^k)/\bar\sigma_{k,h}\Big\|_2\Big\}\notag \\
    & \qquad +  \sum_{k=1}^K \sum_{h=h'}^H\Big[[\PP_h(\vvalue_{k,h+1} - \vvalue_{h+1}^{\pi^k})](s_h^k, a_h^k)  - [\vvalue_{k,h+1} - \vvalue_{h+1}^{\pi^k}](s_{h+1}^k)\Big]\notag \\
    & \leq 2 \underbrace{\sum_{k=1}^K \sum_{h=1}^H\hat\beta_k\bar\sigma_{k,h}\min\Big\{1,
    \Big\|\hat\bSigma_{k, h}^{-1/2}\bphi_{\vvalue_{k, h+1}}(s_h^k, a_h^k)/\bar\sigma_{k,h}\Big\|_2\Big\}}_{I_1} +  4H\sqrt{2T \log(H/\delta)},
   \label{eq:decompose_finite_1}
\end{align}
where the first inequality holds
by a telescoping argument and
 since $\vvalue_{k,H+1}(\cdot) = \vvalue_{h+1}^{\pi^k}(\cdot) = 0$, the second inequality holds due to $\event_1$. To further bound $I_1$, we have
\begin{align}
    I_1 & \leq \sqrt{\sum_{k=1}^K \sum_{h=1}^H\bar\sigma_{k,h}^2}\sqrt{\sum_{k=1}^K \sum_{h=1}^H\hat\beta_k^2\min\Big\{1,
    \Big\|\hat\bSigma_{k, h}^{-1/2}\bphi_{\vvalue_{k, h+1}}(s_h^k, a_h^k)/\bar\sigma_{k,h}\Big\|_2^2\Big\}}\notag \\
    & \leq \hat\beta_K\sqrt{\sum_{k=1}^K \sum_{h=1}^H\bar\sigma_{k,h}^2}\sqrt{\sum_{k=1}^K \sum_{h=1}^H\min\Big\{1,
    \Big\|\hat\bSigma_{k, h}^{-1/2}\bphi_{\vvalue_{k, h+1}}(s_h^k, a_h^k)/\bar\sigma_{k,h}\Big\|_2^2\Big\}}\notag \\
    & \leq \hat\beta_K\sqrt{\sum_{k=1}^K \sum_{h=1}^H\bar\sigma_{k,h}^2}\sqrt{2Hd\log(1+K/\lambda)},\label{eq:decompose_finite_1.1}
\end{align}
where the first inequality holds due to Cauchy-Schwarz inequality, the second inequality holds since $\hat\beta_k \leq \hat\beta_K$, the third inequality holds due to Lemma \ref{lemma:sumcontext} with the fact that 
$\|\bphi_{\vvalue_{k, h+1}}(s_h^k, a_h^k)/\bar\sigma_{k,h}\|_2 \leq \|\bphi_{\vvalue_{k, h+1}}(s_h^k, a_h^k)\|_2\cdot\sqrt{d}/H \leq \sqrt{d}$.
Substituting \eqref{eq:decompose_finite_1.1} into \eqref{eq:decompose_finite_1}
gives
\begin{align}
\sum_{k=1}^K \Big[\vvalue_{k,h'}(s_{k,h'}) - \vvalue_{h'}^{\pi^k}(s_{k,h'})\Big] 
\le
2 \hat\beta_K\sqrt{\sum_{k=1}^K \sum_{h=1}^H\bar\sigma_{k,h}^2}\sqrt{2Hd\log(1+K/\lambda)}
+4H\sqrt{2T \log(H/\delta)}\,.
\label{eq:inres}
\end{align}
Choosing $h'=1$ here we get the first inequality that was to be proven.
To get the second inequality, note that
\begin{align}
    &\sum_{k=1}^K\sum_{h=1}^H\PP_h[\vvalue_{k,h+1} - \vvalue_{h+1}^{\pi^k}](s_h^k, a_h^k)\notag \\
    &  = \sum_{k=1}^K\sum_{h=2}^H[\vvalue_{k,h} - \vvalue_{h}^{\pi^k}](s_h^k)\notag \\
    &\qquad + \sum_{k=1}^K \sum_{h=1}^H\Big[[\PP_h(\vvalue_{k,h+1} - \vvalue_{h+1}^{\pi^k})](s_h^k, a_h^k)  - [\vvalue_{k,h+1} - \vvalue_{h+1}^{\pi^k}](s_{h+1}^k)\Big]\notag \\
    & \leq 2 \hat\beta_K\sqrt{\sum_{k=1}^K \sum_{h=1}^H\bar\sigma_{k,h}^2}\sqrt{2dH^3\log(1+KH/(d\lambda))} +  4H^2\sqrt{2T \log(H/\delta)},\notag
\end{align}
where to get the last inequality we sum up \eqref{eq:inres} for $h'=2,\dots,H$, and use the inequality that defines 
$\event_1$, which is followed by loosening the resulting bound.
\end{proof}
The next lemma is concerned with bounding
$    \sum_{k=1}^K\sum_{h=1}^H \bar\sigma_{k,h}^2$ on $\event\cap \event_2$:
\begin{lemma}\label{lemma:boundvar_finite}
Let $\vvalue_{k,h}, \bar\sigma_{k,h}$ be defined in Algorithm \ref{algorithm:finite}. 
Then,  
on the event $\event \cap \event_2$, we have
\begin{align}
    \sum_{k=1}^K\sum_{h=1}^H \bar\sigma_{k,h}^2 &\leq H^2T/d+3(HT + H^3 \log(1/\delta))+ 2H \sum_{k=1}^K\sum_{h=1}^H \PP_h [\vvalue_{k,h+1} - \vvalue^{\pi^k}_{h+1}](s_h^k, a_h^k) \notag \\
    &\quad   +  2\tilde\beta_K\sqrt{T}\sqrt{2dH\log(1+KH^4/(d\lambda))} + 7\check\beta_K H^2\sqrt{T}\sqrt{2dH\log(1+K/\lambda)}.\notag
\end{align}
\end{lemma}

\begin{proof}
Assume that $\event \cap \event_2$ holds. 
Since we are on $\event$,
by Lemma \ref{lemma:upper:finite}, 
for all $k,h$, $\vvalue_{k,h}(\cdot) \geq \vvalue_h^*(\cdot) \geq \vvalue_{h}^{\pi^k}(\cdot)$. 
Now, we calculate
\begin{align}
    \MoveEqLeft 
    \sum_{k=1}^K\sum_{h=1}^H \bar\sigma_{k,h}^2\notag \\
    & \leq  \sum_{k=1}^K\sum_{h=1}^H \Big[H^2/d + [\bar\var_{k,h}\vvalue_{k,h+1}](s_h^k, a_h^k) + \error_{k,h}\Big]\notag \\
    & =
    H^2T/d  + \underbrace{\sum_{k=1}^K\sum_{h=1}^H \Big[[\var_h \vvalue_{k,h+1}](s_h^k, a_h^k) - [\var_h \vvalue_{h+1}^{\pi^k}](s_h^k, a_h^k)\Big]}_{I_1} + \underbrace{2\sum_{k=1}^K\sum_{h=1}^H \error_{k,h}}_{I_2}\notag \\
    &\qquad + \underbrace{\sum_{k=1}^K\sum_{h=1}^H [\var_h \vvalue_{h+1}^{\pi^k}](s_h^k, a_h^k)}_{I_3} + \underbrace{\sum_{k=1}^K\sum_{h=1}^H \Big[[\bar\var_{k,h} \vvalue_{k,h+1}](s_h^k, a_h^k) - [\var_h \vvalue_{k,h+1}](s_h^k, a_h^k) - \error_{k,h}\Big]}_{I_4} , 
    \label{eq:boundvar_finite_0}
\end{align}
where the first inequality holds due to the definition of $\bar\sigma_{k,h}$. To bound $I_1$, we have
\begin{align}
    I_1& \leq \sum_{k=1}^K\sum_{h=1}^H [\PP_h \vvalue_{k,h+1}^2](s_h^k, a_h^k) - [\PP_h (\vvalue_{h+1}^{\pi^k})^2](s_h^k, a_h^k)\notag \\
    & \leq 2H \sum_{k=1}^K\sum_{h=1}^H [\PP_h (\vvalue_{k,h+1} - \vvalue_{h+1}^{\pi^k})](s_h^k, a_h^k),\notag 
\end{align}
where the first inequality holds since $\vvalue_{h+1}^{\pi^k}(\cdot) \leq \vvalue^*_{h+1}(\cdot) \leq \vvalue_{k,h+1}(\cdot)$, the second inequality holds since $\vvalue_{h+1}^{\pi^k}(\cdot),\vvalue_{k,h+1}(\cdot) \leq H $. To bound $I_2$, we have
\begin{align}
    I_2 &\leq 2\sum_{k=1}^K\sum_{h=1}^H \tilde\beta_k\min\Big\{1, \Big\|\tilde\bSigma_{k,h}^{-1/2}\bphi_{\vvalue_{k,h+1}^2}(s_h^k, a_h^k)\Big\|_2\Big\}\notag \\
    &\qquad +4H \sum_{k=1}^K\sum_{h=1}^H \check\beta_k\bar\sigma_{k,h}\min\Big\{1,\Big\|\hat\bSigma_{k,h}^{-1/2}\bphi_{\vvalue_{k,h+1}}(s_h^k, a_h^k)/\bar\sigma_{k,h}\Big\|_2\Big\}\notag \\
    & \leq 2\tilde\beta_K\sqrt{T} \sqrt{\sum_{k=1}^K\sum_{h=1}^H \min\Big\{1, \Big\|\tilde\bSigma_{k,h}^{-1/2}\bphi_{\vvalue_{k,h+1}^2}(s_h^k, a_h^k)\Big\|_2^2\Big\}}\notag \\
    &\qquad +7\check\beta_K H^2\sqrt{T} \sqrt{\sum_{k=1}^K\sum_{h=1}^H \min\Big\{1,\Big\|\hat\bSigma_{k,h}^{-1/2}\bphi_{\vvalue_{k,h+1}}(s_h^k, a_h^k)/\bar\sigma_{k,h}\Big\|_2^2\Big\}}\notag \\
    & \leq 2\tilde\beta_K\sqrt{T}\sqrt{2dH\log(1+KH^4/(d\lambda))} + 7\check\beta_K H^2\sqrt{T}\sqrt{2dH\log(1+K/\lambda)}\notag,
\end{align}
where the first inequality holds since $\tilde\beta_k \geq H^2$ and $\check\beta_k\bar\sigma_{k,h} \geq \sqrt{d}\cdot H/\sqrt{d} = H$, the second inequality holds due to Cauchy-Schwarz inequality, $\tilde\beta_k \leq \tilde\beta_K$, $\check\beta_k \leq \check\beta_K$, and the following bound on $\bar\sigma_{k, h}$ due to the definitions of $\bar\sigma_{k,h}, [\bar\var_{k,h}\vvalue_{k, h+1}](s_h^k, a_h^k)$ and $\error_{k,h}$:
\begin{align}
    \bar\sigma_{k,h}^2 &= \max\big\{H^2/d, [\bar\var_{k,h}\vvalue_{k, h+1}](s_h^k, a_h^k) + \error_{k,h}\big\} \leq \max\big\{H^2/d,H^2 + 2H^2\big\} = 3H^2\,.\notag
\end{align}
Finally,
the third inequality holds due to Lemma \ref{lemma:sumcontext} together with the facts that $\big\|\bphi_{\vvalue_{k, h+1}^2}(s_h^k, a_h^k)\big\|_2 \leq H^2$ and $\big\|\bphi_{\vvalue_{k, h+1}}(s_h^k, a_h^k)/\bar\sigma_{k,h}\big\|_2 \leq \big\|\bphi_{\vvalue_{k, h+1}}(s_h^k, a_h^k)\big\|_2\cdot\sqrt{d}/H \leq \sqrt{d}$.
To bound $I_3$, since $\event_2$ holds, we have
\begin{align}
    I_3 &\leq 3(HT + H^3 \log(1/\delta)).\notag
\end{align}
Finally, due to Lemma  \ref{thm:concentrate:finite}, we have $I_4 \leq 0$. 
Substituting $I_1, I_2, I_3, I_4$ into \eqref{eq:boundvar_finite_0} ends our proof. 
\end{proof}

With all above lemmas, we are ready to prove Theorem \ref{thm:regret:finite}. 
\begin{proof}[Proof of Theorem \ref{thm:regret:finite}]
By construction, taking a union bound, we have with probability $1-5\delta$ 
that $\event\cap \event_1\cap \event_2$ holds. 
In the remainder of the proof, assume that we are on this event. Thus, we can also use the conclusions of 
Lemmas \ref{lemma:upper:finite}, \ref{lemma:decompose_finite} and \ref{lemma:boundvar_finite}.
We bound the regret as 
\begin{align}
    \text{Regret}(M_{\btheta^*}, K)& \leq \sum_{k=1}^K\Big[\vvalue_{k,1}(s_1^k) - \vvalue_{1}^{\pi^k}(s_1^k)\Big]\notag \\
    & \leq 2\hat\beta_K\sqrt{\sum_{k=1}^K \sum_{h=1}^H\bar\sigma_{k,h}^2}\sqrt{2Hd\log(1+KH/(d\lambda))} + 4H\sqrt{2T \log(H/\delta)}\notag \\
    & =\tilde O\bigg(\sqrt{dH}\sqrt{d}\sqrt{\sum_{k=1}^K \sum_{h=1}^H\bar\sigma_{k,h}^2} + H\sqrt{T}\bigg) ,\label{eq:regret:finite_-1}
\end{align}
where the first inequality holds due to Lemma \ref{lemma:upper:finite}, the second inequality holds due to Lemma \ref{lemma:decompose_finite}, the equality holds since when $\lambda = 1/{\pnorm}^2$, 
\begin{align}
&\hat\beta_K = 8\sqrt{d\log(1+K/\lambda) \log(4K^2H/\delta)}+ 4\sqrt{d} \log(4K^2H/\delta) + \sqrt{\lambda}\pnorm = \tilde\Theta(\sqrt{d}).\notag
\end{align}
It remains to bound $\sum_{k=1}^K \sum_{h=1}^H\bar\sigma_{k,h}^2$. For this we have
\begin{align}
\MoveEqLeft
    \sum_{k=1}^K\sum_{h=1}^H \bar\sigma_{k,h}^2 
    \leq H^2T/d+3(HT + H^3 \log(1/\delta))+ 2H \sum_{k=1}^K\sum_{h=1}^H \PP_h [\vvalue_{k,h+1} - \vvalue^{\pi^k}_{h+1}](s_h^k, a_h^k) \notag \\
    &\qquad \quad    +  2\tilde\beta_K\sqrt{T}\sqrt{2dH\log(1+KH^4/(d\lambda))} + 7\check\beta_K H^2\sqrt{T}\sqrt{2dH\log(1+K/\lambda)}\notag \\
    &\leq H^2T/d + 3(HT + H^3 \log(1/\delta)) + 2H\notag \\
    &\quad \cdot \bigg(2 \hat\beta_K\sqrt{\sum_{k=1}^K \sum_{h=1}^H\bar\sigma_{k,h}^2}\sqrt{2dH^3\log(1+K/\lambda)} +  4H^2\sqrt{2T \log(H/\delta)}\bigg)\notag \\
    &\quad +  2\tilde\beta_K\sqrt{T}\sqrt{2dH\log(1+KH^4/(d\lambda))} + 7\check\beta_K H^2\sqrt{T}\sqrt{2dH\log(1+K/\lambda)}\notag \\
    & = \tilde O\bigg(\sqrt{\sum_{k=1}^K \sum_{h=1}^H\bar\sigma_{k,h}^2}\sqrt{d^2H^5} + H^2 T/d+ TH  +\sqrt{T}d^{1.5}H^{2.5} + H^3\bigg).\label{eq:regret:finite_0}
\end{align}
where the first inequality holds due to Lemma \ref{lemma:boundvar_finite}, the second inequality holds due to Lemma \ref{lemma:decompose_finite}, the last equality holds due to the fact that $\hat\beta_K = \tilde O(\sqrt{d})$, $\lambda = 1/{\pnorm}^2$,
\begin{align}
    &\check\beta_K = 8d\sqrt{\log(1+K/\lambda) \log(4k^2H/\delta)}+ 4\sqrt{d} \log(4k^2H/\delta) + \sqrt{\lambda}\pnorm = \tilde\Theta(d),\notag \\
    & \tilde\beta_K = 8\sqrt{dH^4\log(1+KH^4/(d\lambda)) \log(4k^2H/\delta)}+ 4H^2 \log(4k^2H/\delta) + \sqrt{\lambda}\pnorm = \tilde\Theta(\sqrt{d}H^2).\notag  
\end{align}
Therefore, by the fact that $x \leq a\sqrt{x}+b$ implies $x \leq c(a^2+b)$ with some $c>0$, \eqref{eq:regret:finite_0} yields that
\begin{align}
    \sum_{k=1}^K\sum_{h=1}^H \bar\sigma_{k,h}^2 &\leq \tilde O\big(d^2H^5 +H^2T/d + TH +\sqrt{T}d^{1.5}H^{2.5}\big)\notag \\
    & = \tilde O\big( d^2H^5 + d^4H^3 + TH + H^2T/d\big),\label{eq:regret:finite_1}
\end{align}
where the equality holds since $\sqrt{T}d^{1.5}H^{2.5} \leq (TH^2/d + d^4H^3)/2$. 
Substituting \eqref{eq:regret:finite_1} into \eqref{eq:regret:finite_-1}, we have
\begin{align}
    &\text{Regret}\big(M_{\bTheta^*}, K\big)  = \tilde O\Big(\sqrt{d^2H^2+dH^3}\sqrt{T}+ d^2H^3 + d^3H^2\Big),\notag
\end{align}
finishing the proof. 
\end{proof}

\subsection{Proof of Theorem \ref{thm:lowerbound:finite}}\label{sec:proof:lowerbound:finite}
We select $\delta = 1/H$ as suggested in Section \ref{sec:lowerbound}. 
For brevity, with a slight abuse of notation, we will use $M_{\bmu}$ to denote the MDP
described in Section \ref{sec:lowerbound} corresponding to the parameters
$\bmu = (\bmu_1,\dots,\bmu_H)$. We will use $\EE_{\bmu}$ denote the expectation underlying the distribution generated from the interconnection of a policy and MDP $M_{\bmu}$; since the policy is not denoted, we tacitly assume that the identity of the policy will always be clear from the context.
We will similarly use $\PP_{\bmu}$ to denote the corresponding probability measure.

We start with a lemma 
that will be the basis of our argument that shows that the regret in our MDP can be lower bounded by the regret of $H/2$ bandit instances:
\begin{lemma}\label{lemma:lowertrans_finite}
Suppose $H \geq 3$ and $3(d-1)\Delta \leq \delta$. Fix $\bmu\in (\{-\Delta,\Delta\}^{d-1})^H$.
Fix a possibly history dependent policy $\pi$ and define $\bar \ba_h^\pi = \EE_{\bmu}[ \ba_h \,|\, s_h=\state_h, s_1=\state_1 ]$: the expected action taken by the policy when it visits state $\state_h$ in stage $h$ provided that the initial state is $\state_1$.
Then, letting $\vvalue^*$ ($\vvalue^\pi$) be the optimal value function (the value function of policy $\pi$, respectively), we have
\begin{align}
    \vvalue^*_1(\state_1) - \vvalue^\pi_1(\state_1) \geq \frac{H}{10}\sum_{h=1}^{H/2}\Big(\max_{\ab \in \cA} \la \bmu_h, \ab\ra - \la \bmu_h, \bar \ba_h^\pi )\ra\Big).\notag
\end{align}
\end{lemma}
\begin{proof}
Fix $\bmu$. Since $\bmu$ is fixed, we drop the subindex from $\PP$ and $\EE$.
Since $\cA = \{+1,-1\}^{d-1}$ and $\bmu_h\in \{-\Delta,\Delta\}^{d-1}$,
we have $(d-1)\Delta = \max_{\ab \in \cA} \la \bmu_h, \ab\ra$. 
Recall the definition of the value of policy $\pi$ in state $\state_1$:
\begin{align}
    \vvalue^\pi_1(\state_1) = \EE\bigg[\sum_{h=1}^H\reward_h(s_h, a_h)\bigg|s_1 = \state_1, a_h \sim \pi_h(\cdot| s_1,a_1,\dots,s_{h-1},a_{h-1},s_h)\bigg].\label{lowerbound:finite_-1}
\end{align}
Note that by the definition of our MDPs,
only $\state_{H+2}$ satisfies that $\reward_h(\state_{H+2}, \ab) = 1$, all other rewards are zero.
Also, once entered, the process does not leave $\state_{H+2}$.
Therefore, 
\begin{align}
    \vvalue^\pi_1(\state_1) = \sum_{h=1}^{H-1} (H-h)\PP(N_h|s_1=\state_1).
\end{align}
where $N_h$ is the event of visiting state $\state_h$ in stage $h$ and then entering $\state_{H+2}$:
\begin{align}
    N_h= \{s_{h+1} = \state_{H+2}, s_h = \state_h\}\,.
\end{align}
By the law of total probability,  the Markov property and the definition of $M_{\bmu}$,
\begin{align*}
\MoveEqLeft
\PP( s_{h+1}=\state_{H+2}|s_h = \state_h, s_1=\state_1 )\\
& = \sum_{\ba \in \cA} \PP( s_{h+1}=\state_{H+2}|s_h = \state_h, a_h = \ba ) \PP( a_h=\ba|s_h=\state_h,s_1=\state_1)\\
&= \sum_{\ba \in \cA} (\delta+\ip{\bmu_h,\ba}) \PP( a_h=\ba|s_h=\state_h,s_1=\state_1)\\
& = \delta+\ip{\bmu_h,\bar\ba_h^\pi}\,,
\end{align*}
where the last equality used that by definition, 
$\bar\ba_h^\pi = \sum_{\ba\in \cA} \PP(a_h=\ba|s_h=\state_h,s_1=\state_1) \ba$.
It also follows that $\PP(s_{h+1}=\state_{h+1}|s_h=\state_h,s_1=\state_1)=1-(\delta + \ip{\bmu_h,\bar\ba_h^\pi})$.
Hence,
\begin{align}
    \PP(N_h) = (\delta + \la \bmu_h, \bar\ba_h^\pi\ra)\prod_{j=1}^{h-1}(1-\delta - \la \bmu_j, \bar\ba_j^\pi\ra)\,.
\end{align}
Defining $a_h = \la\bmu_h, \bar\ba_h^\pi \ra$, we get that
\begin{align}
        \vvalue^\pi_1(\state_1) = \sum_{h=1}^H(H-h)(a_h+\delta)\prod_{j=1}^{h-1} (1- a_j-\delta) \,.\notag
\end{align}

Working backwards, it is not hard to see that the optimal policy must take at stage the action that maximizes $\ip{\bmu_h,\ba}$. Since $\max_{a\in \cA} \ip{\bmu_h,\ba} = (d-1)\Delta$, we get
\begin{align}
    \vvalue^*_1(\state_1) = \sum_{h=1}^H(H-h) (1- (d-1)\Delta-\delta)^{h-1} ((d-1)\Delta+\delta).\notag
\end{align}
For $i\in [H]$, introduce
\begin{align}
    S_i = \sum_{h = i}^H (H-h)\prod_{j=i}^{h-1}(1-a_j-\delta) (a_h+\delta),\ T_i = \sum_{h=i}^H(H-h) (1- (d-1)\Delta-\delta)^{h-i} ((d-1)\Delta+\delta).\notag
\end{align}
Then $\vvalue^*_1(\state_1) - \vvalue^\pi_1(\state_1) = T_1 - S_1$. To lower bound $T_1 - S_1$, first note that
\begin{align}
    S_i = (H-i)(a_i+\delta) + S_{i+1}(1-a_i-\delta),\ T_i = (H-i)((d-1)\Delta+\delta) + T_{i+1}(1-(d-1)\Delta-\delta),\notag
\end{align}
which gives that
\begin{align}
    T_i - S_i = (H-i - T_{i+1})((d-1)\Delta - a_i) + (1-a_i-\delta)(T_{i+1} - S_{i+1}).\label{eq:lowertrans_finite_1}
\end{align}
Therefore by induction,  we get that
\begin{align}
    T_1 - S_1 = \sum_{h=1}^{H-1} ((d-1)\Delta - a_h) (H-h-T_{h+1}) \prod_{j=1}^{h-1}(1-a_j-\delta).\label{eq:lowertrans_finite_2}
\end{align}
To further bound \eqref{eq:lowertrans_finite_2}, first we note that $T_h$ can be written as the following closed-form expression:
\begin{align}
    &T_h = \frac{(1-(d-1)\Delta-\delta)^{H-h} -1 }{(d-1)\Delta+\delta} + H-h+1 - (1-(d-1)\Delta-\delta)^{H-h},\notag
\end{align}
Hence, for any $h \leq H/2$, 
\begin{align}
\MoveEqLeft
        H-h-T_{h+1} = \frac{1-(1-(d-1)\Delta-\delta)^{H-h} }{(d-1)\Delta+\delta} + (1-(d-1)\Delta-\delta)^{H-h}\notag\\ &\geq\frac{1-(1-(d-1)\Delta-\delta)^{H/2} }{(d-1)\Delta+\delta}\geq  H/3,\label{eq:lowertrans_finite_3}
\end{align}
where the last inequality holds since $3(d-1)\Delta \leq \delta = 1/H$ and $H \geq 3$. Furthermore we have
\begin{align}
    \prod_{j=1}^{h-1}(1-a_j-\delta) \geq (1-4\delta/3)^H \geq 1/3,\label{eq:lowertrans_finite_4}
\end{align}
where the first inequality holds since $a_j \leq (d-1)\Delta, 3(d-1)\Delta \leq \delta$, the second one holds since $\delta = 1/H$ and $H \geq 3$. Therefore, substituting \eqref{eq:lowertrans_finite_3} and \eqref{eq:lowertrans_finite_4} into \eqref{eq:lowertrans_finite_2}, we have \begin{align}
    \vvalue^*_1(\state_1) - \vvalue^\pi_1(\state_1) = T_1 - S_1 \geq \frac{H}{10}\cdot \sum_{h=1}^{H/2}((d-1)\Delta - a_h),\notag
\end{align}
which finishes the proof.
\end{proof}

We also need a lower bound on the regret on linear bandits with the hypercube action set $\cA = \{-1,1\}^{d-1}$,
Bernoulli bandits with linear mean payoff. While the proof technique used is standard (cf.  \citealt{lattimore2018bandit}), we give the full proof as the ``scaling'' of the reward parameters is nonstandard:
\begin{lemma}\label{lemma:banditlowerbound}
Fix a positive real $0<\delta\le 1/3$,  and positive integers $K,d$ and assume that $K \geq d^2/(2\delta)$. 
Let $\Delta = \sqrt{\delta/K}/(4\sqrt{2})$ and
consider the linear bandit problems $\cL_{\bmu}$ parameterized with a parameter vector $\bmu\in \{-\Delta,\Delta\}^d$
and action set $\cA = \{-1, 1\}^d$ so that the reward distribution 
for taking action $\ba\in \cA$ is a Bernoulli distribution $B(\delta + \la \bmu^*,\ab\ra)$.
Then for any bandit algorithm $\cB$, there exists a $\bmu^* \in \{-\Delta, \Delta\}^d$ such that the expected pseudo-regret of $\cB$ over first $K$ steps on bandit $\cL_{\bmu^*}$ is lower bounded as follows: 
\begin{align}
    \EE_{\bmu^*}\text{Regret}(K) \geq \frac{d\sqrt{K\delta}}{8\sqrt{2}}.\notag
\end{align}
\end{lemma}
Note that the expectation is with respect to a distribution that depends both on $\cB$ and $\bmu^*$, but since $\cB$ is fixed, this dependence is hidden.
\begin{proof}
Let $\ab_k \in \cA = \{-1, 1\}^d$ denote the action chosen in round $k$. 
Then for any $\bmu \in \{-\Delta, \Delta\}^d$, the expected pseudo regret $\EE_{\bmu}\text{Regret}(K)$ corresponding to $\bmu$ satisfies
\begin{align}
    \EE_{\bmu}\text{Regret}(K) 
    = \sum_{k=1}^K\EE_{\bmu}(\max_{\ab \in \cA}\la \bmu, \ab\ra - \la \bmu, \ab_k\ra) 
    &
    = \Delta\sum_{k=1}^K \sum_{j=1}^d\EE_{\bmu}\ind\{\text{sgn}([\bmu]_j) \neq \text{sgn}([\ab_k]_j)\}  \nonumber\\
    & 
    = \Delta \sum_{j=1}^d \underbrace{\sum_{k=1}^K\EE_{\bmu}\ind\{\text{sgn}([\bmu]_j) \neq \text{sgn}([\ab_k]_j)\}}_{N_j(\bmu)}\,,
    \label{banditlowerbound_0}
\end{align}
where for a vector $\bx$, we use $[\bx]_j$ to denote its $j$th entry.
Let $\bmu^j \in\{-\Delta, \Delta\}^d $ denote the vector which differs from $\bmu$ at its $j$th coordinate only. 
Then, we have
\begin{align}
    2\sum_{\bmu}\EE_{\bmu}\text{Regret}(K) &= \Delta \sum_{\bmu}\sum_{j=1}^d(\EE_{\bmu}N_j(\bmu) + \EE_{\bmu^j}N_j(\bmu^j))\notag \\
    & = \Delta \sum_{\bmu}\sum_{j=1}^d(K+\EE_{\bmu}N_j(\bmu) - \EE_{\bmu^j}N_j(\bmu))\notag \\
    & \geq \Delta \sum_{\bmu}\sum_{j=1}^d(K- \sqrt{1 /2} K \sqrt{\text{KL}(\cP_{\bmu}, \cP_{\bmu^j})}),\label{banditlowerbound_1}
\end{align}
where the inequality holds due to $N_j(\bmu)\in [0,K]$ and Pinsker's inequality (Exercise 14.4 and Eq. 14.12, \citealt{lattimore2018bandit}),
$\cP_{\bmu}$ denotes the joint distribution over the all possible reward sequences $(r_1,\dots,r_K)\in \{0,1\}^K$ of length $K$, induced by the interconnection of the algorithm and the bandit parameterized by $\bmu$. 
By the chain rule of relative entropy,
$\text{KL}(\cP_{\bmu}, \cP_{\bmu^j})$ can be further decomposed as (cf. Exercise 14.11 of 
\citealt{lattimore2018bandit}),
\begin{align}
    \text{KL}(\cP_{\bmu}, \cP_{\bmu^j}) 
    &= 
    		\sum_{k=1}^K \EE_{\bmu}[\text{KL}(\cP_{\bmu}(r_{k}|\rb_{1:k-1}), \cP_{\bmu^j}(r_{k}|\rb_{1:k-1}) )]\notag \\
    & = \sum_{k=1}^K  \EE_{\bmu}[
    \text{KL}(B(\delta + \la \ab_k, \bmu\ra), (B(\delta + \la \ab_k, \bmu^j\ra))]\notag \\
    & \leq \sum_{k=1}^K \EE_{\bmu}\left[\frac{2\la\bmu - \bmu^j, \ab_k \ra^2}{\la \bmu, \ab_k\ra + \delta}\right]\notag \\
    & \leq \frac{16K\Delta^2}{\delta},\label{banditlowerbound_2}
\end{align}
where the second equality holds since the round $k$ reward's distribution is 
the  Bernoulli distribution $B(\delta + \la \ab_k, \bmu\ra)$ in the environment parameterized by $\bmu$,
the first inequality holds since for any two Bernoulli distribution $B(a)$ and $B(b)$, we have $\text{KL}(B(a), B(b)) \leq 2(a-b)^2/a$ when $a \leq 1/2, a+b \leq 1$, the second inequality holds since $\bmu$ only differs from $\bmu^j$ at $j$-th coordinate, $\la \bmu, \ab_k\ra \geq -d\Delta \geq -\delta/2$. It can be verified that these requirements hold when $\delta \leq 1/3$, $d\Delta \leq \delta/2$. Therefore, substituting \eqref{banditlowerbound_2} into \eqref{banditlowerbound_1}, we have
\begin{align}
    2\sum_{\bmu}\EE_{\bmu}\text{Regret}(K) \geq \sum_{\bmu} \Delta d (K - \sqrt{2} K^{3/2}\Delta/\sqrt{\delta} ) = \sum_{\bmu} \frac{d\sqrt{K\delta}}{4\sqrt{2}},\notag
\end{align}
where the equality holds since $\Delta = \sqrt{\delta/K}/(4\sqrt{2})$. Selecting $\bmu^*$ which maximizes $\EE_{\bmu}\text{Regret}(K)$ finishes the proof. 
\end{proof}

With this, we are ready to prove Theorem \ref{thm:lowerbound:finite}.
\begin{proof}[Proof of Theorem \ref{thm:lowerbound:finite}]
We can verify that the selection of $K,d, H, \delta$ satisfy the requirement of Lemma \ref{lemma:lowertrans_finite} and Lemma \ref{lemma:banditlowerbound}. 
Let $\pi^k$ denote the possibly nonstationary policy 
that is executed in episode $k$ given the history up to the beginning of the episode.
Then, by Lemma \ref{lemma:lowertrans_finite}, we have
\begin{align}
\MoveEqLeft
    \EE_{\bmu}\text{Regret}\Big( M_{\bmu}, K\Big) 
    = \EE_{\bmu}\bigg[\sum_{k=1}^K[\vvalue^*_1(\state_1) - \vvalue^{\pi^k}_1(\state_1)] \bigg]\notag \\
    &\quad \geq \frac{H}{10}\sum_{h=1}^{H/2}\underbrace{\EE_{\bmu}\bigg[\sum_{k=1}^K\Big(\max_{\ab \in \cA} \la \bmu_h, \ab\ra - \la \bmu_h, \bar\ba_h^{\pi_k}\ra\Big)\bigg]}_{I_h(\bmu,\pi)}.\label{eq:lowerbound:finite_1}
\end{align}
Let $\bmu^{-h} = (\bmu_1,\dots,\bmu_{h-1},\bmu_{h+1},\dots,\bmu_H)$.
Now, every MDP policy $\pi$ 
gives rise to a bandit algorithm $\cB_{\pi,h,\bmu^{-h}}$ for the linear bandit $\cL_{\bmu_h}$ of 
Lemma \ref{lemma:banditlowerbound}.
This bandit algorithm is such that the distribution of action it plays in round $k$ matches
the distribution of action played by $\pi$ in stage $h$ of episode $k$ conditioned on the event that
$s_h^k=\state_h$, i.e., $\PP_{\mu,\pi}( a_h^k = \cdot | s_h^k = \state_h)$ with the tacit assumption that the first state in every episode is $\state_1$.

As the notation suggests, the bandit algorithm depends on $\bmu^{-h}$. In particular, 
to play in round $k$, the bandit algorithm feeds $\pi$ with data from the MDP kernels
up until the beginning of episode $k$: 
For $i\ne h$, this can be done by just following $\PP_i$ since the parameters of these kernels is known to $\cB_{\pi,h,\bmu^{-h}}$. When $i=h$,
since $\PP_h$ is not available to the bandit algorithm, every time it is on stage $h$, if the state is $\state_h$, it feeds the action obtained from $\pi$ to $\cL_\mu$ and 
if the reward is $1$, it feeds $\pi$ with the next state $\state_{H+2}$, otherwise it feeds it with next state $\state_{h+1}$. When $i=h$ and the state is not $\state_h$, it can only be $\state_{H+2}$, in which case the next state fed to $\pi$ is $\state_{H+2}$ regardless of the action it takes.
At the beginning of episode $k$, to ensure that state $\state_h$ is ``reached'', $\pi$ is fed with the states $\state_1$, 
$\state_2$, $\dots$, $\state_h$. Then, $\pi$ is queried for its action, which is the action that the bandit plays in round $k$.
Clearly, by this construction,
the distribution of action played in round $k$ by $\cB_{\pi,h,\bmu^{-h}}$ matches the target.

Denoting by $\text{BanditRegret}(\cB_{\pi,h,\bmu^{-h}},\bmu_h)$ the regret of this bandit algorithm on $\cL_{\bmu}$, by our construction, 
\begin{align*}
I_h(\bmu,\pi) = \text{BanditRegret}(\cB_{\pi,h,\bmu^{-h}},\bmu_h)
\end{align*}
for all $h\in [H/2]$.
Hence,
\begin{align*}
\sup_{\bmu} \EE_{\bmu}\text{Regret}\Big( M_{\bmu}, K\Big) 
& \ge
\sup_{\bmu}
\frac{H}{10}\sum_{h=1}^{H/2} 
\text{BanditRegret}(\cB_{\pi,h,\bmu^{-h}},\bmu_h) \\
& \ge
\sup_{\bmu}
\frac{H}{10}\sum_{h=1}^{H/2} 
\inf_{\tilde \bmu^{-h}}
\text{BanditRegret}(\cB_{\pi,h,\tilde \bmu^{-h}},\bmu_h) \\
& = 
\frac{H}{10}\sum_{h=1}^{H/2} 
\sup_{\bmu^h}
\inf_{\tilde \bmu^{-h}}
\text{BanditRegret}(\cB_{\pi,h,\tilde \bmu^{-h}},\bmu_h) \\
& \ge \frac{H^2}{20} \frac{(d-1)\sqrt{K\delta}}{8\sqrt{2}}\,,
\end{align*}
where the last inequality follows by Lemma \ref{lemma:banditlowerbound}.
The result follows by plugging in $\delta = 1/H$ and $T = KH$.
\end{proof}








\section{Proof of Main Results in Section \ref{section 5}}\label{app:main}
Here we provide the proof for the results in Section~\ref{section 5}.
For this, let $\pi$ denote the policy implemented by $\algnamedis$. 
Note that this is a slight abuse of notation since $\algnamedis$ already defines $\pi_t$.
However, the two definitions are consistent: $\pi_t$ as defined in $\algnamedis$ depends on the history,
which is only made explicit by $\pi$.

Further, by a slight abuse of notation, we let $\PP$ be the distribution over $(\cS \times \cA)^{\NN}$ induced by the interconnection of $\pi$ and the MDP $M$. 
Further, let $\EE$ be the corresponding expectation operator. Note that the only source of randomness are the stochastic transitions in the MDP, hence, all random variables can be defined over the sample space $\Omega=(\cS\times \cA)^{\NN}$. Thus, we work with the probability space given by the triplet $(\Omega,\cF,\PP)$, where $\cF$ is the product $\sigma$-algebra generated by the discrete $\sigma$-algebras underlying $\cS$ and $\cA$, respectively.

Let $\cF_t = \sigma(s_1,\dots,s_t)$ denote the $\sigma$-algebra generated by $s_1,\dots, s_t$. Then $\qvalue_t, \vvalue_t, \pi_t, \hat\bSigma_t, \tilde\bSigma_t$ are $\cF_{t-1}$-measurable and $\hat\btheta_t, \tilde\btheta_t, , \bar \sigma_t, E_t, \hat\cC_t$ are $\cF_t$-measurable. 
\subsection{Proof of Lemma \ref{lemma:theta-ball}}\label{sec:proof:theta-ball}
Let $\check\cC_t, \tilde\cC_t$ denote the following confidence sets:
\begin{align}
    &\check\cC_t = \bigg\{\btheta: \Big\|\hat\bSigma_{t}^{1/2}(\btheta - \hat\btheta_t)\Big\|_2 \leq \check\beta_t\bigg\},\notag \\
    & \tilde\cC_t = \bigg\{\btheta: \Big\|\tilde\bSigma_{t}^{1/2}(\btheta - \tilde\btheta_t)\Big\|_2 \leq \tilde\beta_t\bigg\}.\notag 
\end{align}
Then similar to the proof of Lemma \ref{thm:concentrate:finite}, here we prove a stronger statement: with probability at least $1-3\delta$, simultaneously for all $1 \leq t \leq T$, we have
\begin{align}
    \btheta^* \in \hat\cC_t \cap\check\cC_t \cap \tilde\cC_t \cap \cB,\ |[\bar\var_t\vvalue_t](s_t, a_t) - [\var\vvalue_t](s_t, a_t)| \leq \error_t.\notag
\end{align}
We need the following lemma.
\begin{lemma}\label{lemma:variancebound}
Let $\vvalue_t, \hat\btheta_t, \hat\bSigma_t, \tilde\btheta_t, \tilde\bSigma_t$ be as defined in Algorithm \ref{algorithm}, then we have
\begin{align}
\MoveEqLeft
    \big|[\var\vvalue_t](s_t ,a_t ) - [\bar\var_t\vvalue_t](s_t ,a_t)\big|\notag \\
    &\leq  \min\bigg\{\ig^2, \big\|\tilde\bSigma_{t}^{1/2}(\btheta^* - \tilde\btheta_t)\big\|_2\big\|\tilde\bSigma_{t}^{-1/2} \bphi_{\vvalue_t^2}(s_t, a_t)\big\|_2\bigg\}\notag \\
    &\qquad + \min\bigg\{\ig^2, 2\ig\big\|\hat\bSigma_{t}^{1/2}(\btheta^* - \hat\btheta_t)\big\|_2\big\|\hat\bSigma_{t}^{-1/2} \bphi_{\vvalue_t}(s_t, a_t)\big\|_2\bigg\}.\notag
\end{align}
\end{lemma}
\begin{proof}
The proof is the same as that of Lemma \ref{lemma:variancebound:finite} with $H$ replaced by $\ig$, $\vvalue_{k, h+1}$ replaced by $\vvalue_t$, $\hat\btheta_{k,h}, \tilde\btheta_{k,h}, \hat\bSigma_{k,h}, \tilde\bSigma_{k,h}$ replaced by $\hat\btheta_{t}, \tilde\btheta_{t}, \hat\bSigma_{t}, \tilde\bSigma_{t}$, respectively,
and $\btheta^*_h$ replaced by $\btheta^*$. 
\end{proof}
\begin{proof}[Proof of Lemma \ref{lemma:theta-ball}]
The proof is the same as that of Lemma \ref{thm:concentrate:finite} with Lemma \ref{lemma:variancebound:finite} replaced by Lemma \ref{lemma:variancebound}, 
$H$ replaced by $\ig$, $\vvalue_{k, h+1}$ replaced by $\vvalue_t$, $\hat\btheta_{k,h}, \tilde\btheta_{k,h}, \hat\bSigma_{k,h}, \tilde\bSigma_{k,h}$ replaced by $\hat\btheta_{t}, \tilde\btheta_{t}, \hat\bSigma_{t}, \tilde\bSigma_{t}$, $\btheta^*_h$ replaced by $\btheta^*$. Note that the definitions of $\hat\beta_t, \check\beta_t, \tilde\beta_t$ slightly differ from those of $\hat\beta_{k,h}, \check\beta_{k,h}, \tilde\beta_{k,h}$ since we do not need to take a union bound over $\ig$ in the discounted setting.
Finally, using the fact that $\btheta^* \in \cB$ yields our result.  
\end{proof}

\subsection{Proof of Theorem \ref{thm:regret}}\label{2:main}
In this section we prove Theorem \ref{thm:regret}. 

Let $K(T)-1$ be the number of epochs (counter $k$) when $\algnamedis$ 
finishes after $T$ rounds. For convenience, we also set $t_{K(T)} = T+1$. 
Let $\event$ denote the event when the conclusion of Lemma \ref{lemma:theta-ball} holds. Then by Lemma \ref{lemma:theta-ball} we have $\PP(\event) \geq 1-3\delta$. 
Define the events $\event_1$ and $\event_2$ as follows:
\begin{align}
    &\event_1 = \bigg\{\sum_{t=1}^T \Big\{\big[\PP({\vvalue}_t-\vvalue^{\pi}_{t+1})\big](s_t,a_t)-\big(\vvalue_t(s_{t+1})-\vvalue^{\pi}_{t+1}(s_{t+1})\big)\Big\} \leq 4\ig \sqrt{2T\log(1/\delta)}\bigg\},\notag \\
    &\event_2 = \bigg\{\gamma^2\sum_{t=1}^T[\var \vvalue^{\pi}_{t+1}](s_{t}, a_{t}) \leq 5\ig T + 25/3\cdot \ig^3\sqrt{\log(\ig/\delta)}\bigg\}.\notag
\end{align}
Then we have $\PP(\event_1) \geq 1-\delta$ and $\PP(\event_2) \geq 1-\delta$. The first one holds since $M_t:=[\PP({\vvalue}_t-\vvalue^{\pi}_{t+1})](s_t,a_t)-(\vvalue_t(s_{t+1})-\vvalue^{\pi}_{t+1}(s_{t+1}))$ forms a martingale difference sequence,
 and we have $|M_t| \leq 4\ig$. 
Indeed, since $\pi_{t+1}$ is $\cF_t$-measurable, so is $V_{t+1}^\pi(\cdot)=Q_{t+1}^\pi(\cdot,\pi_{t+1}(\cdot))$. 
As a result, and thanks to also $a_t$ and $V_t$ being $\cF_t$-measurable, $\EE[M_t|\cF_t]=0$.
By the above measurability observations and because $s_{t+1}$ is by definition $\cF_{t+1}$-measurable,
it follows that $M_t$ is $\cF_{t+1}$-measurable, hence $\cF_t$-adapted. 
 Then Lemma \ref{lemma:azuma} implies that with probability at least $1-\delta$, 
\begin{align}
    \sum_{t=1}^T \Big\{\big[\PP({\vvalue}_t-\vvalue^{\pi}_{t+1})\big](s_t,a_t)-\big(\vvalue_t(s_{t+1})-\vvalue^{\pi}_{t+1}(s_{t+1})\big)\Big\} \leq 4\ig \sqrt{2T\log(1/\delta)}, \notag
\end{align}
which gives $\PP(\event_1) \geq 1-\delta$.
That $\PP(\event_2) \geq 1-\delta$ holds follows from the following lemma:
\begin{lemma}[Total variance bound, Lemma A.6,  \citealt{he2020minimax}]\label{lemma:variance}
With probability at least $1-\delta$, we have 
\begin{align}
    \gamma^2\sum_{t=1}^T[\var \vvalue^{\pi}_{t+1}](s_{t}, a_{t}) \leq 5\ig T + 25/3\cdot \ig^3\sqrt{\log(\ig/\delta)}.\notag
\end{align}
\end{lemma}

Based on $\event$, $\event_1$ and $\event_2$, we start with a number of
technical lemmas.  
\begin{lemma}[Lemma 6.2, \citealt{zhou2020provably}]\label{lemma:UCB}
On the event $\event$, 
for any $(s,a) \in \cS \times \cA$ and $1 \leq t \leq T$,
 $\ig \geq \qvalue_t(s,a) \geq \qvalue^*(s,a)$,
and  $\ig \geq \vvalue_t(s) \geq \vvalue^*(s)$ hold.
\end{lemma}

\begin{lemma}[Lemma 6.3,  \citealt{zhou2020provably}]\label{lemma:V-diff}
On the event $\event$, 
for any $0 \leq k \leq K(T)-1$ and $t_k \leq t \leq t_{k+1}-1$, there exists $\btheta_t \in \hat\cC_{t_k} \cap \cB$ such that $\qvalue_t(s_t, a_t) \leq \reward(s_t, a_t) + \gamma \big\la \btheta_t, \bphi_{\vvalue_t}(s_t, a_t)\big\ra + 2\gamma^{U}$.
\end{lemma}

\begin{lemma}\label{lemma:boundk}
Let $K(T)$ be as defined above. Then, $K(T) \leq 2d\log (1+Td/\lambda)$. 
\end{lemma}
\begin{proof}
For simplicity, we denote $K = K(T)$. 
Note that $\det(\bSigma_0) = \lambda^d$. We further have
\begin{align}
    \|\bSigma_T\|_2 &= \bigg\|\lambda\Ib+\sum_{t=1}^T\bphi_{\vvalue_t}(s_t, a_t)\bphi_{\vvalue_t}(s_t, a_t)^\top/\bar\sigma_t^2\bigg\|_2  \leq \lambda + \sum_{t=1}^T\big\|\bphi_{\vvalue_t}(s_t, a_t)\big\|_2^2/\bar\sigma_t^2 \leq \lambda + d T, \label{eq:boundk_0}
\end{align}
where the first inequality holds due to the triangle inequality, the second inequality holds due to the fact $\vvalue_t(\cdot) \leq \ig$, $\bar\sigma_t^2 \geq \ig^2/d$ and $\big\|\bphi_{\vvalue_t}(s_t, a_t)\big\|_2 \leq \ig$. Inequality \eqref{eq:boundk_0} implies that $\det(\bSigma_T) \leq (\lambda+d T)^d$. Therefore, we have
\begin{align}
    (\lambda + d T)^d \geq \det(\bSigma_T) \geq \det(\bSigma_{t_{K-1}})\geq 2^{K-1}\det(\bSigma_{t_{0}}) = 2^{K-1}\lambda^d,  \label{eq:boundk_1}
\end{align}
where the second inequality holds since $\bSigma_T\succeq\bSigma_{t_{K-1}} $, the third inequality holds due to the fact that $\det(\bSigma_{t_{k}}) \geq 2\det(\bSigma_{t_{k-1}})$ by the update rule in Algorithm \ref{algorithm}. Inequality \eqref{eq:boundk_1} implies that
\begin{align}
    K \leq d\log (1+d T/\lambda) +1 \leq 2d\log (1+d T/\lambda),\notag
\end{align}
which ends our proof. 
\end{proof}

\begin{lemma}\label{lemma:expectation}
Let $\vvalue_t$ be as defined in Algorithm \ref{algorithm}.
Then, on the event $\event_1$, we have
\begin{align}
    \sum_{t=1}^T\PP\big[\vvalue_t -\vvalue^{\pi}_{t+1} \big](s_t, a_t) \leq 4\ig\sqrt{2T \log(1/\delta)} + \sum_{t=1}^T\big(\vvalue_t(s_t) - \vvalue_t^{\pi}(s_t)\big) + 4\ig K(T).\notag
\end{align}
\end{lemma}
\begin{proof}
Assume that $\event_1$ holds.
We have
\begin{align}
    &\sum_{t=1}^T\PP\big[\vvalue_t -\vvalue^{\pi}_{t+1} \big](s_t, a_t) \notag \\
    &= \underbrace{\sum_{t=1}^T \Big\{\big[\PP({\vvalue}_t-\vvalue^{\pi}_{t+1})\big](s_t,a_t)-\big(\vvalue_t(s_{t+1})-\vvalue^{\pi}_{t+1}(s_{t+1})\big)\Big\}}_{I_1} +\underbrace{ \sum_{t=1}^T\big(\vvalue_t(s_{t+1})-\vvalue^{\pi}_{t+1}(s_{t+1})\big)}_{I_2}.\label{eq:expectation_0}
\end{align}
For the term $I_1$, by the definition of $\event_1$, we have
\begin{align}
    I_1\leq  4\ig\sqrt{2T \log(1/\delta)}.\label{eq:I_2}
\end{align}
For the term $I_2$, 
first note that the value function $V_t$ does not change over the intervals $[t_1,t_2-1]$, $[t_2,t_3-1]$, $\dots$.
Hence, 
\begin{align}
    I_2&= \sum_{k=0}^{K(T)-1}\sum_{t=t_k}^{t_{k+1}-1}\vvalue_{t}(s_{t+1})-\sum_{t=1}^T\vvalue^{\pi}_{t+1}(s_{t+1})\notag\\
    &= \sum_{k=0}^{K(T)-1}\bigg[\sum_{t=t_k}^{t_{k+1}-2}\vvalue_{t}(s_{t+1}) +\vvalue_{t_{k+1}-1}(s_{t_{k+1}}) \bigg]-\sum_{t=1}^T\vvalue^{\pi}_{t+1}(s_{t+1})\notag\\
     &= \sum_{k=0}^{K(T)-1}\bigg[\sum_{t=t_k}^{t_{k+1}-2}\vvalue_{t+1}(s_{t+1}) +\vvalue_{t_{k+1}-1}(s_{t_{k+1}}) \bigg]-\sum_{t=1}^T\vvalue^{\pi}_{t+1}(s_{t+1})\notag\\
     & \leq \sum_{k=0}^{K(T)-1}\bigg[\sum_{t=t_k}^{t_{k+1}-1}\vvalue_{t}(s_{t}) +\vvalue_{t_{k+1}-1}(s_{t_{k+1}}) \bigg]-\sum_{t=1}^T\vvalue^{\pi}_{t+1}(s_{t+1})\notag\\
     & = \sum_{t=1}^T\big(\vvalue_t(s_t) - \vvalue_t^\pi(s_t)\big) + \vvalue^\pi_1(s_1) -\vvalue^\pi_{T+1}(s_{T+1}) + \sum_{k=0}^{K(T)-1}\vvalue_{t_{k+1}}(s_{t_{k+1}+1})\notag,
\end{align}
where the third equality holds since $\vvalue_t(\cdot) = \vvalue_{t+1}(\cdot)$ for any $t_k \leq t \leq t_{k+1}-2$.
Using the fact that $0 \leq \vvalue_t(\cdot), \vvalue_t^\pi(\cdot) \leq \ig$, $I_2$ can be further bounded as
\begin{align}
    I_2 & \leq \sum_{t=1}^T\big(\vvalue_t(s_t) - \vvalue_t^\pi(s_t)\big) + \ig + \ig K(T) \leq \sum_{t=1}^T\big(\vvalue_t(s_t) - \vvalue_t^\pi(s_t)\big) + 2\ig K(T).\label{eq:I_3:1}
\end{align}
Substituting \eqref{eq:I_2} and \eqref{eq:I_3:1} into \eqref{eq:expectation_0} gets our result. 
\end{proof}

\begin{lemma}\label{lemma:sumregret}
Let $\vvalue_t, \bar\sigma_t, \tilde\beta_T, \check\beta_T$ be as defined in Algorithm \ref{algorithm}. 
Then,  
on the event $\event \cap \event_2$, 
we have
\begin{align}
    \gamma^2\sum_{t=1}^T \bar\sigma_t^2 &\leq \gamma^2\ig^2 T/d + 2\gamma^2\tilde\beta_T\sqrt{2d T \log(1+T\ig^4/(d\lambda))} + 7\gamma^2\ig^{1.5}\check\beta_T\sqrt{2d T\log(1+T/\lambda)} \notag \\
    &\quad + 5\ig T + 25/3\cdot \ig^3\sqrt{\log(\ig/\delta)} + 2\gamma^2\ig\sum_{t=1}^T\PP\big[\vvalue_t -\vvalue^{\pi}_{t+1} \big](s_t, a_t).\notag
\end{align}
\end{lemma}
\begin{proof}
Assume that $\event\cap \event'$ holds.
Then, by Lemma \ref{lemma:UCB},
$\ig \geq \qvalue_t(\cdot, \cdot) \geq \qvalue^*(\cdot, \cdot)$, $\ig \geq \vvalue_t(\cdot) \geq \vvalue^*(\cdot)$ holds 
for any $1 \leq t \leq T$ and
\begin{align}
    \gamma^2\sum_{t=1}^T \bar\sigma_t^2 & =  
    \gamma^2\sum_{t=1}^T \max\bigg\{\ig^2/d, [\bar\var_t\vvalue_t](s_t, a_t) + \error_t\bigg\}\notag \\
    & = 
    \gamma^2\sum_{t=1}^T \max\bigg\{\ig^2/d, [\var\vvalue_t](s_t, a_t) + 2\error_t + [\bar\var_t\vvalue_t](s_t, a_t) -[\var\vvalue_t](s_t, a_t) - \error_t \bigg\}\notag \\
    & \leq \gamma^2\ig^2 T/d + \gamma^2\sum_{t=1}^T\big[[\var\vvalue_t](s_t, a_t) + 2\error_t\big] + \gamma^2 \sum_{t=1}^T \big[[\bar\var_t\vvalue_t](s_t, a_t) -[\var\vvalue_t](s_t, a_t) - \error_t\big]\notag \\
    & \leq \gamma^2\ig^2 T/d + \underbrace{2\gamma^2 \sum_{t=1}^T \error_t}_{I_1} +\underbrace{\gamma^2\sum_{t=1}^T[\var\vvalue^{\pi}_{t+1}](s_t, a_t)}_{I_2}+ \underbrace{\gamma^2\sum_{t=1}^T\big[[\var\vvalue_t](s_t, a_t) - [\var\vvalue^{\pi}_{t+1}](s_t, a_t)\big]}_{I_3},\label{eq:sumregret_-1}
\end{align}
where the second inequality is by the definition of $\event$.
To bound $I_1$, recall that by the definition of $\error_t$, we have
\begin{align}
    I_1
    & \leq 2\gamma^2\sum_{t=1}^T\tilde\beta_t\min\big\{1, \big\|\tilde\bSigma_{t}^{-1/2} \bphi_{\vvalue_t^2}(s_t, a_t)\big\|_2\big\} + 2\gamma^2\ig \sum_{t=1}^T\check\beta_t\bar\sigma_t\min\big\{1, \big\|\hat\bSigma_{t}^{-1/2} \bphi_{\vvalue_t}(s_t, a_t)/\bar\sigma_t\big\|_2\big\}\notag \\
    & \leq 2\gamma^2\tilde\beta_T\sum_{t=1}^T\min\big\{1, \big\|\tilde\bSigma_{t}^{-1/2} \bphi_{\vvalue_t^2}(s_t, a_t)\big\|_2\big\} + 7\gamma^2\ig^2\check\beta_T \sum_{t=1}^T\min\big\{1, \big\|\hat\bSigma_{t}^{-1/2} \bphi_{\vvalue_t}(s_t, a_t)/\bar\sigma_t\big\|_2\big\},\label{eq:sumregret_-2}
\end{align}
where the first inequality holds since $\tilde\beta_t \geq \ig^2$, $\check\beta_t\bar\sigma_t \geq \ig$, the second inequality holds since $\tilde\beta_T \geq \tilde\beta_t$ and $\bar\sigma_t \leq \sqrt{3}\ig$, since by the definition of $\bar\sigma_t$, 
\begin{align}
    \bar\sigma_t^2 = \max\big\{\ig^2/d,[\bar \var_t \vvalue_t](s_t, a_t) + \error_t\big\} \leq\max\big\{\ig^2/d, \ig^2 + 2\ig^2\big\} = 3\ig^2.\notag
\end{align}
By the Cauchy-Schwarz inequality, we have
\begin{align}
    \sum_{t=1}^T\min\big\{1, \big\|\tilde\bSigma_{t}^{-1/2} \bphi_{\vvalue_t^2}(s_t, a_t)\big\|_2\big\} &\leq \sqrt{T}\sqrt{\sum_{t=1}^T\min\big\{1, \big\|\tilde\bSigma_{t}^{-1/2} \bphi_{\vvalue_t^2}(s_t, a_t)\big\|_2^2\big\}}\notag \\
    &\leq \sqrt{2d T \log(1+T\ig^4/(d\lambda))},\label{eq:sumregret_-3}
\end{align}
where the second inequality holds due to Lemma \ref{lemma:sumcontext}. Similarily, we have
\begin{align}
    \sum_{t=1}^T\min\big\{1, \big\|\hat\bSigma_{t}^{-1/2} \bphi_{\vvalue_t}(s_t, a_t)/\bar\sigma_t\big\|_2\big\} &\leq \sqrt{T}\sqrt{\sum_{t=1}^T\min\big\{1, \big\|\hat\bSigma_{t}^{-1/2} \bphi_{\vvalue_t}(s_t, a_t)/\bar\sigma_t\big\|_2^2\big\}}\notag \\
    & \leq \sqrt{2d T\log(1+T/\lambda)}.\label{eq:sumregret_-4}
\end{align}
Substituting \eqref{eq:sumregret_-3} and \eqref{eq:sumregret_-4} into \eqref{eq:sumregret_-2}, we get
\begin{align}
    I_1 \leq 2\gamma^2\tilde\beta_T\sqrt{2d T \log(1+T\ig^4/(d\lambda))} + 7\gamma^2\ig^2 \check\beta_T\sqrt{2d T\log(1+T/\lambda)}.\label{eq:sumregret_0}
\end{align}
To bound $I_2$, since $\event_2$ holds, then we have
\begin{align}
    I_2 \leq 5\ig T + 25/3\cdot \ig^3\sqrt{\log(\ig/\delta)}.\label{eq:sumregret_1}
\end{align}
To bound $I_3$, we have
\begin{align}
    &\gamma^2\sum_{t=1}^T\bigg[[\var \vvalue_t](s_t, a_t) - [\var\vvalue^{\pi}_{t+1}](s_t, a_t)\bigg]\notag \\
    &= \gamma^2\sum_{t=1}^T\bigg[[\PP\vvalue_t^2](s_t,a_t) - [\PP (\vvalue^{\pi}_{t+1})^2](s_t, a_t) - [ [\PP \vvalue_t](s_t,a_t)]^2 + [[\PP \vvalue^{\pi}_{t+1}](s_t, a_t)]^2\bigg]\notag \\
    & \leq  \gamma^2\sum_{t=1}^T\PP\big[\vvalue_t^2 -[\vvalue^{\pi}_{t+1}]^2 \big](s_t, a_t)\notag \\
    & \leq 2\gamma^2\ig\sum_{t=1}^T\PP\big[\vvalue_t -\vvalue^{\pi}_{t+1} \big](s_t, a_t),\label{eq:sumregret_2}
\end{align}
where the first inequality holds since 
on $\event$, by Lemma \ref{lemma:UCB},
$0\le \vvalue^{\pi}_{t+1}(\cdot) \leq\vvalue^*(\cdot) \leq  \vvalue_t(\cdot)$, the second inequality holds since $0 \leq \vvalue^{\pi}_{t+1}(\cdot), \vvalue_t(\cdot) \leq \ig$. Substituting \eqref{eq:sumregret_0}, \eqref{eq:sumregret_1} and \eqref{eq:sumregret_2} into \eqref{eq:sumregret_-1} ends our proof. 
\end{proof}

\begin{lemma}[Lemma 12,  \citealt{abbasi2011improved}]\label{lemma:det}
Suppose $\Ab, \Bb\in \RR^{d \times d}$ are two positive definite matrices satisfying that $\Ab \succeq \Bb$, then for any $\xb \in \RR^d$, $\|\xb\|_{\Ab} \leq \|\xb\|_{\Bb}\cdot \sqrt{\det(\Ab)/\det(\Bb)}$.
\end{lemma}

With all above lemmas, we begin to prove Theorem \ref{thm:regret}.
\begin{proof} [Proof of Theorem \ref{thm:regret}]
Taking a union bound, we have that with probability at least $1-5\delta$, $\event\cap\event_1\cap\event_2$ holds.  
Assume that $\event\cap\event_1\cap\event_2$ holds.
We first bound $\text{Regret}'(T)$, which is defined as follows:
\begin{align}
 \text{Regret}'(T)&=\sum_{t=1}^T \Big[\qvalue_{t}(s_t,a_t)-\vvalue_t^{\pi}(s_t)\Big].\label{eq:1}
\end{align}
Since $\event$ holds, by Lemma \ref{lemma:V-diff} it holds that for any $t \leq T$, 
for the unique index $k(t):=k\ge 0$ such that
  $t_k +1 \leq t\leq t_{k+1}$
there exists $\btheta_t \in \hat\cC_{t_k} \cap \cB$ such that
\begin{align}
    \qvalue_t(s_t,a_t)&\leq \reward(s_t,a_t) + \gamma  \big\la \btheta_{t}, \bphi_{{\vvalue}_t}(s_t,a_t)\big\ra + 2\gamma^U, \label{Qt-update}
\end{align}
By the definition of $V_t^\pi$ and $V_{t+1}^\pi$ and the fact that $a_t = \pi_t(s_t)$, 
we have
\begin{align}
    \vvalue_t^{\pi}(s_t)&=\reward(s_t,a_t)+\gamma[\PP \vvalue^{\pi}_{t+1}](s_t,a_t)\notag\\
    &=\reward(s_t,a_t)+\gamma \sum_{s' \in \cS} \big\la \btheta^*, \bphi(s'|s_t,a_t)\big \ra\vvalue^{\pi}_{t+1}(s') \notag\\
    &=\reward(s_t,a_t) + \gamma  \big\la \btheta^*, \bphi_{\vvalue^{\pi}_{t+1}}(s_t,a_t)\big\ra,\label{Qpi-update}
\end{align}
where the second equality holds due to Assumption \ref{assumption-linear}. Substituting \eqref{Qt-update} and \eqref{Qpi-update} into \eqref{eq:1}, we have
\begin{align}
    &\text{Regret}'(T)-2T\gamma^{U} \notag \\
    &\leq\gamma\sum_{t=1}^T \big(\big\la \btheta_{t}, \bphi_{{\vvalue}_t}(s_t,a_t)\big\ra-\big\la \btheta^*, \bphi_{\vvalue^{\pi}_{t+1}}(s_t,a_t)\big\ra\big)\notag\\
    &=\underbrace{\gamma \sum_{t=1}^T\big\la \btheta_{t}-\btheta^*, \bphi_{{\vvalue}_t}(s_t,a_t)\big\ra}_{I_1}  +\gamma\sum_{t=1}^T \big\la \btheta^*,\bphi_{{\vvalue}_t}(s_t,a_t)- \bphi_{\vvalue^{\pi}_{t+1}}(s_t,a_t)\big\ra\notag\\
    &=I_1+\gamma \sum_{t=1}^T \big[\PP({\vvalue}_t-\vvalue^{\pi}_{t+1})\big](s_t,a_t)\notag \\
    & \leq I_1 +4\gamma \ig\sqrt{2T \log(1/\delta)} + \gamma \underbrace{\sum_{t=1}^T\big(\vvalue_t(s_t) - \vvalue_t^{\pi}(s_t)\big)}_{\text{Regret}'(T)} +2 \gamma \ig K(T)
    ,\label{eq:regret_-1}
\end{align}
where the last inequality holds due to Lemma \ref{lemma:expectation}. Solving \eqref{eq:regret_-1} for $\text{Regret}'(T)$ gives 
\begin{align}
    \text{Regret}'(T) \leq \ig I_1 + 4\gamma \ig^2\sqrt{2T \log(1/\delta)} + 2\gamma \ig^2K(T) + 2\ig T\gamma^U.\label{eq:regret}
\end{align}
Next we bound $I_1$. Take any $1\le t \le T$. To bound the $t$th term of $I_1$ note that
by definition $\btheta_t\in \cB$. Hence, $\la\bphi(s'|s_t,a_t), \btheta_{t}\ra$ is a probability distribution and thus
\begin{align}
    \la \btheta_{t}-\btheta^*, \bphi_{{\vvalue}_t}(s_t,a_t)\big\ra = \la \btheta_{t}, \bphi_{{\vvalue}_t}(s_t,a_t)\big\ra - [\PP{\vvalue}_t](s_t, a_t) \leq \ig.\label{help}
\end{align}
Meanwhile, $I_1$ can be bounded in another way:
\begin{align}
        I_1&=\gamma \sum_{t=1}^T\big\la \btheta_{t}-\btheta^*, \bphi_{{\vvalue}_t}(s_t,a_t)\big\ra\notag\\
        & = \gamma \sum_{k=0}^{K(T)-1}\bigg[
        \sum_{t=t_k}^{t_{k+1}-1}\big\la \btheta_{t}-\btheta^*, \bphi_{{\vvalue}_t}(s_t,a_t)\big\ra
        \bigg]\notag\\
        &\leq  \gamma\sum_{k=0}^{K(T)-1}\bigg[\sum_{t=t_k}^{t_{k+1}-1}\big(\big\|\btheta_{t}-\hat\btheta_{t_k}\big\|_{\hat\bSigma_{t}}+\big\|\hat\btheta_{t_k}-\btheta^*\big\|_{\hat\bSigma_{t}}\big)\|\bphi_{{\vvalue}_t}(s_t,a_t)\|_{\hat\bSigma_{t}^{-1}}\bigg]\notag\\
        & \leq 
        \gamma\sum_{k=0}^{K(T)-1}\bigg[\sum_{t=t_k}^{t_{k+1}-1}
        2\big(\big\|\btheta_{t}-\hat\btheta_{t_k}\big\|_{\hat\bSigma_{t_k}}+\big\|\hat\btheta_{t_k}-\btheta^*\big\|_{\hat\bSigma_{t_k}}\big)\|\bphi_{{\vvalue}_t}(s_t,a_t)\|_{\hat\bSigma_{t}^{-1}}\bigg]\notag\\
        &\leq   4\gamma\sum_{t=1}^T\hat\beta_{t_k} \|\bphi_{{\vvalue}_t}(s_t,a_t)/\bar\sigma_t\|_{\hat\bSigma_{t}^{-1}}\bar\sigma_t,\label{help_1}
\end{align}
where the first inequality holds due to the Cauchy-Schwarz and the triangle inequalities, 
the second inequality holds due to Lemma \ref{lemma:det} and that, by construction, $\det(\bSigma_{t}) \leq 2\det(\bSigma_{t_k})$ for $t_k \leq t < t_{k+1}$, 
while the last holds by
the definition $\cC_{t_k}$ and that for $k=k(t)$, 
by its definition, $\btheta_t\in \hat\cC_{t_k}$ 
and since on event $\event$, $\btheta^*\in  \hat\cC_{t_k}$ also holds.

Combining \eqref{help} and \eqref{help_1}, $I_1$ can be further bounded as
\begin{align}
    I_1 &\leq \sum_{t=1}^T\min\bigg\{\ig,  4\gamma\hat\beta_{t_k} \bar\sigma_t \|\bphi_{\vvalue_t}(s_t,a_t)/\bar\sigma_t\|_{\hat\bSigma_{t}^{-1}}\bigg\}\notag \\
    &\leq  \sum_{t=1}^T (4\hat\beta_{t_k}\gamma\bar\sigma_t + \ig)\min\bigg\{1, \|\bphi_{\vvalue_t}(s_t,a_t)/\bar\sigma_t\|_{\hat\bSigma_{t}^{-1}}\bigg\} \notag \\
    & \leq \sqrt{\underbrace{\sum_{t=1}^T   (4\hat\beta_T\gamma\bar\sigma_t + \ig)^2}_{J_1} }\sqrt{\underbrace{\sum_{t=1}^T\min\bigg\{1, \|\bphi_{\vvalue_t}(s_t,a_t)/\bar\sigma_t\|_{\hat\bSigma_{t}^{-1}}^2\bigg\}}_{J_2}},\label{eq:i_1_1}
\end{align}
where the second inequality holds since $\min\{ac, bd\} \leq (a+b)\min\{c,d\}$ for $a,b,c,d>0$, the last inequality holds due to the Cauchy-Schwarz inequality and $\hat\beta_{t} \leq \hat\beta_T$. To further bound \eqref{eq:i_1_1}, we have
\begin{align}
    J_1 &\leq 2\ig^2T + 32\hat\beta_T^2\gamma^2\sum_{t=1}^T\bar\sigma_t^2 \notag \\
    & =\tilde O\bigg(\ig^2T + d\bigg(\ig^2T/d +  d^{1.5}\ig^2\sqrt{T} + \ig T + \ig^3 + \ig\sum_{t=1}^T\PP\big[\vvalue_t -\vvalue^{\pi}_{t+1} \big](s_t, a_t)\bigg)\bigg)\notag \\
    & = \tilde O\bigg(T(\ig^2 + d\ig) + \sqrt{T} d^{2.5}\ig^2 + d\ig^3 + d\ig\sum_{t=1}^T\PP\big[\vvalue_t -\vvalue^{\pi}_{t+1} \big](s_t, a_t)\bigg),
\end{align}
where the first equality holds due to Lemma \ref{lemma:sumregret} and the facts
\begin{align}
    &\hat\beta_t= 8\sqrt{d\log(1+t/\lambda) \log(4t^2/\delta)}+ 4\sqrt{d} \log(4t^2/\delta) + \sqrt{\lambda}\pnorm = \tilde\Theta(\sqrt{d}),\notag\\
    &\check\beta_t = 8d\sqrt{\log(1+t\ig/(d\lambda)) \log(4t^2/\delta)}+ 4\sqrt{d} \log(4t^2/\delta) + \sqrt{\lambda}\pnorm= \tilde\Theta(d),\notag \\
    & \tilde\beta_t = 8\sqrt{d\ig^4\log(1+t\ig^4/(d\lambda)) \log(4t^2/\delta)}+ 4\ig^2 \log(4t^2/\delta) + \sqrt{\lambda}\pnorm = \tilde\Theta(\sqrt{d}\ig^2).\notag
\end{align}
with the selection $\lambda = 1/{\pnorm}^2$. 
Since $\event_1$ holds by assumption,
by Lemma \ref{lemma:expectation} we can bound $\sum_{t=1}^T\PP\big[\vvalue_t -\vvalue^{\pi}_{t+1} \big](s_t, a_t)$, which leads to
\begin{align}
    J_1 &\leq \tilde O\bigg(T(\ig^2 + d\ig) + \sqrt{T} d^{2.5}\ig^2 + d\ig^3  + d\ig^2\sqrt{T} + d\ig\text{Regret}'(T) + d\ig^2K(T)\bigg)\notag \\
    & = \tilde O\bigg(T(\ig^2 + d\ig) + \sqrt{T} d^{2.5}\ig^2 + d\ig^3 + d\ig\text{Regret}'(T) \bigg)\notag \\
    & = \tilde O\bigg(T(\ig^2 + d\ig) + d^5\ig^2 +d\ig^3+ d\ig\text{Regret}'(T)\bigg), \label{eq:temp0}
\end{align}
where the first equality holds since $K(T) = \tilde O(d)$ by Lemma \ref{lemma:boundk}, the second equality holds since $2\sqrt{T}d^{2.5}\ig^2 \leq d^5\ig^2 + T\ig^2$. To bound $J_2$, taking $\xb_t = \bphi_{{\vvalue}_t}(s_t,a_t)/\bar\sigma_t$ with the fact $\|\bphi_{{\vvalue}_t}(s_t,a_t)/\bar\sigma_t\|_2 \leq \ig/\bar\sigma_t \leq \sqrt{d}$, by Lemma \ref{lemma:sumcontext}, we have
\begin{align}
    J_2 \leq 2d\log(1+ T/\lambda).\label{eq:temp1}
\end{align}
 Substituting \eqref{eq:temp0} and \eqref{eq:temp1} into \eqref{eq:i_1_1}, we have
\begin{align}
    I_1 & = \tilde O\bigg(\sqrt{T}\sqrt{d\ig^2 + d^2\ig} + d^{2.5}\ig + \sqrt{d}\ig^{1.5} + d\sqrt{\ig\text{Regret}'(T)}\bigg)
   ,\label{eq:i_1_3}
\end{align}
where we again use the fact $K(T) = \tilde O(d)$. 
Substituting \eqref{eq:i_1_3} into \eqref{eq:regret} and rearranging it, we have 
\begin{align}
    &\text{Regret}'(T)\notag \\
    &\quad \leq  4\gamma \ig^2\sqrt{2T \log(1/\delta)} + 2\gamma \ig^2K(T) + 2\ig T\gamma^U + \ig I_1\notag \\
    &\quad = \tilde O\bigg(\sqrt{T}\sqrt{d\ig^4 + d^2\ig^3} + d^{2.5}\ig^2 + \sqrt{d}\ig^{2.5} + \ig T\gamma^U + d\sqrt{\ig^3\text{Regret}'(T)}\bigg),\notag
\end{align}
where we use Lemma \ref{lemma:boundk}. Therefore, by the fact $x = \tilde O(a\sqrt{x}+b) \Rightarrow x = \tilde O(a^2+b)$, we have
\begin{align}
    &\text{Regret}'(T) = \tilde O\bigg(\sqrt{T}\sqrt{d\ig^4 + d^2\ig^3} +  d^{2.5}\ig^2 + d^2\ig^3 + \ig T\gamma^U\bigg).\notag
\end{align}
Finally, by Lemma \ref{lemma:UCB} we have $\text{Regret}(T) \leq \text{Regret}'(T)$, finishing our proof. 
\end{proof}





\bibliographystyle{ims}
\bibliography{reference}

\begin{thebibliography}{63}
\expandafter\ifx\csname natexlab\endcsname\relax\def\natexlab#1{#1}\fi
\expandafter\ifx\csname url\endcsname\relax
  \def\url#1{\texttt{#1}}\fi
\expandafter\ifx\csname urlprefix\endcsname\relax\def\urlprefix{URL }\fi

\bibitem[{Abbasi-Yadkori et~al.(2011)Abbasi-Yadkori, P{\'a}l and
  Szepesv{\'a}ri}]{abbasi2011improved}
\textsc{Abbasi-Yadkori, Y.}, \textsc{P{\'a}l, D.} and \textsc{Szepesv{\'a}ri,
  C.} (2011).
\newblock Improved algorithms for linear stochastic bandits.
\newblock In \textit{Advances in Neural Information Processing Systems}.

\bibitem[{Agarwal et~al.(2020)Agarwal, Kakade and Yang}]{agarwal2020model}
\textsc{Agarwal, A.}, \textsc{Kakade, S.} and \textsc{Yang, L.~F.} (2020).
\newblock Model-based reinforcement learning with a generative model is minimax
  optimal.
\newblock In \textit{Conference on Learning Theory}.

\bibitem[{Audibert et~al.(2009)Audibert, Munos and
  Szepesv{\'a}ri}]{audibert2009}
\textsc{Audibert, J.-Y.}, \textsc{Munos, R.} and \textsc{Szepesv{\'a}ri, C.}
  (2009).
\newblock Exploration-exploitation tradeoff using variance estimates in
  multi-armed bandits.
\newblock \textit{Theoretical Computer Science} \textbf{410} 1876--1902.

\bibitem[{Auer(2002)}]{auer2002using}
\textsc{Auer, P.} (2002).
\newblock Using confidence bounds for exploitation-exploration trade-offs.
\newblock \textit{Journal of Machine Learning Research} \textbf{3} 397--422.

\bibitem[{Ayoub et~al.(2020)Ayoub, Jia, Szepesvari, Wang and
  Yang}]{ayoub2020model}
\textsc{Ayoub, A.}, \textsc{Jia, Z.}, \textsc{Szepesvari, C.}, \textsc{Wang,
  M.} and \textsc{Yang, L.~F.} (2020).
\newblock Model-based reinforcement learning with value-targeted regression.
\newblock \textit{arXiv preprint arXiv:2006.01107} .

\bibitem[{Azar et~al.(2013)Azar, Munos and Kappen}]{azar2013minimax}
\textsc{Azar, M.~G.}, \textsc{Munos, R.} and \textsc{Kappen, H.~J.} (2013).
\newblock Minimax {PAC} bounds on the sample complexity of reinforcement
  learning with a generative model.
\newblock \textit{Machine learning} \textbf{91} 325--349.

\bibitem[{Azar et~al.(2017)Azar, Osband and Munos}]{azar2017minimax}
\textsc{Azar, M.~G.}, \textsc{Osband, I.} and \textsc{Munos, R.} (2017).
\newblock Minimax regret bounds for reinforcement learning.
\newblock In \textit{Proceedings of the 34th International Conference on
  Machine Learning-Volume 70}. JMLR. org.

\bibitem[{Azuma(1967)}]{azuma1967weighted}
\textsc{Azuma, K.} (1967).
\newblock Weighted sums of certain dependent random variables.
\newblock \textit{Tohoku Mathematical Journal, Second Series} \textbf{19}
  357--367.

\bibitem[{Bertsekas and Shreve(2004)}]{bertsekas2004stochastic}
\textsc{Bertsekas, D.~P.} and \textsc{Shreve, S.} (2004).
\newblock \textit{Stochastic optimal control: the discrete-time case}.

\bibitem[{Cai et~al.(2019)Cai, Yang, Jin and Wang}]{cai2019provably}
\textsc{Cai, Q.}, \textsc{Yang, Z.}, \textsc{Jin, C.} and \textsc{Wang, Z.}
  (2019).
\newblock Provably efficient exploration in policy optimization.
\newblock \textit{arXiv preprint arXiv:1912.05830} .

\bibitem[{Chu et~al.(2011)Chu, Li, Reyzin and Schapire}]{chu2011contextual}
\textsc{Chu, W.}, \textsc{Li, L.}, \textsc{Reyzin, L.} and \textsc{Schapire,
  R.} (2011).
\newblock Contextual bandits with linear payoff functions.
\newblock In \textit{Proceedings of the Fourteenth International Conference on
  Artificial Intelligence and Statistics}.

\bibitem[{Dani et~al.(2008)Dani, Hayes and Kakade}]{dani2008stochastic}
\textsc{Dani, V.}, \textsc{Hayes, T.~P.} and \textsc{Kakade, S.~M.} (2008).
\newblock Stochastic linear optimization under bandit feedback.
\newblock In \textit{Conference on Learning Theory}.

\bibitem[{Dann and Brunskill(2015)}]{dann2015sample}
\textsc{Dann, C.} and \textsc{Brunskill, E.} (2015).
\newblock Sample complexity of episodic fixed-horizon reinforcement learning.
\newblock In \textit{Advances in Neural Information Processing Systems}.

\bibitem[{Dann et~al.(2018)Dann, Jiang, Krishnamurthy, Agarwal, Langford and
  Schapire}]{dann2018oracle}
\textsc{Dann, C.}, \textsc{Jiang, N.}, \textsc{Krishnamurthy, A.},
  \textsc{Agarwal, A.}, \textsc{Langford, J.} and \textsc{Schapire, R.~E.}
  (2018).
\newblock On oracle-efficient pac rl with rich observations.
\newblock In \textit{Advances in neural information processing systems}.

\bibitem[{Du et~al.(2019)Du, Kakade, Wang and Yang}]{du2019good}
\textsc{Du, S.~S.}, \textsc{Kakade, S.~M.}, \textsc{Wang, R.} and \textsc{Yang,
  L.~F.} (2019).
\newblock Is a good representation sufficient for sample efficient
  reinforcement learning?
\newblock In \textit{International Conference on Learning Representations}.

\bibitem[{Faury et~al.(2020)Faury, Abeille, Calauz{\`e}nes and
  Fercoq}]{faury2020improved}
\textsc{Faury, L.}, \textsc{Abeille, M.}, \textsc{Calauz{\`e}nes, C.} and
  \textsc{Fercoq, O.} (2020).
\newblock Improved optimistic algorithms for logistic bandits.
\newblock \textit{arXiv preprint arXiv:2002.07530} .

\bibitem[{Freedman(1975)}]{Fre75}
\textsc{Freedman, D.} (1975).
\newblock On tail probabilities for martingales.
\newblock \textit{The Annals of Probability} \textbf{3} 100--118.

\bibitem[{He et~al.(2020{\natexlab{a}})He, Zhou and Gu}]{he2020logarithmic}
\textsc{He, J.}, \textsc{Zhou, D.} and \textsc{Gu, Q.} (2020{\natexlab{a}}).
\newblock Logarithmic regret for reinforcement learning with linear function
  approximation.
\newblock \textit{arXiv preprint arXiv:2010.11566} .

\bibitem[{He et~al.(2020{\natexlab{b}})He, Zhou and Gu}]{he2020minimax}
\textsc{He, J.}, \textsc{Zhou, D.} and \textsc{Gu, Q.} (2020{\natexlab{b}}).
\newblock Minimax optimal reinforcement learning for discounted {MDP}s.
\newblock \textit{arXiv preprint arXiv:2010.00587} .

\bibitem[{Henderson(1975)}]{henderson1975best}
\textsc{Henderson, C.~R.} (1975).
\newblock Best linear unbiased estimation and prediction under a selection
  model.
\newblock \textit{Biometrics}  423--447.

\bibitem[{Jaksch et~al.(2010)Jaksch, Ortner and Auer}]{jaksch2010near}
\textsc{Jaksch, T.}, \textsc{Ortner, R.} and \textsc{Auer, P.} (2010).
\newblock Near-optimal regret bounds for reinforcement learning.
\newblock \textit{Journal of Machine Learning Research} \textbf{11} 1563--1600.

\bibitem[{Jia et~al.(2020)Jia, Yang, Szepesvari and Wang}]{jia2020model}
\textsc{Jia, Z.}, \textsc{Yang, L.}, \textsc{Szepesvari, C.} and \textsc{Wang,
  M.} (2020).
\newblock Model-based reinforcement learning with value-targeted regression.
\newblock In \textit{L4DC}.

\bibitem[{Jiang and Agarwal(2018)}]{jiang2018open}
\textsc{Jiang, N.} and \textsc{Agarwal, A.} (2018).
\newblock Open problem: The dependence of sample complexity lower bounds on
  planning horizon.
\newblock In \textit{Conference On Learning Theory}.

\bibitem[{Jiang et~al.(2017)Jiang, Krishnamurthy, Agarwal, Langford and
  Schapire}]{jiang2017contextual}
\textsc{Jiang, N.}, \textsc{Krishnamurthy, A.}, \textsc{Agarwal, A.},
  \textsc{Langford, J.} and \textsc{Schapire, R.~E.} (2017).
\newblock Contextual decision processes with low {B}ellman rank are
  {PAC}-learnable.
\newblock In \textit{Proceedings of the 34th International Conference on
  Machine Learning-Volume 70}. JMLR. org.

\bibitem[{Jin et~al.(2018)Jin, Allen-Zhu, Bubeck and Jordan}]{jin2018q}
\textsc{Jin, C.}, \textsc{Allen-Zhu, Z.}, \textsc{Bubeck, S.} and
  \textsc{Jordan, M.~I.} (2018).
\newblock Is {Q}-learning provably efficient?
\newblock In \textit{Advances in Neural Information Processing Systems}.

\bibitem[{Jin et~al.(2020)Jin, Yang, Wang and Jordan}]{jin2019provably}
\textsc{Jin, C.}, \textsc{Yang, Z.}, \textsc{Wang, Z.} and \textsc{Jordan,
  M.~I.} (2020).
\newblock Provably efficient reinforcement learning with linear function
  approximation.
\newblock In \textit{Conference on Learning Theory}.

\bibitem[{Kakade et~al.(2003)}]{kakade2003sample}
\textsc{Kakade, S.~M.} \textsc{et~al.} (2003).
\newblock \textit{On the sample complexity of reinforcement learning}.
\newblock Ph.D. thesis.

\bibitem[{Kirschner and Krause(2018)}]{kirschner2018information}
\textsc{Kirschner, J.} and \textsc{Krause, A.} (2018).
\newblock Information directed sampling and bandits with heteroscedastic noise.
\newblock In \textit{Conference On Learning Theory}.

\bibitem[{Lattimore et~al.(2015)Lattimore, Crammer and
  Szepesv{\'a}ri}]{LaCrSze15}
\textsc{Lattimore, T.}, \textsc{Crammer, K.} and \textsc{Szepesv{\'a}ri, C.}
  (2015).
\newblock Linear multi-resource allocation with semi-bandit feedback.
\newblock In \textit{Advances in Neural Information Processing Systems}.

\bibitem[{Lattimore and Hutter(2012)}]{lattimore2012pac}
\textsc{Lattimore, T.} and \textsc{Hutter, M.} (2012).
\newblock {PAC} bounds for discounted {MDP}s.
\newblock In \textit{International Conference on Algorithmic Learning Theory}.
  Springer.

\bibitem[{Lattimore and Szepesv{\'a}ri(2020)}]{lattimore2018bandit}
\textsc{Lattimore, T.} and \textsc{Szepesv{\'a}ri, C.} (2020).
\newblock \textit{Bandit algorithms}.
\newblock Cambridge University Press.

\bibitem[{Lattimore et~al.(2020)Lattimore, Szepesvari and
  Weisz}]{lattimore2020learning}
\textsc{Lattimore, T.}, \textsc{Szepesvari, C.} and \textsc{Weisz, G.} (2020).
\newblock Learning with good feature representations in bandits and in rl with
  a generative model.
\newblock In \textit{International Conference on Machine Learning}. PMLR.

\bibitem[{Li et~al.(2010)Li, Chu, Langford and Schapire}]{li2010contextual}
\textsc{Li, L.}, \textsc{Chu, W.}, \textsc{Langford, J.} and \textsc{Schapire,
  R.~E.} (2010).
\newblock A contextual-bandit approach to personalized news article
  recommendation.
\newblock In \textit{Proceedings of the 19th international conference on World
  wide web}.

\bibitem[{Li et~al.(2019{\natexlab{a}})Li, Wang and Zhou}]{li2019nearly}
\textsc{Li, Y.}, \textsc{Wang, Y.} and \textsc{Zhou, Y.} (2019{\natexlab{a}}).
\newblock Nearly minimax-optimal regret for linearly parameterized bandits.
\newblock In \textit{Conference on Learning Theory}.

\bibitem[{Li et~al.(2019{\natexlab{b}})Li, Wang and Zhou}]{li2019tight}
\textsc{Li, Y.}, \textsc{Wang, Y.} and \textsc{Zhou, Y.} (2019{\natexlab{b}}).
\newblock Tight regret bounds for infinite-armed linear contextual bandits.
\newblock \textit{arXiv preprint arXiv:1905.01435} .

\bibitem[{Liu and Su(2020)}]{liu2020regret}
\textsc{Liu, S.} and \textsc{Su, H.} (2020).
\newblock Regret bounds for discounted mdps.
\newblock \textit{arXiv preprint arXiv:2002.05138} .

\bibitem[{Maurer and Pontil(2009)}]{maurer2009empirical}
\textsc{Maurer, A.} and \textsc{Pontil, M.} (2009).
\newblock Empirical {B}ernstein bounds and sample variance penalization.
\newblock \textit{stat} \textbf{1050} 21.

\bibitem[{Modi et~al.(2020)Modi, Jiang, Tewari and Singh}]{modi2019sample}
\textsc{Modi, A.}, \textsc{Jiang, N.}, \textsc{Tewari, A.} and \textsc{Singh,
  S.} (2020).
\newblock Sample complexity of reinforcement learning using linearly combined
  model ensembles.
\newblock In \textit{International Conference on Artificial Intelligence and
  Statistics}. PMLR.

\bibitem[{Neu and Pike-Burke(2020)}]{Neu2020-xp}
\textsc{Neu, G.} and \textsc{Pike-Burke, C.} (2020).
\newblock A unifying view of optimism in episodic reinforcement learning.
\newblock \textit{Advances Neural Information Processing Systems} .

\bibitem[{Pires and Szepesv{\'a}ri(2016)}]{PiSze16:FLM}
\textsc{Pires, B.} and \textsc{Szepesv{\'a}ri, C.} (2016).
\newblock Policy error bounds for model-based reinforcement learning with
  factored linear models.
\newblock In \textit{COLT}.

\bibitem[{Puterman(2014)}]{puterman2014Markov}
\textsc{Puterman, M.~L.} (2014).
\newblock \textit{{M}arkov decision processes: discrete stochastic dynamic
  programming}.
\newblock John Wiley \& Sons.

\bibitem[{Rusmevichientong and Tsitsiklis(2010)}]{rusmevichientong2010linearly}
\textsc{Rusmevichientong, P.} and \textsc{Tsitsiklis, J.~N.} (2010).
\newblock Linearly parameterized bandits.
\newblock \textit{Mathematics of Operations Research} \textbf{35} 395--411.

\bibitem[{Schweitzer and Seidman(1985)}]{schwesei85}
\textsc{Schweitzer, P.} and \textsc{Seidman, A.} (1985).
\newblock Generalized polynomial approximations in {M}arkovian decision
  processes.
\newblock \textit{J. of Math. Anal. and Appl.} \textbf{110} 568--582.

\bibitem[{Sidford et~al.(2018)Sidford, Wang, Wu, Yang and Ye}]{sidford2018near}
\textsc{Sidford, A.}, \textsc{Wang, M.}, \textsc{Wu, X.}, \textsc{Yang, L.~F.}
  and \textsc{Ye, Y.} (2018).
\newblock Near-optimal time and sample complexities for for solving discounted
  {M}arkov decision process with a generative model.
\newblock \textit{arXiv preprint arXiv:1806.01492} .

\bibitem[{Simchowitz and Jamieson(2019)}]{simchowitz2019non}
\textsc{Simchowitz, M.} and \textsc{Jamieson, K.~G.} (2019).
\newblock Non-asymptotic gap-dependent regret bounds for tabular {MDP}s.
\newblock In \textit{Advances in Neural Information Processing Systems}.

\bibitem[{Sun et~al.(2019)Sun, Jiang, Krishnamurthy, Agarwal and
  Langford}]{sun2019model}
\textsc{Sun, W.}, \textsc{Jiang, N.}, \textsc{Krishnamurthy, A.},
  \textsc{Agarwal, A.} and \textsc{Langford, J.} (2019).
\newblock Model-based {RL} in contextual decision processes: {PAC} bounds and
  exponential improvements over model-free approaches.
\newblock In \textit{Conference on Learning Theory}. PMLR.

\bibitem[{Tossou et~al.(2019)Tossou, Basu and Dimitrakakis}]{tossou2019near}
\textsc{Tossou, A.}, \textsc{Basu, D.} and \textsc{Dimitrakakis, C.} (2019).
\newblock Near-optimal optimistic reinforcement learning using empirical
  {B}ernstein inequalities.
\newblock \textit{arXiv preprint arXiv:1905.12425} .

\bibitem[{Wang et~al.(2020{\natexlab{a}})Wang, Du, Yang and
  Kakade}]{wang2020long}
\textsc{Wang, R.}, \textsc{Du, S.~S.}, \textsc{Yang, L.~F.} and \textsc{Kakade,
  S.~M.} (2020{\natexlab{a}}).
\newblock Is long horizon reinforcement learning more difficult than short
  horizon reinforcement learning?
\newblock \textit{arXiv preprint arXiv:2005.00527} .

\bibitem[{Wang et~al.(2020{\natexlab{b}})Wang, Salakhutdinov and
  Yang}]{wang2020reinforcement}
\textsc{Wang, R.}, \textsc{Salakhutdinov, R.~R.} and \textsc{Yang, L.}
  (2020{\natexlab{b}}).
\newblock Reinforcement learning with general value function approximation:
  Provably efficient approach via bounded eluder dimension.
\newblock \textit{Advances in Neural Information Processing Systems}
  \textbf{33}.

\bibitem[{Wang et~al.(2019)Wang, Wang, Du and Krishnamurthy}]{wang2019optimism}
\textsc{Wang, Y.}, \textsc{Wang, R.}, \textsc{Du, S.~S.} and
  \textsc{Krishnamurthy, A.} (2019).
\newblock Optimism in reinforcement learning with generalized linear function
  approximation.
\newblock \textit{arXiv preprint arXiv:1912.04136} .

\bibitem[{Weisz et~al.(2020)Weisz, Amortila and
  Szepesv{\'a}ri}]{weisz2020exponential}
\textsc{Weisz, G.}, \textsc{Amortila, P.} and \textsc{Szepesv{\'a}ri, C.}
  (2020).
\newblock Exponential lower bounds for planning in {MDP}s with
  linearly-realizable optimal action-value functions.
\newblock \textit{arXiv preprint arXiv:2010.01374} .

\bibitem[{Wu et~al.(2015)Wu, Gy{\"o}rgy and Szepesv{\'a}ri}]{WGySz:NeurIPS15}
\textsc{Wu, Y.}, \textsc{Gy{\"o}rgy, A.} and \textsc{Szepesv{\'a}ri, C.}
  (2015).
\newblock Online learning with gaussian payoffs and side observations.
\newblock In \textit{Advances in Neural Information Processing Systems}.

\bibitem[{Yang et~al.(2020)Yang, Yang and Du}]{yang2020q}
\textsc{Yang, K.}, \textsc{Yang, L.~F.} and \textsc{Du, S.~S.} (2020).
\newblock {Q}-learning with logarithmic regret.
\newblock \textit{arXiv preprint arXiv:2006.09118} .

\bibitem[{Yang and Wang(2019{\natexlab{a}})}]{yang2019sample}
\textsc{Yang, L.} and \textsc{Wang, M.} (2019{\natexlab{a}}).
\newblock Sample-optimal parametric {Q}-learning using linearly additive
  features.
\newblock In \textit{International Conference on Machine Learning}.

\bibitem[{Yang and Wang(2019{\natexlab{b}})}]{yang2019reinforcement}
\textsc{Yang, L.~F.} and \textsc{Wang, M.} (2019{\natexlab{b}}).
\newblock Reinforcement leaning in feature space: Matrix bandit, kernels, and
  regret bound.
\newblock \textit{arXiv preprint arXiv:1905.10389} .

\bibitem[{Zanette et~al.(2020{\natexlab{a}})Zanette, Brandfonbrener, Brunskill,
  Pirotta and Lazaric}]{zanette2020frequentist}
\textsc{Zanette, A.}, \textsc{Brandfonbrener, D.}, \textsc{Brunskill, E.},
  \textsc{Pirotta, M.} and \textsc{Lazaric, A.} (2020{\natexlab{a}}).
\newblock Frequentist regret bounds for randomized least-squares value
  iteration.
\newblock In \textit{International Conference on Artificial Intelligence and
  Statistics}.

\bibitem[{Zanette and Brunskill(2019)}]{zanette2019tighter}
\textsc{Zanette, A.} and \textsc{Brunskill, E.} (2019).
\newblock Tighter problem-dependent regret bounds in reinforcement learning
  without domain knowledge using value function bounds.
\newblock \textit{arXiv preprint arXiv:1901.00210} .

\bibitem[{Zanette et~al.(2020{\natexlab{b}})Zanette, Lazaric, Kochenderfer and
  Brunskill}]{zanette2020learning}
\textsc{Zanette, A.}, \textsc{Lazaric, A.}, \textsc{Kochenderfer, M.} and
  \textsc{Brunskill, E.} (2020{\natexlab{b}}).
\newblock Learning near optimal policies with low inherent {B}ellman error.
\newblock \textit{arXiv preprint arXiv:2003.00153} .

\bibitem[{Zhang and Ji(2019)}]{zhang2019regret}
\textsc{Zhang, Z.} and \textsc{Ji, X.} (2019).
\newblock Regret minimization for reinforcement learning by evaluating the
  optimal bias function.
\newblock In \textit{Advances in Neural Information Processing Systems}.

\bibitem[{Zhang et~al.(2020{\natexlab{a}})Zhang, Ji and
  Du}]{zhang2020reinforcement}
\textsc{Zhang, Z.}, \textsc{Ji, X.} and \textsc{Du, S.~S.}
  (2020{\natexlab{a}}).
\newblock Is reinforcement learning more difficult than bandits? {A}
  near-optimal algorithm escaping the curse of horizon.
\newblock \textit{arXiv preprint arXiv:2009.13503} .

\bibitem[{Zhang et~al.(2020{\natexlab{b}})Zhang, Zhou and Ji}]{zhang2020almost}
\textsc{Zhang, Z.}, \textsc{Zhou, Y.} and \textsc{Ji, X.} (2020{\natexlab{b}}).
\newblock Almost optimal model-free reinforcement learning via
  reference-advantage decomposition.
\newblock \textit{arXiv preprint arXiv:2004.10019} .

\bibitem[{Zhang et~al.(2020{\natexlab{c}})Zhang, Zhou and Ji}]{zhang2020model}
\textsc{Zhang, Z.}, \textsc{Zhou, Y.} and \textsc{Ji, X.} (2020{\natexlab{c}}).
\newblock Model-free reinforcement learning: from clipped pseudo-regret to
  sample complexity.
\newblock \textit{arXiv preprint arXiv:2006.03864} .

\bibitem[{Zhou et~al.(2020)Zhou, He and Gu}]{zhou2020provably}
\textsc{Zhou, D.}, \textsc{He, J.} and \textsc{Gu, Q.} (2020).
\newblock Provably efficient reinforcement learning for discounted {MDP}s with
  feature mapping.
\newblock \textit{arXiv preprint arXiv:2006.13165} .

\end{thebibliography}
\end{document}